\documentclass{article}

\usepackage[dvipsnames]{xcolor}

\usepackage{microtype}
\usepackage{graphicx}
\usepackage{scalerel}
\usepackage{subfigure}
\usepackage{subcaption}
\usepackage{overpic}
\usepackage{booktabs} %

\usepackage{hyperref}
\usepackage{xurl}

\usepackage[accepted]{icml2024}

\usepackage{amsmath}
\usepackage{amssymb}
\usepackage{mathtools}
\usepackage{amsthm}
\usepackage[font=small]{caption}
\usepackage[font=small]{subcaption}
\usepackage{multirow}
\usepackage[leftcaption]{sidecap}
\usepackage[scr=dutchcal, scrscaled=1.05,]{mathalfa} %

\usepackage[inline]{enumitem}
\usepackage{cases}

\usepackage{tabu}
\usepackage{makecell}
\usepackage{diagbox}

\usepackage{comment}
\usepackage{version}
\AddToHook{env/excludesubmission/begin}{\color{purple}} %

\includeversion{excludesubmission} %
\includeversion{arxiv_version} %
\excludeversion{icml_version}  %

\usepackage{tikz}
\usetikzlibrary{arrows}
\usetikzlibrary{calc}
\usetikzlibrary{positioning}
\usetikzlibrary{tikzmark}
\usepackage{tikz-cd}
\tikzcdset{every label/.append style = {font = \small}}

\usepackage[capitalize,noabbrev]{cleveref}

\theoremstyle{plain}

\theoremstyle{definition}

\theoremstyle{remark}

\newcommand{\Abb}{\mathbb{A}}
\newcommand{\Rbb}{\mathbb{R}}

\newcommand{\Nbb}{\mathbb{N}}
\newcommand{\Hom}{\ensuremath{\operatorname{Hom}}}

\newcommand{\Alg}{\ensuremath{\operatorname{Alg}}}

\newcommand{\Exp}{\ensuremath{\operatorname{Exp}}}
\newcommand{\Rpq}{\Rbb^{p,q}}
\newcommand{\Aff}{\mathrm{Aff}}
\newcommand{\tg}{(t,\mkern-2mug)}
\newcommand{\sgn}{\ensuremath{\operatorname{sgn}}}

\DeclareMathSymbol{\shortminus}{\mathbin}{AMSa}{"39}
\newcommand{\minus}{\scalebox{0.75}[1.0]{$-$}}

\newcommand{\algcom}[1]{{\color[HTML]{808080}\texttt{\# #1}}}

\newcommand{\jin}{\mathrm{in}}
\newcommand{\jout}{\mathrm{out}}
\newcommand{\cin}{{c_\jin}}
\newcommand{\cout}{{c_\jout}}
\newcommand{\Ain}{{\Acal_\jin}}
\newcommand{\Aout}{{\Acal_\jout}}
\newcommand{\Win}{{W_\jin}}
\newcommand{\Wout}{{W_\jout}}
\newcommand{\rin}{{\rho_\jin}}
\newcommand{\rout}{{\rho_\jout}}

\newcommand{\rhom}{{\rho_{\mathrm{Hom}}}}
\newcommand{\vecrm}{\mathrm{Vec}}
\newcommand{\algrm}{\mathrm{Alg}}

\newcommand{\rcl}[1][]{\rho_{\Cl}^{#1}}

\newcommand{\ins}{\subseteq}

\newcommand{\bij}[1][]{\xrightarrow[#1]{\sim}}

\newcommand{\Cl}{\ensuremath{\operatorname{Cl}}}
\newcommand{\GL}{\ensuremath{\operatorname{GL}}}
\newcommand{\E}{\ensuremath{\operatorname{E}}}

\renewcommand{\O}{\ensuremath{\operatorname{O}}}
\newcommand{\SO}{\ensuremath{\operatorname{SO}}}
\newcommand{\U}{\ensuremath{\operatorname{U}}}

\newcommand{\Isom}{\ensuremath{\operatorname{Isom}}}

\newcommand{\id}{\ensuremath{\operatorname{id}}}

\newcommand{\Tan}{\mathrm{T}}
\newcommand{\Frm}{\mathrm{F}}
\newcommand{\Gs}{\ensuremath{\operatorname{\Gamma}}}

\newcommand{\ip}{\eta}  %
\newcommand\sbull[1][.5]{\mathbin{\ThisStyle{\vcenter{\hbox{%
  \scalebox{#1}{$\SavedStyle\bullet$}}}}}%
}
\newcommand{\gp}{\sbull[0.65]} %
\newcommand{\agp}{\ast}

\newcommand{\eqr}{\mathord{\sim}} %
\newcommand{\lact}{\rhd}
\newcommand{\ract}{\lhd}
\newcommand{\Trpt}{\operatorname{\mathfrak{T}}}

\newcommand{\LC}{\mathrm{LC}}
\newcommand{\Tens}{\operatorname{Tens}}
\newcommand{\Span}{\operatorname{span}}

\newcommand{\KB}{\mathscr{K}}
\newcommand{\kb}{\mathscr{k}}
\newcommand{\fd}{\mathfrak{f}}

\newcommand{\Acal}{\mathcal{A}}

\newcommand{\Ical}{\mathcal{I}}

\newcommand{\lp}{\left(}
\newcommand{\rp}{\right)}

\newcommand{\lB}{\left[ }
\newcommand{\rB}{\right] }
\newcommand{\lC}{\left\{ }
\newcommand{\rC}{\right\} }

\newcommand{\st}{\,\middle|\,}

\theoremstyle{plain}

\newtheorem{sa}{Theorem}[section]
\newtheorem{Thm}[sa]{Theorem}
\newtheorem{Lem}[sa]{Lemma}
\newtheorem{Prp}[sa]{Proposition}
\newtheorem{Cor}[sa]{Corollary}

\newtheorem{Def}[sa]{Definition}

\newtheorem{DefLem}[sa]{Definition/Lemma}
\newtheorem{DefThm}[sa]{Definition/Theorem}

\newtheorem{Not}[sa]{Notation}

\newtheorem{Rem}[sa]{Remark}

\newtheorem{Eg}[sa]{Example}

\crefname{proof}{proof}{proofs}
\newenvironment{Prf}[2][]{%
    \par\bigskip\noindent%
    \refstepcounter{sa}%
    \label[proof]{#2:prf}%
    \textbf{Proof \thesa\ for \Cref{#2}}%
    \ifx&#1&\else\space(#1)\fi%
    \textbf{.}%
    }{\hfill\qed\par\bigskip}%

\newcommand{\inputgraphicskernel}[3]{%
    \begin{tikzpicture}%
        \clip (0,0)  circle (#1);%
        \node[anchor=center] at (0,0) {\includegraphics[width=#2]{#3}};%
    \end{tikzpicture}%
}%

\makeatletter
\renewcommand*{\ALG@name}{Function}
\makeatother

\usepackage[textsize=small]{todonotes}

\allowdisplaybreaks[2] %

\icmltitlerunning{Clifford-Steerable Convolutional Neural Networks}

\begin{document}

\twocolumn[
    \icmltitle{Clifford-Steerable Convolutional Neural Networks}

    \icmlsetsymbol{equal}{\kern-.5pt\raisebox{-.0pt}{\scalebox{1.3}{*}}\kern-.3pt}

    \begin{icmlauthorlist}
        \icmlauthor{Maksim Zhdanov}{AMLab}
        \icmlauthor{David Ruhe}{equal,AMLab,AI4Sc,API}
        \icmlauthor{Maurice Weiler}{equal,AMLab}
        \icmlauthor{Ana Lucic}{MSR}
        \icmlauthor{Johannes Brandstetter}{Linz,NXAI}
        \icmlauthor{Patrick Forré}{AMLab,AI4Sc}
    \end{icmlauthorlist}

    \icmlaffiliation{AMLab}{AMLab, Informatics Institute, University of Amsterdam}
    \icmlaffiliation{AI4Sc}{AI4Science Lab, Informatics Institute, University of Amsterdam}
    \icmlaffiliation{API}{Anton Pannekoek Institute for Astronomy, University of Amsterdam}
    \icmlaffiliation{MSR}{AI4Science, Microsoft Research}
    \icmlaffiliation{Linz}{ELLIS Unit Linz, Institute for Machine Learning, JKU Linz, Austria}
    \icmlaffiliation{NXAI}{NXAI GmbH}

    \icmlcorrespondingauthor{Maksim Zhdanov}{m.zhdanov@uva.nl}

    \icmlkeywords{Machine Learning, ICML}

    \vskip 0.3in
]

\printAffiliationsAndNotice{\icmlEqualContribution} %

\begin{abstract}
    We present \emph{Clifford-Steerable Convolutional Neural Networks} (CS-CNNs), a novel class of $\E(p, q)$-equivariant CNNs.
    CS-CNNs process \emph{multivector fields} on pseudo-Euclidean spaces $\Rbb^{p,q}$.
    They cover, for instance, $\E(3)$-equivariance on $\Rbb^3$ and Poincaré-equivariance on Minkowski spacetime $\Rbb^{1,3}$.
    Our approach is based on an implicit parametrization of $\O(p,q)$-steerable kernels via Clifford group equivariant neural networks.
    We significantly and consistently outperform baseline methods on fluid dynamics as well as relativistic electrodynamics forecasting tasks.
\end{abstract}

\section{Introduction}

Physical systems are often described by \emph{fields} on (pseudo)-Euclidean spaces.
Their equations of motion obey various symmetries,
such as isometries $\E(3)$ of Euclidean space $\Rbb^3$ or relativistic Poincaré transformations $\E(1,3)$ of Minkowski spacetime $\Rbb^{1,3}$.
PDE solvers should respect these symmetries. 
In the case of
deep learning based surrogates, this property is ensured by making the neural networks \emph{equivariant}
\mbox{(commutative) w.r.t. the transformations of interest.}

\begin{arxiv_version}
\begin{figure}
    \vspace*{-10pt}%
    \makebox[\columnwidth][c]{
    \includegraphics[width=1.16\columnwidth]{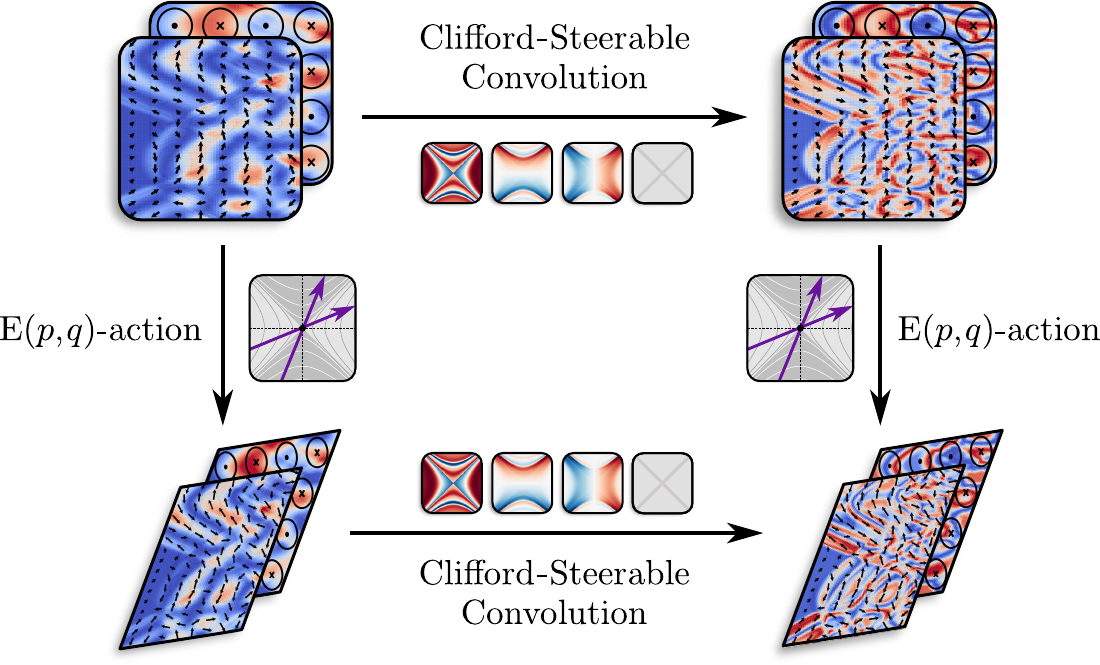}}
    \vspace*{-22pt}%
    \caption{
        CS-CNNs process multivector fields while respecting $\E(p,q)$-equivariance.
        Shown here is
        a Lorentz-boost $\O(1,1)$ of electromagnetic data on 1+1-dimensional spacetime $\Rbb^{1,1}$.
    }
    \label{fig:main}%
    \vspace*{-12pt}%
\end{figure}
\end{arxiv_version}
\begin{icml_version}
\begin{figure}
    \vspace*{-8pt}%
    \makebox[\columnwidth][c]{
    \hspace*{-.13\columnwidth}
    \includegraphics[width=1.13\columnwidth]{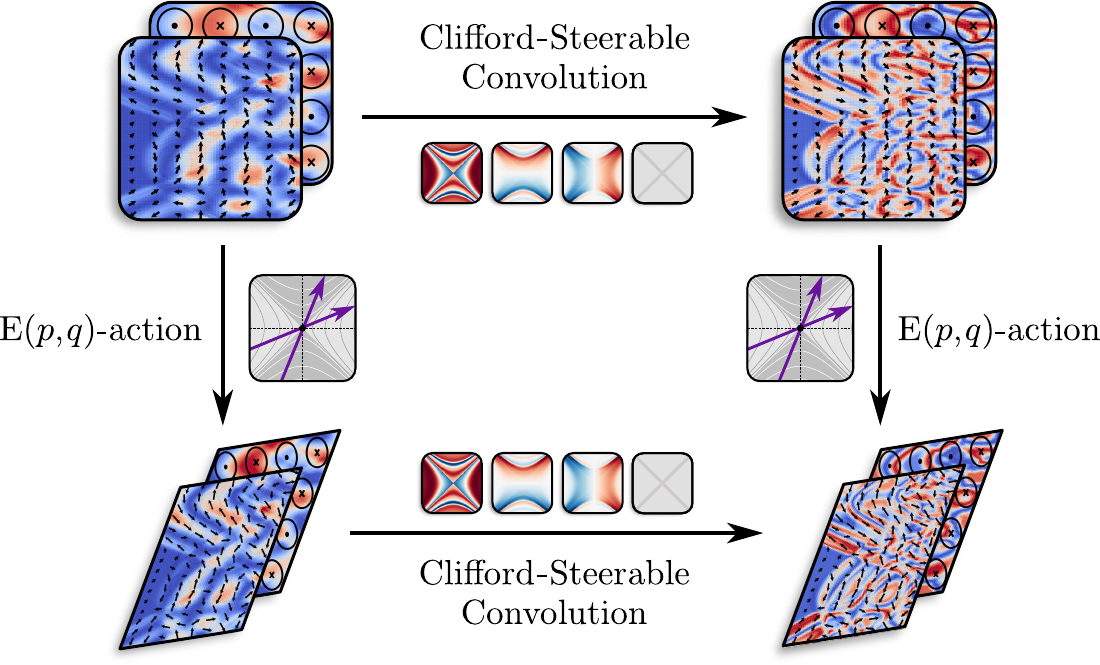}}
    \vspace*{-12pt}%
    \caption{
        CS-CNNs process multivector fields while respecting $\E(p,q)$-equivariance.
        Shown here is
        a Lorentz-boost $\O(1,1)$ of electromagnetic data on 1+1-dimensional spacetime $\Rbb^{1,1}$.
    }
    \label{fig:main}%
    \vspace*{-12pt}%
\end{figure}
\end{icml_version}

A fairly general class of equivariant CNNs 
covering arbitrary spaces and field types 
is described by the theory of \emph{steerable CNNs} \cite{weiler2023EquivariantAndCoordinateIndependentCNNs}.
The central result there is that equivariance requires a \emph{``$G$-steerability''}
constraint on convolution kernels,
where $G=\O(n)$ or $\O(p,q)$ for $\E(n)$- or $\E(p,q)$-equivariant CNNs, respectively.
This constraint was solved and implemented for $\O(n)$ \cite{lang2020WignerEckart,cesa2021ENsteerable},
however, $\O(p,q)$-steerable kernels are so far still missing.

This work proposes \emph{Clifford-steerable CNNs} (CS-CNNs),
which process \emph{multivector fields} on pseudo-Euclidean spaces $\Rpq$,
and are equivariant to the pseudo-Euclidean group $\E(p,q)$: the isometries of $\Rpq$.
Multivectors are elements of the Clifford (or geometric) algebra $\Cl(\Rpq)$ of $\Rpq$.
Neural networks based on Clifford algebras have seen a recent surge in popularity in the field of deep learning
and were used to build both non-equivariant
\cite{Brandstetter2022CliffordNL,ruhe2023geometric}
and equivariant
\cite{ruhe2023CliffordGroupEquivariantNNs,brehmer2023geometric}
models.
While multivectors do not cover all possible field types, e.g. general tensor fields,
they include those most relevant in physics.
For instance, the Maxwell or Dirac equation and General Relativity can be formulated using the spacetime algebra $\Cl(\Rbb^{1,3})$.

The steerability constraint on convolution kernels is usually either solved analytically or numerically,
however, such solutions are not yet known for $\O(p,q)$.
Observing that the $G$-steerability constraint is just a $G$-equivariance constraint,
\citet{zhdanov2022implicit} propose to implement $G$-steerable kernels \emph{implicitly} via $G$-equivariant MLPs.
Our CS-CNNs follow this approach,
implementing implicit $\O(p,q)$-steerable kernels via the
$\O(p,q)$-equivariant neural networks for multivectors developed by \citet{ruhe2023CliffordGroupEquivariantNNs}.

We demonstrate the efficacy of our approach by predicting the evolution of several physical systems.
In particular, we consider a fluid dynamics forecasting task on $\Rbb^2$,
as well as relativistic electrodynamics simulations on both $\Rbb^3$ and $\Rbb^{1,2}$.
CS-CNNs are the first models respecting the full spacetime symmetries of these problems. 
They significantly outperform competitive baselines,
including conventional steerable CNNs and non-equivariant Clifford CNNs.
This result remains consistent over dataset sizes.
When evaluating the empirical equivariance error of our approach for $\E(2)$ symmetries,
we find that we perform on par with the analytical solutions of \citet{Weiler2019_E2CNN}.

The main contributions of this work are the following:
\begin{itemize}[leftmargin=18pt, topsep=-1pt, itemsep=-6pt]
    \item 
        While prior work considered only individual multivectors,
        CS-CNNs process full multivector fields on pseudo-Euclidean spaces or manifolds.
    \item
        We investigate the representation theory of $\O(p,q)$-steerable kernels for multivector fields
        and develop an implicit implementation via $\O(p,q)$-equivariant MLPs.
    \item
        The resulting $\E(p,q)$-equivariant CNNs are evaluated on various PDE simulation tasks,
        where they consistently outperform strong baselines.
\end{itemize}

This paper is organized as follows:
Section~\ref{sec:theoretical_background} introduces the theoretical background underlying our method.
CS-CNNs are then developed in Section~\ref{sec:clifford_steerable_CNNs_main},
and empirically evaluated in Section~\ref{sec:experimental_results}.
A generalization from flat spaces to general \emph{pseudo-Riemannian manifolds} is presented in
Appendix~\ref{sec:cs-cnn-prm}.

\section{Theoretical Background}
\label{sec:theoretical_background}

The core contribution of this work is to provide a framework for the construction of \emph{steerable CNNs} for processing \emph{multivector fields} on general \emph{pseudo-Euclidean spaces}.
We provide background on pseudo-Euclidean spaces and their symmetries in Section~\ref{sec:pseudo_Euclidean_spaces_groups},
on equivariant (steerable) CNNs in Section~\ref{sec:feature_fields_and_steerable_CNNs_general}, and
on multivectors and the Clifford algebra formed by them in Section~\ref{sec:clifford_algebra_group}.

\subsection{Pseudo-Euclidean spaces and groups}
\label{sec:pseudo_Euclidean_spaces_groups}

Conventional Euclidean spaces are metric spaces,
i.e. they are equipped with a metric that assigns \emph{positive} distances to any pair of distinct points.
Pseudo-Euclidean spaces allow for more general \emph{indefinite metrics},
which relax the positivity requirement on distances.
Pseudo-Euclidean spaces appear in our theory in two distinct settings:
First, the (affine) base spaces on which feature vector fields are supported,~e.g. Minkowski spacetime, are pseudo-Euclidean.
Second, the feature vectors attached to each point of spacetime are themselves elements of pseudo-Euclidean vector spaces.
We~introduce these spaces and their symmetries in the following.

\subsubsection{Pseudo-Euclidean vector spaces}

\begin{Def}[Pseudo-Euclidean vector space]
\label{def:pseudo_Euclidean_vector_space}

    A \emph{pseudo-Euclidean vector space} (inner product space) $(V,\eta)$ of signature $(p,q)$ is a ${p\mkern-1.5mu+\mkern-2muq}$-dimensional vector space $V$ over $\Rbb$
    equipped with an \emph{inner product} $\ip$, which we define as a \emph{non-degenerate}%
    \footnote{
       Note that we explicitly refrain from imposing \emph{positive-definiteness} onto the definition of inner product, in order to include typical Minkowski spacetime inner products, etc. 
    }
    symmetric bilinear form
    {\abovedisplayskip=3pt%
     \belowdisplayskip=2pt%
    \begin{align}
        \ip:\, V \times V &\to \Rbb, & (v_1,v_2) & \mapsto \ip(v_1,v_2)
    \end{align}
    }%
    with $p$ and $q$ positive and negative eigenvalues, respectively.
\end{Def}

If ${q\mkern-2mu=\mkern-2mu0}$, $\ip$ becomes positive-definite, and ${(V,\ip)}$ is a conventional Euclidean inner product space.
For ${q\geq1}$, ${\ip(v,v)}$ can be negative, rendering ${(V,\ip)}$ pseudo-Euclidean. 

Since every 
inner product space $(V,\ip)$ of signature $(p,q)$ has an orthonormal basis, we can always find a linear isometry with the standard pseudo-Euclidean space $\Rpq \cong (V,\ip)$, to which we mostly will restrict our attention in this paper. 

\begin{Def}[Standard pseudo-Euclidean vector spaces]
\label{eg:standard_pseudo_euclidean_space}
    Let $e_1,\dots,e_{p+q}$ be the standard basis of $\Rbb^{p+q}$.
    Define an inner product
    of signature ${(p,\mkern-2muq)}$
    {\abovedisplayskip=4pt%
     \belowdisplayskip=3pt%
    \begin{equation}
        \ip^{p,q}(v_1,v_2) \,:=\, v_1^\top \Delta^{p,q} v_2
    \end{equation}
    }%
    in this basis via its matrix representation
    {\abovedisplayskip=4pt%
     \belowdisplayskip=3pt%
    \begin{align}
        \Delta^{p,q} := \mathrm{diag}(\underbrace{1,\dots,1\vphantom{\big|}}_{\textup{$p$ times}},\, \underbrace{\minus1,\dots,\minus1\vphantom{\big|}}_{\textup{$q$ times}})\,.
    \end{align}
    }%
    We call the inner product space $\Rpq:=(\Rbb^{p+q},\ip^{p,q})$ the \emph{standard pseudo-Euclidean vector space} of signature $(p,q)$.
\end{Def}

\begin{figure}[t]
    \centering
    \vspace*{-6pt}
    \includegraphics[width=.91\columnwidth]{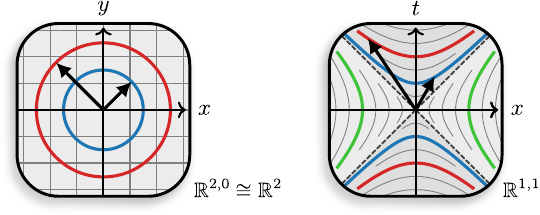}
    \vspace*{-13pt}
    \caption{
        Examples of pseudo-Euclidean spaces $\Rbb^{2,0}$ and $\Rbb^{1,1}$.
        Colors depict $\O(p,q)$-orbits,
        given by sets of all points $v\in\Rpq$ with the \emph{same squared distance} $\eta^{p,q}(v,v)$ from the origin.
    }
    \label{fig:pe_spaces}%
    \vspace*{-0.40cm}
\end{figure}

\begin{Eg}
    ${\Rbb^{3,0} \equiv \Rbb^3}$ recovers the 3-dimensional Euclidean vector space with its standard positive-definite inner product
    $\Delta^{3,0} = \mathrm{diag}(1,1,1)$.
    The signature ${(p,q)}={(1,3)}$ corresponds, instead, to Minkowski spacetime $\Rbb^{1,3}$ with Minkowski inner product $\Delta^{1,3} = \mathrm{diag}(1,\minus1,\minus1,\minus1)$\,.%
    \footnote{
        There exist different conventions regarding whether time or space components are assigned the negative sign.
    }
\end{Eg}

\subsubsection{Pseudo-Euclidean groups}
We are interested in neural networks that respect (i.e., commute with, or are \emph{equivariant} to) the symmetries of pseudo-Euclidean spaces, which we define here.
For concreteness, we give these definitions for the standard pseudo-Euclidean vector spaces $\Rpq$.
Let us start with the two cornerstone groups that define such symmetries:

\begin{Def}[Translation groups]
    The \emph{translation group} $(\Rpq,+)$ associated with $\Rpq$ is formed by its set of vectors and its (canonical) vector addition.
\end{Def}

\begin{Def}[Pseudo-orthogonal groups]
\label{def:Opq}
    The \emph{pseudo-orthogonal group} $\O(p,\mkern-1muq)$ associated to $\Rpq$ is formed~by~all invertible linear maps that preserve its inner product,
    {%
    \abovedisplayskip=4pt
    \belowdisplayskip=6pt
    \begin{align}
        \O(p,q) \,:=\, \big\{ g\in\GL(\Rpq) \,\big|\, g^\top\!\Delta^{p,q}g = \Delta^{p,q} \big\},
    \end{align}
    }%
    together with matrix multiplication.
    $\O(p,q)$ is compact
    for ${p=0}$ or ${q=0}$, and non-compact for mixed signatures.
\end{Def}

\begin{Eg}
    For ${(p,q)=(3,0)}$, we obtain the usual orthogonal group $\O(3)$, i.e. rotations and reflections,
    while ${(p,q)=(1,3)}$ corresponds to the relativistic \emph{Lorentz group} $\O(1,3)$, which also includes 
    boosts between inertial frames.
\end{Eg}

Taken together, translations and pseudo-orthogonal transformations of $\Rpq$ form its \emph{pseudo-Euclidean group}, which is the group of all metric preserving symmetries (isometries).%
\footnote{
    As the translations contained in $\E(p,q)$ move the origin of $\Rpq$,
    they do not preserve the \emph{vector} space structure of $\Rpq$,
    but only its structure as \emph{affine} space.
}

\begin{Def}[Pseudo-Euclidean groups]
    The \emph{pseudo-Euclidean group} for $\Rpq$ is defined as semidirect product
    \begin{align}
        \E(p,q) \,:=\, (\Rpq,+) \rtimes \O(p,q)
    \end{align}
    with group multiplication
    defined by
    ${(\tilde{t},\mkern-2mu\tilde{g}) \cdot \tg} = {(\tilde{t}+\tilde{g}t,}\, \tilde{g}g)$.
    Its canonical action on $\Rpq$ is given by
    \begin{align}
        \E(p,q)\times\Rpq\to\Rpq, \quad \big((t,\mkern-2mug),\,x\big) \mapsto gx+t
    \end{align}
\end{Def}

\begin{Eg}
    The usual Euclidean group $\E(3)$ is reproduced for ${(p,q)=(3,0)}$.
    For Minkowski spacetime, ${(p,q)=(1,3)}$, we obtain the Poincaré group $\E(1,3)$.
\end{Eg}

\subsection{Feature vector fields \& Steerable CNNs}
\label{sec:feature_fields_and_steerable_CNNs_general}
Convolutional neural networks operate on spatial signals,
formalized as \emph{fields of feature vectors} on a base space $\Rpq$.
Transformations of the base space imply corresponding transformations of the feature vector fields defined on them, see Fig. \ref{fig:main} (left column).
The specific transformation laws depend thereby on their geometric ``field type''
(e.g., scalar, vector, or tensor fields).
Equivariant CNNs commute with such transformations of feature fields.
The theory of \emph{steerable CNNs} shows that this requires a $G$-equivariance constraint on convolution kernels \cite{weiler2023EquivariantAndCoordinateIndependentCNNs}.
We briefly review the definitions and basic results of feature fields and steerable CNNs in Sections~\ref{sec:feature_fields_general} and~\ref{sec:steerable_CNNs_general} below.

For generality, this section considers topologically closed matrix groups $G\leq\GL(\Rpq)$ and affine groups $\Aff(G) = (\Rpq,+)\rtimes G$,
and allows for any field type.
Section~\ref{sec:clifford_steerable_CNNs_main} will
more specifically focus on
pseudo-orthogonal groups $G=\O(p,q)$,
pseudo-Euclidean groups $\Aff(\O(p,q)) = \E(p,q)$,
and multivector fields.
For a detailed review of Euclidean steerable CNNs and their generalization to Riemannian manifolds we refer to
\citet{weiler2023EquivariantAndCoordinateIndependentCNNs}.

\subsubsection{Feature vector fields}
\label{sec:feature_fields_general}

Feature vector fields are functions ${f: \Rpq\to W}$
that assign to each point $x\in\Rpq$
a feature $f(x)$ in some feature vector space $W$.
They are additionally equipped with an $\Aff(G)$-action
determined by a $G$-representation $\rho$ on $W$.

The specific choice of $(W,\rho)$ fixes the geometric ``type'' of feature vectors.
For instance, $W=\Rbb$ and trivial $\rho(g)=1$ corresponds to \emph{scalars},
$W=\Rpq$ and $\rho(g)=g$ describes \emph{tangent vectors}.
Higher order tensor spaces and representations give rise to \emph{tensor fields}.
Later on, ${W\mkern-3mu=\mkern-1mu\Cl(\Rpq)}$ will be the Clifford algebra and feature vectors will be \emph{multivectors}
with a natural $\O(p,q)$-representation $\rcl$.

\begin{Def}[Feature vector field]
    \label{def:feature_vector_field_general}
    Consider a pseudo-Euclidean ``base space'' $\Rpq$.
    Fix any $G\leq\GL(\Rpq)$ and consider a $G$-representation $(W,\rho)$, called ``field type''.

    Let ${\Gamma(\Rpq,W) := \{f:\Rpq\to W\}}$ denote the vector space of \mbox{$W$\!-feature} fields.
    Define an $\Aff(G)$-action
    \begin{align}
        \rhd_\rho: \Aff(G)\times \Gamma(\Rpq,W) \to \Gamma(\Rpq,W)
    \end{align}%
    by setting $\forall\ \tg\in\Aff(G),\ f\in\Gamma(\Rpq,W),\ x\in\Rpq$:
    \begin{align*}
        \big[ \tg\mkern-2mu\rhd_\rho\mkern-3mu f \big](x)
        := \rho(g) f\big( \tg^{\minus1}x \big)
        = \rho(g) f\big( g^{\minus1}(x\minus t) \big).
    \end{align*}

    Since $\Gamma(\Rpq,W)$ is a vector space and $\rhd_\rho$ is linear,
    the tuple $\big(\Gamma(\Rpq,W), \rhd_\rho\big)$ forms the $\Aff(G)$-representation
    of \emph{feature vector fields} of type $(W,\rho)$.%
    \footnote{
        ${\big(\Gamma(\Rpq\!,W), \rhd_\rho \mkern-2mu\big)}$ is called \emph{induced representation} $\mathrm{Ind}_G^{\Aff(G)}\mkern-3mu\rho$ \cite{Cohen2019-generaltheory}.
        From a differential geometry perspective, it can be viewed as the space of bundle sections of a $G$-associated feature vector bundle; see Defs.~\ref{def:G_associated_vector_bundle}, \ref{def:global_bundle_section} and \cite{weiler2023EquivariantAndCoordinateIndependentCNNs}.
    }
\end{Def}

\begin{minipage}{\linewidth}
\begin{Rem}
    Intuitively, $\tg$ acts on $f$ by
    \begin{enumerate}[itemsep=0pt, topsep=2pt]
        \item moving feature vectors across the base space, from points $g^{-1}(x-t)$ to new locations $x$, and
        \item $G$-transforming individual feature vectors $f(x)\in W$ themselves by means of the $G$-representation $\rho(g)$.
    \end{enumerate}
\end{Rem}
\end{minipage}

Besides the field types mentioned above, equivariant neural networks
often rely on \emph{irreducible, regular} or \emph{quotient} representations.
More choices of field types are discussed and benchmarked in \citet{Weiler2019_E2CNN}.

\subsubsection{Steerable CNNs}
\label{sec:steerable_CNNs_general}

Steerable convolutional neural networks are composed of layers that are $\Aff(G)$-equivariant,
that is, which commute with affine group actions on feature fields:

\begin{Def}[$\Aff(G)$-equivariance]
\label{def:AffG_equivariance}
    Consider any two $G$-representations $(W_\textup{in},\rho_\textup{in})$ and $(W_\textup{out},\rho_\textup{out})$.
    Let
    $L: \Gamma(\Rpq,W_\textup{in}) \to \Gamma(\Rpq,W_\textup{out})$ be a function (``layer'') between the corresponding spaces of feature fields.
    This layer is said to be $\Aff(G)$-equivariant iff it satisfies
    \begin{align}
        L\big( \tg\rhd_{\rho_\textup{in}} f \big)
        \,=\, \tg\rhd_{\rho_\textup{out}} L(f)
    \end{align}
    for any $\tg\in\Aff(G)$ and any $f\in\Gamma(\Rpq,W_\textup{in})$.
    Equivalently, the following diagram should commute:
    \begin{equation}
        \begin{tikzcd}[row sep=2.8em, column sep=3.4em]
            \Gamma(\Rpq,W_\textup{in})
            \arrow[r, "L"]
            \arrow[d, "\tg\,\rhd_{\rho_\textup{in}}"']
            &
            \Gamma(\Rpq,W_\textup{out})
            \arrow[d, "\tg\,\rhd_{\rho_\textup{out}}"]
            \\
            \Gamma(\Rpq,W_\textup{in})
            \arrow[r, "L"']
            &
            \Gamma(\Rpq,W_\textup{out})
        \end{tikzcd}
    \end{equation}
\end{Def}

The most basic operations used in neural networks are parameterized \emph{linear} layers.
If one demands \emph{translation equivariance}, these layers are necessarily \emph{convolutions} (see Theorem 3.2.1 in \cite{weiler2023EquivariantAndCoordinateIndependentCNNs}). 
Similarly, linearity and $\Aff(G)$-equivariance requires \emph{steerable convolutions},
that is, convolutions with $G$-steerable kernels:

\begin{samepage}
\begin{Thm}[Steerable convolution]
\label{thm:steerable_conv}
    Consider a layer ${L: \Gamma(\Rpq,W_\textup{in}) \to \Gamma(\Rpq,W_\textup{out})}$ mapping between feature fields of types
    $(W_\textup{in},\rho_\textup{in})$ and $(W_\textup{out},\rho_\textup{out})$, respectively.
    If $L$ is demanded to be \emph{linear} and \emph{$\Aff(G)$-equivariant}, then:
    \begin{enumerate}[topsep=0pt, itemsep=-2pt]
        \item
            $L$ needs to be a \emph{convolution} integral
            \footnote{
                $dv$ is the usual Lebesgue measure on $\Rbb^{p+q}$.
                For the integral to exist, we assume $f$ to be bounded and have compact support.
            }
            \begin{align*}
                L\mkern-1mu\big(f_\jin\big)(u)
                \mkern1mu=\mkern1mu \big[K\mkern-2mu\ast\mkern-2mu f_\jin\big]\mkern-2mu(u)
                \mkern1mu:=\mkern-2mu \int_{\Rpq}\mkern-18mu K(v) \big[f_\jin(u\minus v)\big]\mkern2mu dv,
            \end{align*}
            parameterized by a convolution kernel
            \begin{align}
            \label{eq:kernel_def}
                K: \Rpq \to \Hom_\vecrm(W_\textup{in},W_\textup{out}) \,.
            \end{align}
            The kernel is \emph{operator-valued} since it aggregates
            input features in $W_\textup{in}$ linearly into output features in $W_\textup{out}$.%
            \footnote{
                $\Hom_\vecrm(W_\textup{in},\mkern-2muW_\textup{out})$,
                the space of vector space homomorphisms,
                consists of all \emph{linear maps} ${W_\textup{in}\!\to\! W_\textup{out}}$.
                When putting $W_\textup{in}=\Rbb^{C_\textup{in}}$ and $W_\textup{out}=\Rbb^{C_\textup{out}}$,
                this space can be identified with 
                the space $\Rbb^{C_\textup{out}\times C_\textup{in}}$ of ${C_\textup{out}\mkern-4mu\times\mkern-4muC_\textup{in}}$ matrices.
            }%
            \footnote{
                $K:\Rpq \!\to\! \Hom_\vecrm(W_\textup{in},W_\textup{out})$
                itself need \emph{not} be linear.
            }

        \item
            The kernel is required to be \emph{$G$-steerable},
            that is, it needs to satisfy the $G$-equivariance constraint%
            \footnote{
                This is in particular \emph{not} demanding $K(v)$ to be (equivariant) homomorphisms of $G$-representations in $\Hom_G(W_\textup{in},W_\textup{out})$,
                despite $(W_\textup{in},\rho_\textup{in})$ and $(W_\textup{out},\rho_\textup{out})$ being $G$-representations.
                Only $K$ itself is $G$-equivariant as map $\Rpq\to\Hom_\vecrm(W_\textup{in},\mkern-2muW_\textup{out})$.
            }%
            \vspace*{-3pt}
            \begin{align}
            \label{eq:kernel_constraint_general}
                K(gx)
                \ &=\ \scalebox{.875}{$\displaystyle\frac{1}{|\det(g)|}$}
                      \rho_\textup{out}(g) K(x) \rho_\textup{in}(g)^{-1} \\
                \ &=:\ \rho_\textup{Hom}(g)(K(x)) \notag
            \end{align}
            for any $g\in G$ and $x\in\Rpq$.
            This constraint is diagrammatically visualized by the commutativity of:
            \begin{equation}
                \begin{tikzcd}[row sep=2.8em, column sep=4.2em]
                    \Rpq
                    \arrow[r, "K"]
                    \arrow[d, "g\cdot\,"']
                    &
                    \Hom_\vecrm(W_\textup{in}, W_\textup{out})
                    \arrow[d, "\,\rho_\textup{Hom}(g)"]
                    \\
                    \Rpq
                    \arrow[r, "K"']
                    &
                    \Hom_\vecrm(W_\textup{in}, W_\textup{out})
                \end{tikzcd}
            \end{equation}
            \end{enumerate}
\end{Thm}
\begin{proof}
    See Theorem 4.3.1 in \cite{weiler2023EquivariantAndCoordinateIndependentCNNs}. 
\end{proof}
\end{samepage}

\begin{Rem}[Discretized kernels]
    In practice, kernels are often discretized as arrays of~shape
    {\abovedisplayskip=4pt
     \belowdisplayskip=5pt
    $$\big( X_1,\dots,X_{p+q},\, C_\textup{out}, C_\textup{in} \big)$$
    }%
    with ${C_\textup{out} \mkern-1mu=\mkern-1mu \dim(W_\textup{out})}$
    and  ${C_\textup{in}  \mkern-1mu=\mkern-1mu \dim(W_\textup{in} )}$.
    The first $p+q$ axes are indexing a pixel grid on the domain $\Rpq$,
    while the last two axes represent the linear operators in the codomain by
    ${C_\textup{out} \mkern-2mu\times\mkern-2mu C_\textup{in}}$ matrices.
\end{Rem}
\vspace*{-4pt}

The main takeaway of this section is that
\emph{one needs to implement $G$-steerable kernels in order to implement $\Aff(G)$-equivariant CNNs.}
This is a notoriously difficult problem,
requiring specialized approaches for different categories of groups $G$ and field types $(W,\rho)$.
Unfortunately, the usual approaches do not immediately apply to our goal of implementing $\O(p,q)$-steerable kernels for multivector fields.
These include the following cases:

\vspace*{-4pt}
\begin{itemize}[align=left, itemindent=10pt, leftmargin=10pt, topsep=0pt, itemsep=0pt]
    \item[\emph{Analytical}:]
        Most commonly, steerable kernels are parameterized in \emph{analytically derived steerable kernel bases}.%
        \footnote{
            Unconstrained kernels, Eq.~\eqref{eq:kernel_def}, can be linearly combined, and therefore form a vector space.
            The steerability constraint, Eq.~\eqref{eq:kernel_constraint_general} is \emph{linear}.
            Steerable kernels span hence a linear subspace
            and can be parameterized in terms of a \emph{basis} of steerable kernels.
        }
        Solutions are known for
        $\SO(3)$ \cite{3d_steerableCNNs},~$\O(3)$ \cite{geiger2020e3nn}
        and any $G\leq\O(2)$ \cite{Weiler2019_E2CNN}.~\citet{lang2020WignerEckart} and \citet{cesa2021ENsteerable}
        generalized this to any \emph{compact} groups ${G\mkern-2mu\leq\mkern-2mu\U(d)}$.
        However, their solutions still require knowledge of
        \begin{icml_version}
            irreps,
        \end{icml_version}
        \begin{arxiv_version}
            irreducible representations,
        \end{arxiv_version}
        Clebsch-Gordan coefficients and harmonic basis functions, which need to be derived and implemented for each single group individually.
        Furthermore, these
        solutions do not cover pseudo-orthogonal groups $\O(p,q)$ of mixed signature, since these are \emph{non-compact}.

    \item[\emph{Regular}:]
        For regular and quotient representations,
        steerable kernels can be implemented via channel permutations in the matrix dimensions.
        This is, for instance, done in regular group convolutions \cite{Cohen2016-GCNN,weiler2018SFCNN,bekkers2018roto,cohen2019gaugeIco,finzi2020generalizing}.
        However, these approaches require \emph{finite} $G$ or rely on sampling \emph{compact} $G$,
        again ruling out general (non-compact) $\O(p,q)$.

    \item[\emph{Numerical}:]
        \citet{Cohen2017-STEER} solved the kernel constraint for \emph{finite} $G$ numerically.
        For $\SO(2)$, \citet{deHaan2020meshCNNs} derived numerical solutions based on Lie-algebra representation theory.
        The numerical routine by \citet{shutty2020learning} solves for Lie-algebra irreps given their structure constants.
        Corresponding Lie group irreps follow via the matrix exponential, however, only on \emph{connected} groups like the subgroups $\SO^+(p,q)$ of $\O(p,q)$.
        
    \item[\emph{Implicit}:]
        Convolution kernels, Eq.~\eqref{eq:kernel_def}, are merely
        maps between vector spaces $\Rpq$ and $\Hom_\vecrm(W_\textup{in},W_\textup{out})$,
        which can be implemented \emph{implicitly} via \emph{MLPs} \cite{romero2021ckconv}.
        \emph{Steerable} kernels are additionally $G$-equivariant, Eq.~\eqref{eq:kernel_constraint_general}.
        Combining these insights, \citet{zhdanov2022implicit} parameterize them implicitly via $G$-equivariant MLPs.
        However, to implement these MLPs, one usually requires irreps, irrep endomorphisms and Clebsch-Gordan coefficients for each $G$ of interest.

\end{itemize}

Our approach presented in Section~\ref{sec:clifford_steerable_CNNs_main}
is based on the implicit kernel parametrization via neural networks by \citet{zhdanov2022implicit},
which requires us to implement $\O(p,q)$-equivariant neural networks.
Fortunately, the \emph{Clifford group equivariant neural networks} by \citet{ruhe2023CliffordGroupEquivariantNNs}
establish $\O(p,q)$-equivariance for the practically relevant case of \emph{Clifford-algebra} representations $\rcl$,
i.e., $\O(p,q)$-actions on multivectors.
The Clifford algebra, and Clifford group equivariant neural networks, are introduced in the next section.

\subsection{The Clifford Algebra \& Clifford Group Equivariant Neural Networks}
\label{sec:clifford_algebra_group}

This section introduces \emph{multivector features},
a specific type of geometric feature vectors with $\O(p,q)$-action.
Multivectors are the elements of a \emph{Clifford algebra} $\Cl(V,\ip)$
corresponding to a pseudo-Euclidean $\Rbb$-vector space $(V,\ip)$.
The most relevant properties of Clifford algebras
in relation to applications in geometric deep learning are the following:
\begin{itemize}[leftmargin=14pt, topsep=0pt, itemsep=1pt]
    \item
        $\Cl(V,\ip)$ is, in itself, an $\Rbb$-vector space of dimension
        $2^d$ with $d := \dim(V) = {p\mkern-1mu+\mkern-1muq}$.
        This allows to use multivectors as feature vectors of neural networks \cite{Brandstetter2022CliffordNL,ruhe2023geometric, brehmer2023geometric}.

    \item
        As an \emph{algebra}, ${\Cl(\mkern-.5muV,\mkern-1.5mu\ip)}$ comes with an
        $\Rbb$-bilinear operation
        {\abovedisplayskip=-5.5pt%
         \belowdisplayskip=4.5pt%
        \begin{align*}
            \gp:\; \Cl(V,\ip) \mkern-1mu\times\mkern-1mu \Cl(V,\ip) \,\to\, \Cl(V,\ip),
        \end{align*}
        }%
        called \emph{geometric product}.%
        \footnote{
            The geometric product is \emph{unital, associative, non-commutative}, and \emph{$\O(V,\ip)$-equivariant}.
            Its main defining property is highlighted in Eq.~\eqref{eq:Cl_self_product}.
            A proper definition is given in \Cref{def:CliffordAlgebraAppendix}, Eq.~\eqref{eq:GeometricProductDefAppendix}.
        }
        We can therefore multiply multivectors with each other,
        which will be a key aspect in various  neural network operations.

    \item
        $\Cl(V,\ip)$ is furthermore a \emph{representation} space
        of the pseudo-orthogonal group $\O(V,\ip)$ via $\rcl$,
        defined in Eq \eqref{eq:pseudo_orthogonal_group_abstract} below.
        This allows to use multivectors as features of \emph{$\O(V,\ip)$-equivariant} networks \cite{ruhe2023CliffordGroupEquivariantNNs}.
\end{itemize}

A formal definition of Clifford algebras can be found in Appendix \ref{sec:cliff-alg}.
\Cref{sec:intro-Clifford-algebra} offers a less technical introduction,
highlighting basic constructions and results.
Sections \ref{sec:Clifford-algebra-as-Opq-rep} and \ref{sec:cegnns} focus on the natural $\O(p,q)$-action on multivectors, and on Clifford group equivariant neural networks.
While we will later mostly be interested in $(V,\ip)\mkern-2mu=\mkern-1mu\Rpq$ and $\O(V,\ip)=\O(p,q)$,
we keep the discussion here general.

\subsubsection{Introduction to the Clifford algebra}
\label{sec:intro-Clifford-algebra}

Multivectors are constructed by multiplying and summing vectors.
Specifically, $l$ vectors $v_1,\dots, v_l \in V$ multiply to ${v_1\gp\dots\gp v_l} \in \Cl(V,\ip)$.
A general multivector arises as a linear combination of such products,
{\abovedisplayskip=4pt%
 \belowdisplayskip=4pt%
\begin{align}
    x = \sum\nolimits_{i \in I}
    c_i \cdot v_{i,1} \gp \cdots \gp v_{i,l_i} \,,
\end{align}
}%
with some finite index set $I$ and $v_{i,k} \in V$ and $c_i \in \Rbb$.

The main algebraic property of the Clifford algebra is that it relates the geometric product of vectors $v\in V$ to the inner product $\ip$ on $V$ by requiring:
{\abovedisplayskip=3pt%
 \belowdisplayskip=5pt%
\begin{align}
\label{eq:Cl_self_product}
    v\gp v \,\overset{!}{=}\, \ip(v,v) \cdot 1_{\Cl(V,\ip)}
    \qquad \forall\, v\in V \subset \Cl(V,\ip)
\end{align}
}%
Intuitively, this means that the product of a vector with itself collapses to a scalar value $\ip(v,v)\in \Rbb\subseteq \Cl(V,\ip)$,
from which all other properties of the algebra follow by bilinearity.
This leads in particular to the \emph{fundamental relation}%
\footnote{
    To see this, use $v:=v_1+v_2$ in Eq.~\eqref{eq:Cl_self_product} and expand.
}:
{
\abovedisplayskip=7pt
\belowdisplayskip=4pt
\begin{align*}
    v_2 \gp v_1 = - v_1 \gp v_2 \,+\, 2 \ip(v_1,v_2) \!\cdot\! 1_{\Cl(V,\ip)}
    \quad \forall\, v_1,v_2 \in V.
\end{align*}
}%

For the standard orthonormal basis $[e_1,\dots,e_{p+q}]$ of $\Rpq$ this reduces to the following simple rules:
{
\abovedisplayskip=6pt
\belowdisplayskip=3pt
\begin{subnumcases}{e_i \gp e_j \,=\, }
    -e_j \gp e_i            & for $i \neq j$        \label{eq:geometric_product_ON_basis_a} \\
    \ip(e_i,e_i) = +1      & for $i = j \leq p$    \label{eq:geometric_product_ON_basis_b} \\
    \ip(e_i,e_i) = -1      & for $i = j > p$       \label{eq:geometric_product_ON_basis_c}
\end{subnumcases}
}

An (orthonormal)
basis of $\Cl(V,\ip)$ is constructed by repeatedly taking geometric products of any basis vectors $e_i\in V$.
Note that, \emph{up to sign flip},
(1) the \emph{ordering} of elements in any product is \emph{irrelevant} due to Eq.~\eqref{eq:geometric_product_ON_basis_a}, and
(2) any \emph{elements occurring twice cancel out} due to Eqs.~(\ref{eq:geometric_product_ON_basis_b},\ref{eq:geometric_product_ON_basis_c}).

\begin{table}
    \centering
    \small
    \setlength{\tabcolsep}{5.9pt}
    \begin{tabular}{ r c c c c c }
        \toprule
        name                            & grade $k$             & dim\,${d\choose k}$   & basis $k$-vectors     & \!norm\!   \\
        \midrule
        scalar                          & $0$                   & $1$                   & $1$                   & $+1$ \\ 
        \hline
        \multirow{2}{*}{vector}         & \multirow{2}{*}{$1$}  & \multirow{2}{*}{$3$}  & $e_1$                 & $+1$ \\  
                                        &                       &                       & $e_2,\;e_3$           & $-1$ \\  
        \hline
        \!\multirow{2}{*}{pseudovector} & \multirow{2}{*}{$2$}  & \multirow{2}{*}{$3$}  & $e_{12},\;e_{13}$     & $-1$ \\
                                        &                       &                       & $e_{23}$              & $+1$ \\
        \hline
        \!pseudoscalar                  & $3$                   & $1$                   & $e_{123}$             & $+1$ \\
        \bottomrule
    \end{tabular}
    \vspace*{-1ex}
    \captionsetup{format=hang}
    \caption{
            Orthonormal basis for $\Cl(\Rpq)$ with $(p,q)=(1,2)$.
            ``Norm'' refers to $\bar\ip(e_A,e_A)=\ip_A$; see Eq.~\eqref{eq:Clifford_induced_metric}.
    }
    \label{tab:basis_blades}
    \vspace*{-8pt}
\end{table}

The basis elements constructed this way
can be identified with (and labeled by) subsets $A \subseteq [d] := \{1,\dots,d\}$,
where the presence~or absence of an index $i\in A$ signifies whether
the corresponding $e_i$ appears in the product.
Agreeing furthermore on an ordering to disambiguate signs,
we define
{
\abovedisplayskip=4pt
\belowdisplayskip=5pt
\begin{align*}
    e_A := e_{i_1} \!\gp e_{i_2} \!\gp \dots\gp e_{i_k}
    \mkern15mu
    \textup{for }\
    A = \{i_1 \mkern-2mu<\mkern-2mu \cdots \mkern-2mu<\mkern-2mu i_k\} \neq \varnothing
\end{align*}
}
and ${e_\varnothing := 1_{\Cl(V,\ip)}}$.
From this, it is clear
that $\dim \Cl(V,\ip)$ $= 2^d$.
Table~\ref{tab:basis_blades} gives a specific example for $(V,\eta)=\Rbb^{1,2}$.

Any multivector $x\in\Cl(V,\ip)$ can be uniquely expanded in this basis,
{
\abovedisplayskip=-4pt
\belowdisplayskip=4pt
\begin{align}
\label{eq:gen-el-clf-ord}
    x \,=\, \sum\nolimits_{A \ins [d]} x_A \cdot e_A,
\end{align}
}%
where $x_A \in \Rbb$ are coefficients.

Note that there are $\binom{d}{k}$ basis elements $e_A$ of \emph{``grade''} ${|\mkern-1muA|\mkern-3mu=\mkern-3muk}$,
i.e., which are composed from $k$ out of the $d$ distinct $e_i\in V$.
These span ${d\mkern-2mu+\mkern-2mu1}$ linear subspaces $\Cl^{(k)}(V,\ip)$,
the elements of which are called \emph{$k$-vectors}.
They include
scalars (${k\mkern-2mu=\mkern-2mu0}$),
vectors (${k\mkern-2mu=\mkern-2mu1}$),
bivectors (${k\mkern-2mu=\mkern-2mu2}$), etc.
The full Clifford algebra decomposes thus into a direct sum over grades:
{
\abovedisplayskip=2pt
\belowdisplayskip=5pt        
\begin{align*}
    \Cl(V,\ip) = \bigoplus\nolimits_{k=0}^d \Cl^{(k)}(V,\ip),
    \mkern16mu
    \dim \Cl^{(k)}(V,\ip) \mkern-1mu=\mkern-1mu
    \binom{d}{k}.
\end{align*}
}%
Given any multivector $x$, expanded as in Eq.~\eqref{eq:gen-el-clf-ord},
we can define its $k$-th \emph{grade projection}
on $\Cl^{(k)}(V,\ip)$ as:
{
\abovedisplayskip=8pt
\belowdisplayskip=4pt
\begin{align}
\label{eq:grd-k-clf}
    x^{(k)} \,=\, \sum\nolimits_{A \ins [d],\, |A|=k} x_A \cdot e_A.
\end{align}
}%

Finally, the inner product $\ip$ on $V$ is naturally extended to $\Cl(V,\ip)$ by defining
$\bar\ip: \Cl(V,\ip) \times \Cl(V,\ip) \to \Rbb$ as
\begin{align}
\label{eq:Clifford_induced_metric}
    \bar\ip(x,y) \,:=\, \sum\nolimits_{A \ins [d]} \ip_A \cdot x_A \cdot y_A,
\end{align}
where $\ip_A:=\prod\nolimits_{i \in A} \ip(e_i,\!e_i) \in \{\pm1\}$ are sign factors.
The tuple $(e_A)_{A \ins [d]}$ is an orthonormal basis of $\Cl(V,\ip)$ w.r.t. $\bar\ip$.

All of these constructions and statements are more formally defined and proven in the appendix of \cite{ruhe2023geometric}.

\subsubsection{Clifford grades as $\O\mkern-.5mu(p,\mkern-2muq\mkern-.5mu)$-representations}
\label{sec:Clifford-algebra-as-Opq-rep}

The individual grades $\Cl^{(k)}(V,\ip)$ turn out to be representation spaces of the (abstract) pseudo-orthogonal group
{
\abovedisplayskip=-8pt
\begin{align}
\label{eq:pseudo_orthogonal_group_abstract}
    \\[-2pt] \mkern-6mu \notag
    \O(V,\ip) :=
    \big\{ g \mkern-1mu\in\mkern-1mu \GL(V)
    \mkern2mu\big|\, \forall v \mkern-3mu\in\mkern-3mu V\!:
    \ip(gv,\mkern-2mugv) \mkern-2mu=\mkern-2mu \ip(v,\mkern-2muv) \mkern-1mu\big\},
\end{align}
}%
which coincides for $(V,\ip)=\Rpq$ with $\O(p,q)$ in Def.~\ref{eg:standard_pseudo_euclidean_space}.
$\O(V,\ip)$ acts thereby on multivectors by
\emph{individually multiplying each $1\mkern-1mu$-vector} from which they are constructed~\emph{with~$g$}.

\begin{DefThm}[$\O(V,\ip)$-action on $\Cl(V,\ip)$]
\label{thm:group_action_on_Clifford_full}
    Let $(V,\ip)$ be a pseudo-Euclidean space,
    $g,g_i\in\O(V,\ip)$,
    $c_i\in\Rbb$,
    $v_{i,j}\in V$,
    $x,x_i\in\Cl(V,\ip)$,
    and
    $I$ a finite index set.
    \\[2pt]
    Define the orthogonal algebra representation
    \begin{align}
        \rcl:\, \O(V,\ip) \to \O_{\Alg}\lp \Cl(V,\ip),\bar\ip\rp\footnotemark
    \end{align}
    \footnotetext{%
        $\O_{\Alg} \!\big(\! \Cl(V,\ip),\bar\ip\big)$ is the group of all linear orthogonal transformations of $\Cl(V,\ip)$ that are also multiplicative w.r.t.\ $\gp$\,.
    }%
    of $\O(V,\ip)$ via the canonical $\O(V,\ip)$-action on each of the contained $1$-vectors:
    {
    \begin{align}
        &
        \rcl(g) \Big(\sum\nolimits_{i \in I} c_i \mkern-2mu\cdot\mkern-2mu v_{i1} \gp \ldots \gp v_{ij_i} \Big)
        \\[2pt] \notag
        \mkern0mu:=&
        \sum\nolimits_{i \in I} c_i \mkern-2mu\cdot\mkern-2mu (gv_{i1}) \gp \ldots \gp (gv_{ij_i}).
    \end{align}
    }%
    $\rcl$ is well-defined as an \emph{orthogonal representation:}
    \begin{itemize}[leftmargin=64pt, topsep=0pt, itemsep=6pt]
    \item[linear:]          $\mkern8mu  \rcl(g)(c_1 \cdot x_1+c_2 \cdot x_2)$ 
    \\[3pt] \hspace*{-0pt}  $           = c_1 \cdot \rcl(g)(x_1) + c_2 \cdot \rcl(g)(x_2)$
    \item[composing:]       $\mkern8mu  \rcl(g_2)\lp \rcl(g_1)(x) \rp = \rcl(g_2g_1)(x)$
    \item[invertible:]      $\mkern8mu  \rcl(g)^{-1}(x) = \rcl(g^{-1})(x)$,
    \item[orthogonal:]      $\mkern8mu  \bar\ip(\rcl(g)(x_1),\rcl(g)(x_2)) = \bar\ip(x_1,x_2)$
    \end{itemize}
    Moreover, the geometric product is \emph{$\O(V,\ip)$-equivariant},
    making $\rcl$ an (orthogonal) \emph{algebra representation:}
    \begin{align}
        \rcl(g)(x_1) \gp \rcl(g)(x_2) = \rcl(g)(x_1 \gp x_2).
    \end{align}
    \begin{equation}
        \begin{tikzcd}[row sep=2.6em, column sep=4.2em]
            \Cl(V,\ip)\times\Cl(V,\ip)
            \arrow[r, "\gp"]
            \arrow[d, "\rcl(g)\mkern-2mu\times\mkern-2mu\rcl(g)"']
            &
            \Cl(V,\ip)
            \arrow[d, "\,\rcl(g)"]
            \\
            \Cl(V,\ip)\times\Cl(V,\ip)
            \arrow[r, "\gp"']
            &
            \Cl(V,\ip)
        \end{tikzcd}
    \end{equation}

\end{DefThm}

This representation $\rcl$ reduces furthermore to independent sub-representations on individual $k$-vectors.

\begin{Thm}[$\O(V,\ip)$-action on grades $\Cl^{(k)}(V,\ip)$]
\label{thm:group_action_on_Clifford_grade}
    Let $g\in\O(V,\ip)$, $x\in\Cl(V,\ip)$ and $k\in{0,\dots,d}$ a grade.
    \\[4pt]
    The grade projection $(\mkern2mu\cdot\mkern2mu)^{(k)}$ is $\O(V,\ip)$-equivariant:
    \begin{align}
        \big( \rcl(g)\,x \big)^{(k)} = \rcl(g)\big(x^{(k)}\big)
    \end{align}
   \begin{equation}
       \begin{tikzcd}[row sep=2.4em, column sep=4.4em]
           \Cl(V,\ip)
           \arrow[r, "(\mkern2mu\cdot\mkern2mu)^{(k)}"]
           \arrow[d, "\rcl(g)\,"']
           &
           \Cl^{(k)}(V,\ip)
           \arrow[d, "\,\rcl(g)"]
           \\
           \Cl(V,\ip)
           \arrow[r, "(\mkern2mu\cdot\mkern2mu)^{(k)}"']
           &
           \Cl^{(k)}(V,\ip)
       \end{tikzcd}
   \end{equation}
    This implies in particular that $\Cl(V,\ip)$ is reducible to subrepresentations $\Cl^{(k)}(V,\ip)$, i.e.\ $\rcl(g)$ does not mix grades.
\end{Thm}
\vspace*{-9pt}
\begin{proof}
    Both theorems are proven in \cite{ruhe2023CliffordGroupEquivariantNNs}.
\end{proof}
\vspace*{-9pt}

\begin{arxiv_version}
\begin{SCfigure*}[.8]
    \centering
    \includegraphics[width=1.355\linewidth]{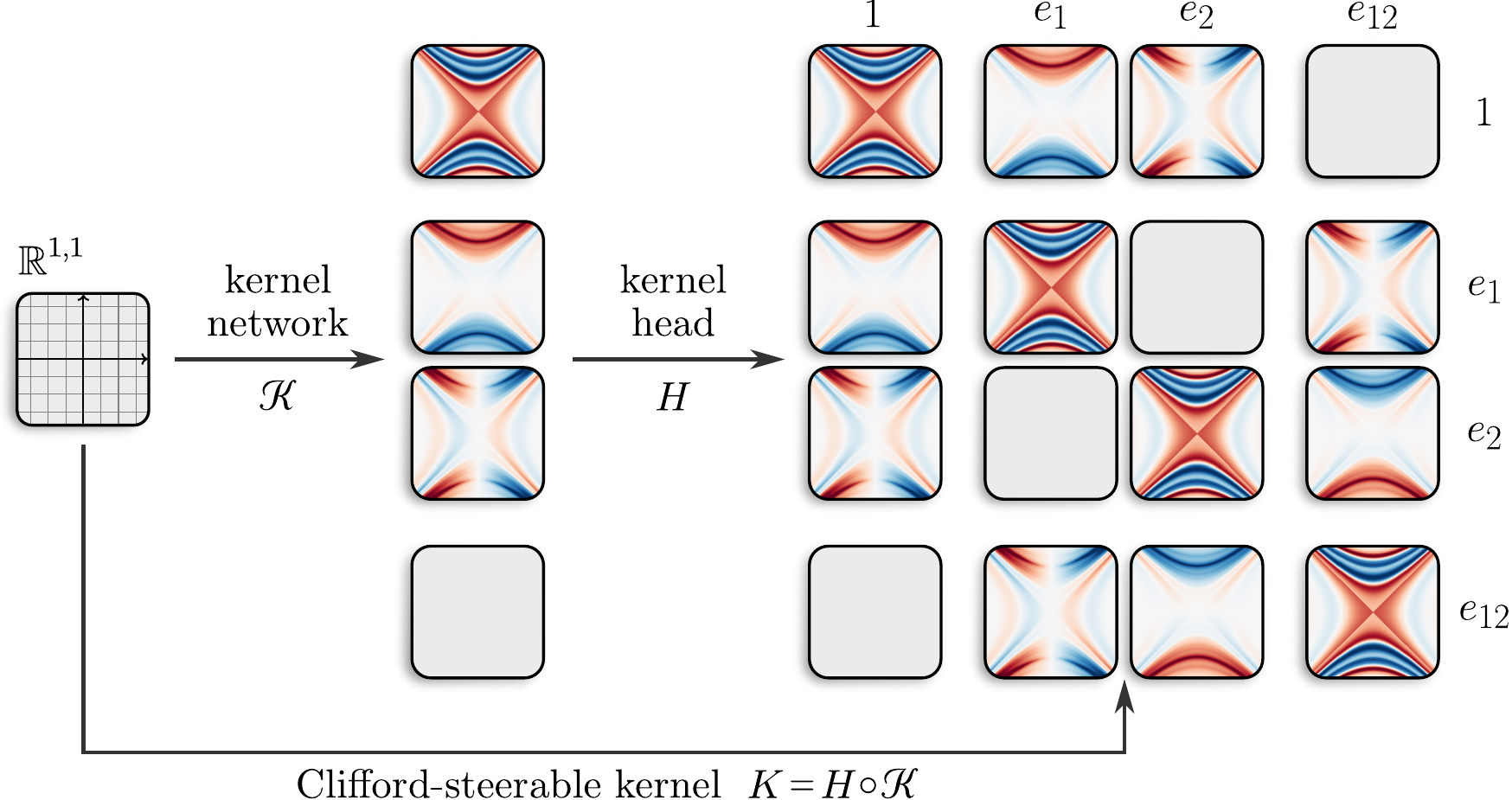}
    \vspace*{10pt}
    \captionsetup{width=.67\textwidth}
    \caption{%
        Implicit~Clifford-steerable kernel
        with light-cone structure
        for ${(p,\mkern-2muq) \mkern-2mu = \mkern-2mu (1,\mkern-2mu1)}$
        and $\cin =$ ${\cout \mkern-2mu = \mkern-2mu 1}$.
        It is parameterized by a \emph{kernel network} $\KB$,
        producing a field of
        (${\cin \mkern-4mu\times\mkern-4mu \cout}$)
        multi-vector valued outputs.
        These are convolved with multivector fields
        by taking their weighted geometric product
        at each location in a convolutional manner.
        This is equivalent to a conventional steerable convolution
        after expansion to a ${\O(1,\mkern-3mu1)}$-steerable kernel via a \emph{kernel head} operation $H$.
        For more details and equivariance properties see the commutative diagram in Fig. \ref{fig:implicit_steerable_kernel_diagram}.
        A more detailed variant for $\Rbb^{2,0}$ and $\O(2)$ which additionally visualizes weighting parameters is shown in Fig. \ref{fig:implicit_steerable_kernel_visual_O2}.
    }
    \label{fig:implicit_steerable_kernel_visual}
\end{SCfigure*}
\end{arxiv_version}

\subsubsection{\mbox{$\O\mkern-.5mu(p,\mkern-2muq\mkern-.5mu)$-equivariant Clifford Neural Nets}}
\label{sec:cegnns}

Based on those properties, \citet{ruhe2023CliffordGroupEquivariantNNs} proposed
\emph{Clifford group equivariant neural networks} (CGENNs).
Due to a group isomorphism, this is equivalent to the network's $\O(V,\ip)$-equivariance.

\begin{DefThm}[Clifford Group Equivariant NN]
\label{thm:cgenn}
    Consider a grade ${k=0,\mkern-1mu...,d}$ and weights $w^k_{mn}\in\Rbb$.
    A \emph{Clifford group equivariant neural network} (CGENN)
    is constructed from the following functions,
    operating on one or more multivectors $x_i\in\Cl(V,\ip)$.
    \begin{itemize}[align=left, leftmargin=10pt, itemindent=-4pt, topsep=0pt, itemsep=1pt]
        \item[\emph{Linear layers:}]
            mix $k$-vectors. %
            For each ${1\mkern-2mu\leq\mkern-2mu m\mkern-2mu\leq\mkern-2mu\cout}$\textup{:}
            {
            \abovedisplayskip=4pt
            \belowdisplayskip=4pt
            \begin{align}
                L_m^{(k)}(x_1,\dots,x_{\cin}) := \sum\nolimits_{n=1}^{\cin} w_{mn}^k \cdot x_n^{(k)} 
            \end{align}
            }%
            Such weighted linear mixing within sub-representations $\Cl^{(k)}(V,\ip)$ is common in equivariant MLPs.
        \item[\emph{Geometric product layers:}]
            compute weighted geometric products
            with grade-dependent weights:
            {
            \abovedisplayskip=-7pt
            \belowdisplayskip=6pt
            \begin{align}
            \label{eq:geom-prod-layer}
            \\[-4pt] \notag
                P^{(k)}(x_1,x_2)
                := \sum\nolimits_{m=0}^d\sum\nolimits_{n=0}^d
                w_{mn}^k \mkern-2mu\cdot\mkern-2mu \big(x_1^{(m)} \gp x_2^{(n)} \big)^{(k)} 
            \end{align}
            }%
            This is similar to the irrep-feature tensor products in MACE \cite{batatia2022MACE}.
            \vspace*{-2pt}
        \item[\emph{Nonlinearity:}]
            As activations, we use
            ${A(x) \mkern-1mu:=\mkern-1mu x\mkern-1.5mu\cdot\mkern-1.5mu \Phi\big( x^{(0)} \big)}$
            where $\Phi$ is the CDF of the Gaussian distribution.
            This is inspired by $\mathrm{GatedGELU}$ from \citet{brehmer2023geometric}.
  \end{itemize}
\vspace*{-1pt}
  All of these operations are by
  Theorems~\ref{thm:group_action_on_Clifford_full} and \ref{thm:group_action_on_Clifford_grade}
  $\O(V,\ip)$-equivariant.
\end{DefThm}
\vspace*{-1pt}

\begin{icml_version}
\begin{figure*}
     \centering
     \vspace*{-5pt}
     \begin{minipage}{0.43\textwidth}
        \centering
        \includegraphics[width=.95\linewidth]{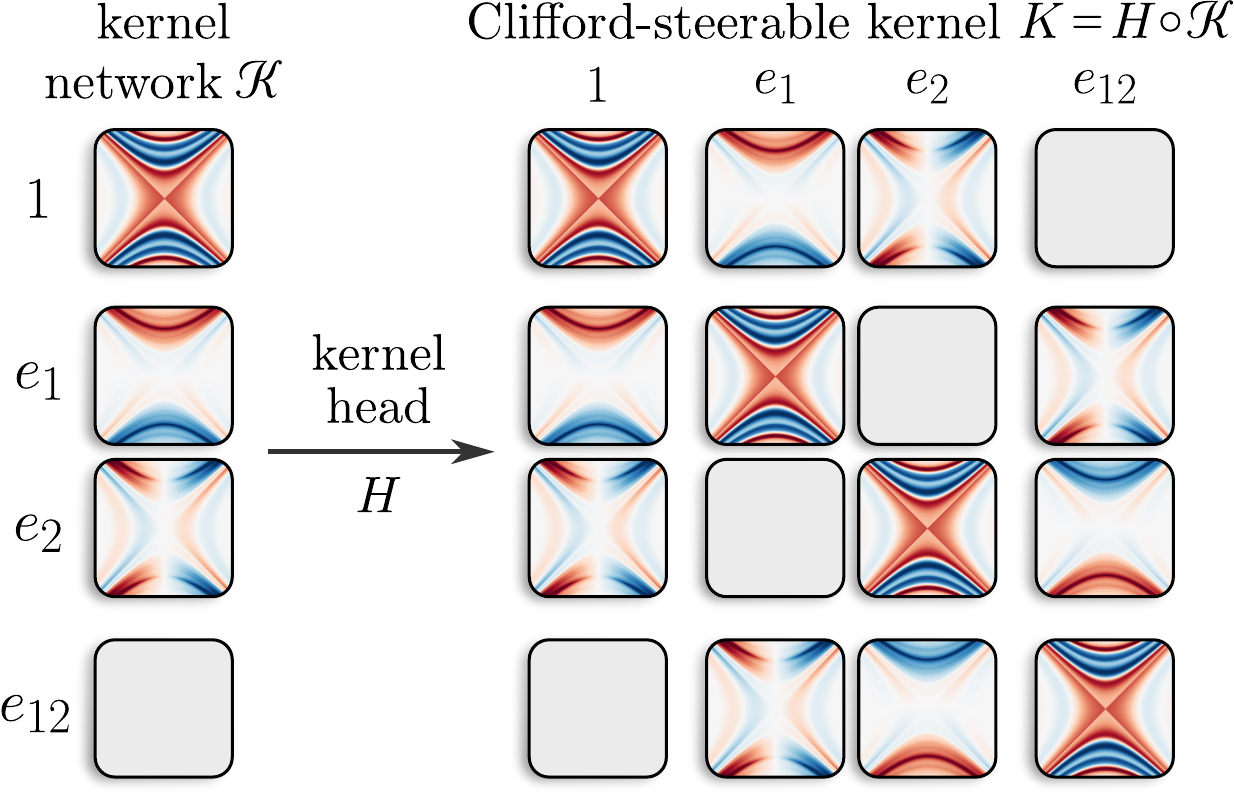}
     \end{minipage}
     \hfill
     \begin{minipage}{0.54\textwidth}
     \centering
     \vspace*{1pt}
        \begin{tikzcd}[row sep=3.2em, column sep=4.8em]
            \Rpq
            \vphantom{\big|}
            \arrow[rr, pos=.5, rounded corners, to path={
            -- ([yshift=3.2ex]\tikztostart.north)
            --node[above]{\small$K$} ([yshift=3.2ex]\tikztotarget.north)
            -- (\tikztotarget.north)
            }]
            \arrow[r, "\KB"]
            \arrow[d, "g\cdot\,"']
            &
            \Cl^{\cout\times\cin}
            \arrow[r, "H"]
            \arrow[d, "\rho_{\Cl}^{\cout\times\cin}(g)"]
            &
            \Hom_{\vecrm}( \Cl^\cin, \Cl^\cout )
            \vphantom{\big|}
            \arrow[d, "\,\rho_{\Hom}(g)"]
            \\
            \Rpq
            \vphantom{\big|}
            \arrow[rr, pos=.5, rounded corners, to path={
            -- ([yshift=-3.2ex]\tikztostart.south)
            --node[below]{\small$K$} ([yshift=-3.2ex]\tikztotarget.south)
            -- (\tikztotarget.south)
            }]
            \arrow[r, "\KB"']
            &
            \Cl^{\cout\times\cin}
            \arrow[r, "H"']
            &
            \Hom_{\vecrm}( \Cl^\cin, \Cl^\cout )
            \vphantom{\big|}
        \end{tikzcd}
     \end{minipage}
     \vspace*{-10pt}
   \caption{%
        \emph{Left:}
        Multi-vector valued output of the kernel-network $\KB$ for 
        ${\cin \mkern-2mu = \mkern-2mu \cout \mkern-2mu = \mkern-2mu 1, (p,\mkern-2muq) \mkern-2mu = \mkern-2mu (1,\mkern-2mu1)}$,
        and its expansion to a full ${\O(1,\mkern-3mu1)}$-steerable kernel via the kernel head $H$.
        \emph{Right:}
        Commutative diagram of the construction and ${\O(p,\mkern-2muq)}$-equivariance of implicit steerable kernels ${K\mkern-3mu=\mkern-2muH \mkern-2mu\circ\mkern-2mu \KB}$,
        composed from a kernel network $\KB$ with ${\cout\!\times\!\cin}$ multivector outputs and the kernel head $H$.
        The two inner squares show the individual equivariance of $\KB$ and $H$, from which the kernels' overall equivariance follows.
        We abbreviate $\Cl(\Rpq)$ by $\Cl$.
    }
    \label{fig:implicit_steerable_kernel_diagram}
    \vspace*{-.37cm}
\end{figure*}
\end{icml_version}

\section{Clifford-Steerable CNNs}
\label{sec:clifford_steerable_CNNs_main}

This section presents \emph{Clifford-Steerable Convolutional Neural Networks} (CS-CNNs),
which operate on multivector fields on $\Rpq$, and are equivariant to the isometry group $\E(p,q)$ of $\Rpq$. 
To achieve $\E(p,q)$-equivariance,
we need to find a way to implement $\O(p,q)$-steerable kernels (\Cref{sec:feature_fields_and_steerable_CNNs_general}),
which we do by leveraging the connection between $\Cl(\Rpq)$ and $\O(p,q)$ presented in \Cref{sec:clifford_algebra_group}.

CS-CNNs process (multi-channel) \emph{multivector fields}
{%
\abovedisplayskip=6pt%
\belowdisplayskip=6pt%
\begin{equation}
    f: \Rpq \to \Cl(\Rpq)^c
\end{equation}
}%
of type $(W,\rho) = (\Cl(\Rpq)^c, \rcl[c])$ with $c \geq 1$ channels.
The representation
{%
\abovedisplayskip=6pt%
\belowdisplayskip=6pt%
\begin{align}
    \rcl[c] \,=\, {\textstyle\bigoplus\nolimits_{i=1}^c} \rcl :\ \O(p,q) \to \GL\!\big(\!\Cl(\Rpq)^c\big)
\end{align}
}%
is given by the action $\rcl$ from \Cref{thm:group_action_on_Clifford_full},
however, applied to each of the $c$ components individually.

Following \Cref{thm:steerable_conv}, our main goal is the construction of a convolution operator
{%
\abovedisplayskip=6pt%
\belowdisplayskip=6pt%
\begin{align}
    \label{eq:conv-op-cl}
    &L:\, \Gamma\big( \Rpq,\Cl(\Rpq)^{\cin} \big) \to \Gamma\big( \Rpq,\Cl(\Rpq)^{\cout} \big), \notag \\
    &L(f_\textup{in})(u) := \int_{\Rpq} K(v)\big[ f_\textup{in}(u-v) \big]\, dv,\mkern10mu
\end{align}
}%
parameterized by a convolution kernel
{%
\abovedisplayskip=6pt%
\belowdisplayskip=6pt%
\begin{equation}
    \label{eq:clifford_kernel}
    K:\, \Rpq \to \Hom_{\vecrm} \!\big(\! \Cl(\Rpq)^\cin, \Cl(\Rpq)^\cout \big)
\end{equation}
}%
that satisfies the following \emph{$\O(p,q)$-steerability (equivariance) constraint}
for every $g\in \O(p,q)$ and $v\in\Rpq$.%
\footnote{
    The volume factor $|\det g|=1$ drops out for $g \in \O(p,q)$.
}%
{%
\abovedisplayskip=-8pt%
\belowdisplayskip=6pt%
\begin{align}
   \label{eq:kernel-constraint-again}
   \\[-4pt] \mkern-6mu \notag
    K(gv) \,\stackrel{!}{=}\,  \rcl[\cout](g)\, K(v)\, \rcl[\cin](g^{\minus1})
    \,=:\, \rho_{\Hom}(g) (K(v)),
\end{align}
}%

As mentioned in \Cref{sec:steerable_CNNs_general}, constructing such $\O(p,q)$-steerable kernels is typically difficult.
To overcome this challenge, we follow \citet{zhdanov2022implicit} and implement the kernels \emph{implicitly}.
Specifically, they are based on $\O(p,q)$-equivariant
\emph{``kernel networks''}%
\footnote{
    The kernel network's output
    $\Cl(\Rpq)^{\cout\cdot\cin}$
    is here reshaped to matrix form
    $\Cl(\Rpq)^{\cout \times \cin}$.
}
{%
\abovedisplayskip=6pt%
\belowdisplayskip=6pt%
\begin{align}
    \KB:\, \Rpq \to \Cl(\Rpq)^{\cout \times \cin} ,
\end{align}
}%
implemented as CGENNs (\Cref{sec:cegnns}).

Unfortunately, the codomain of $\KB$ is $\Cl(\Rpq)^{\cout \times \cin}$
instead of $\Hom_{\vecrm}\!\big(\! \Cl(\Rpq)^\cin, \Cl(\Rpq)^\cout \big)$,
as required by steerable kernels, Eq.~\eqref{eq:clifford_kernel}.
To bridge the gap between~these spaces,
we introduce an ${\O(p,\mkern-2muq)}$-equivariant linear layer, called \emph{kernel head} $H$.
Its purpose is to transform the kernel network's output $\kb := \KB(v)\in \Cl(\Rpq)^{\cout \times \cin}$
into the desired $\Rbb$-linear map between multivector channels
$H(\kb) \in$ $\Hom_{\vecrm} \!\big(\! \Cl(\Rpq)^\cin, \Cl(\Rpq)^\cout \big)$.
The relation between kernel network $\KB$, kernel head $H$, and the resulting steerable kernel $K:=H\circ\KB$ is visualized in
\begin{icml_version}
Fig.\  \ref{fig:implicit_steerable_kernel_diagram} (right).
\end{icml_version}
\begin{arxiv_version}
Figs.\,\ref{fig:implicit_steerable_kernel_visual} and \ref{fig:implicit_steerable_kernel_diagram}.
\end{arxiv_version}

To achieve ${\O(p,\mkern-2muq)}$-equivariance (steerability) of
${K\mkern-3.5mu=\mkern-2.5muH \mkern-3mu\circ\mkern-2mu \KB\mkern-2.5mu,}$
we have to make the kernel head $H$ of a specific form:

\begin{arxiv_version}
\begin{SCfigure*}[.38]
   \begin{tikzcd}[row sep=3.1em, column sep=4.5em]
       \Rpq
       \vphantom{\big|}
       \arrow[rr, pos=.5, rounded corners, to path={
       -- ([yshift=3.4ex]\tikztostart.north)
       --node[above]{\small$K$} ([yshift=3.4ex]\tikztotarget.north)
       -- (\tikztotarget.north)
       }]
       \arrow[r, "\KB"]
       \arrow[d, "g\cdot\,"']
       &
       \Cl(\Rpq)^{\cout\times\cin}
       \arrow[r, "H"]
       \arrow[d, "\rho_{\Cl}^{\cout\times\cin}(g)"]
       &
       \Hom_{\vecrm} \!\big(\mkern-2mu \Cl(\Rpq)^\cin,\, \Cl(\Rpq)^\cout \big) \mkern-16mu
       \arrow[d, "\,\rho_{\Hom}(g)"]
       \\
       \Rpq
       \vphantom{\big|}
       \arrow[rr, pos=.5, rounded corners, to path={
       -- ([yshift=-3.4ex]\tikztostart.south)
       --node[below]{\small$K$} ([yshift=-3.4ex]\tikztotarget.south)
       -- (\tikztotarget.south)
       }]
       \arrow[r, "\KB"']
       &
       \Cl(\Rpq)^{\cout\times\cin}
       \arrow[r, "H"']
       &
       \Hom_{\vecrm} \!\big(\mkern-2mu \Cl(\Rpq)^\cin,\, \Cl(\Rpq)^\cout \big) \mkern-16mu
   \end{tikzcd}
   \captionsetup{width=.96\textwidth}
   \vspace*{-.3cm}
   \caption{%
        ~ Construction~and~${\O(p,\mkern-2muq)}$- equivariance of
        implicit steerable kernels ${K\mkern-3mu=\mkern-2muH \mkern-2mu\circ\mkern-2mu \KB}$,
        which are composed from a \emph{kernel network} $\KB$ with ${\cout\mkern-5mu\times\mkern-5mu\cin}$ multivector outputs and a \emph{kernel head} $H$.
        The whole diagram commutes.
        The two inner squares show the individual equivariance of $\KB$ and $H$,
        from which the kernel's overall equivariance follows.
   }
   \label{fig:implicit_steerable_kernel_diagram}
\end{SCfigure*}
\end{arxiv_version}

\begin{Def}[Kernel head]
\label{def:kernel_head}
    A \emph{kernel head} is a map
    {%
    \abovedisplayskip=6pt%
    \belowdisplayskip=4pt%
    \begin{align}
        H\mkern-4mu:\mkern1mu  \Cl(\Rpq)^{\cout\mkern-1mu\times \cin}
        \! & \to
        \Hom_{\vecrm} \mkern-4mu\big(\mkern-4mu \Cl(\Rpq)^\cin\!, \Cl(\Rpq)^\cout\mkern-1mu \big) \notag \\
    \kb & \mapsto
        H(\kb), 
    \end{align}
    }%
    where the $\Rbb$-linear operator
    \begin{align}
    H(\kb):\, \Cl(\Rpq)^\cin &\to \Cl(\Rpq)^\cout, &
    \fd &\mapsto H(\kb)[\mkern2mu\fd\mkern2mu], \notag
    \end{align} 
    is defined on each output channel $i \in [\cout]$ and grade component $k=0,\dots,d$, by:
    {%
    \abovedisplayskip=-8pt%
    \belowdisplayskip=6pt%
    \begin{align}
        \label{eq:kernel-head-1}
        \\[-4pt] \mkern-6mu \notag
        H(\kb)[\mkern2mu\fd\mkern2mu]_i^{(k)}
         & := \sum\nolimits_{\substack{j \in [\cin] \\m,n=0,\dots,d}}
        w_{mn,ij}^k \mkern-1mu\cdot\mkern-2mu \lp \kb_{ij}^{(m)} \gp\mkern2mu \fd_{j}^{(n)}\rp^{\!(k)} \mkern6mu
    \end{align}
    }%
    $m,n=0,\dots,d$ label grades and ${j\in [\cin]}$ input channels. 
    The $w_{mn,ij}^k \in \Rbb$ are parameters that allow for weighted mixing between grades and channels.

\end{Def}

Our implementation of the kernel head is discussed in \Cref{sec:kernel_head}.
Note that the kernel head $H$ can be seen as a linear combination of partially evaluated geometric product layers $P^{(k)}(\kb_{ij}, \cdot )$ from \eqref{eq:geom-prod-layer}, which mixes input channels to get the output channels. 
The specific form of the kernel head $H$ comes from the following, most important property:

\begin{Prp}[Equivariance of the kernel head]
\label{prp:kernel-head-equiv}
The~\emph{kernel head} $H$ is $\O(p,q)$-equivariant w.r.t.\ $\rcl[\cout\times\cin]$ and $\rho_{\Hom}$, i.e.\ for $g \in \O(p,q)$ and $\kb \in \Cl(\Rpq)^{\cout\times\cin}$ we have:
  {%
  \abovedisplayskip=5pt
  \belowdisplayskip=2pt
  \begin{align}
      H\big( \rcl[\cout\times\cin](g) (\mkern2mu\kb) \big)
       & \,=\, \rho_{\Hom}(g) (H(\mkern2mu\kb)\mkern2mu). 
  \end{align}
  }%
\end{Prp}
\vspace*{-9pt}
\begin{proof}
The proof relies on the $\O(p,q)$-equivariance of the geometric product and of linear combinations within grades.
It can be found in the Appendix in \Cref{prp:kernel-head-equiv:prf}.
\end{proof}
\vspace*{-8pt}

With these obstructions out of the way, we can now give the core definition of this paper:

\begin{Def}[Clifford-steerable kernel]
    \label{def:Clifford-steerable_kernel}
    A \emph{Clifford-steerable kernel} $K$ is a map as in Eq.\ \eqref{eq:clifford_kernel} that factorizes as: $K=H \circ \KB$ with a \emph{kernel head} $H$ from Eq.\ \eqref{eq:kernel-head-1} and a  \emph{kernel network} $\KB$
    given by
    a Clifford group equivariant neural network (CGENN)%
    \footnote{%
        More generally we could employ any $\O(p,q)$-equivariant neural network $\KB$ w.r.t.\ the standard action ${\rho(g)\mkern-2mu=\mkern-1.5mug}$ and $\rcl[\cout\times\cin]$.
    } 
    from \Cref{thm:cgenn}:
    {%
    \abovedisplayskip=6pt
    \belowdisplayskip=0pt
    \begin{align}
        \label{eq:kernel-base-1}
        \KB = [\KB_{ij}]_{\substack{i \in [\cout] \\j \in [\cin]}}:
        \Rpq \to \Cl(\Rpq)^{\cout \times \cin}.
    \end{align}
    }%
\end{Def}

The main theoretical result of this paper is that Clifford-steerable kernels are always $\O(p,q)$-steerable:

\begin{Thm}[Equivariance of Clifford-steerable kernels]
    \label{thm:kernel-head-equiv}
  Every \emph{Clifford-steerable kernel}
  ${K\mkern-2mu=\mkern-2mu H \mkern-2mu\circ\mkern-2mu \KB}$
  is $\O(p,q)$-steerable w.r.t.\ the standard action
  ${\rho(g)\mkern-2mu=\mkern-1.5mug}$ and $\rho_{\Hom}$:
  {%
  \abovedisplayskip=4pt%
  \belowdisplayskip=0pt%
  \begin{align*}
      K(gv) %
      \,=\, \rho_{\Hom}(g) (K(v))
      \mkern24mu\forall\ g\in\O(p,\mkern-1.5muq),\ v\in\Rpq
  \end{align*}
  }%
\end{Thm}
\vspace*{-11pt}
\begin{proof}
    $\KB$ and $H$ are $\O(p,q)$-equivariant by Definition/The-orem \ref{thm:cgenn} and \Cref{prp:kernel-head-equiv}, respectively.
    The $\O(p,q)$-equivariance of the composition $K=H \circ \KB$ then follows from
    Fig.\ \ref{fig:implicit_steerable_kernel_diagram}           %
    or by direct calculation:
    {%
    \abovedisplayskip=4pt%
    \belowdisplayskip=-0pt%
    \begin{align}
        K(gv)\ &=\ H\big( \KB(gv) \big) \\
            &=\ H\big( \rcl[\cout\times\cin](g) (\KB(v)) \big) \notag \\
            &=\ \rho_{\Hom}(g) \lp H\big(\KB(v)\big) \rp \notag \\
            &=\ \rho_{\Hom}(g) \big( K(v) \big). \notag \qedhere
    \end{align}
    }%
\end{proof}
\vspace*{-8pt}

\begin{arxiv_version}
A direct Corollary of \Cref{thm:kernel-head-equiv} and \Cref{thm:steerable_conv} is now the following desired result.
\end{arxiv_version}
\begin{icml_version}
A direct Corollary of \Cref{thm:kernel-head-equiv} and \Cref{thm:steerable_conv} is:
\end{icml_version}

\begin{Cor}
    \!Let ${K \mkern-3mu=\mkern-3mu H \mkern-3mu\circ\mkern-2.5mu \KB}$ be a \emph{Clifford-steerable~kernel}.
    The corresponding convolution operator $L$ (Eq. \eqref{eq:conv-op-cl}) is then $\E(p,q)$-equivariant,
    i.e. $\forall\ f_\jin \in \Gs\mkern-1mu\big(\Rpq, \Cl(\Rpq)^\cin \big)$:
    {%
    \abovedisplayskip=6pt
    \belowdisplayskip=6pt
    \begin{align*}
        (t,g) \lact L(f_\jin) & = L\big( (t,g) \lact f_\jin \big)
        \mkern24mu\forall\, (t,\mkern-2mug) \in \E(p,q)
    \end{align*}
    }%
\end{Cor}

\begin{Def}[Clifford-steerable CNN]
    We call a convolutional network (that operates on multivector fields and~is) based on Clifford-steerable kernels a \emph{Clifford-Steerable Convolutional Neural Network} (CS-CNN).
\end{Def}

\begin{Rem}
    \citet{Brandstetter2022CliffordNL} use a similar kernel head $H$ as ours, Eq.~\eqref{eq:kernel-head-1}.
    However, their kernel network $\KB$ is not $\O(p,\mkern-1muq)$-equivariant,
    making their overall architecture merely
    equivariant to translations instead of $\E(p,\mkern-1muq)$.
\end{Rem}
\begin{Rem}
    The vast majority of parameters of
    CS-CNNs
    reside in their kernel networks $\KB$.
    Further parameters are found in the kernel heads' weighted geometric product operation and summation of steerable biases to scalar grades.
\end{Rem}
\begin{Rem}
    While CS-CNNs are formalized in \emph{continuous space},
    they are in practice typically applied to \emph{discretized fields}.
    Our implementation allows for any sampling points,
    thus covering both \emph{pixel grids} and \emph{point clouds}.
\end{Rem}

\Cref{sec:cs-cnn-prm} generalizes CS-CNNs from flat spacetimes to general curved \emph{pseudo-Riemannian manifolds}.
\Cref{apx:implementation_details} provides details on our implementation of CS-CNNs,
available at
\url{https://github.com/maxxxzdn/clifford-group-equivariant-cnns}.

\section{Experimental Results}
\label{sec:experimental_results}

\begin{icml_version}
\begin{figure*}[ht]
    \vspace*{-4pt}
    \centering
    \hspace*{-.01\textwidth}
    \includegraphics[width=1.02\textwidth]{figures/experiments-v5.pdf}
    \vspace*{-22pt}
    \caption{
        \emph{Plots 1\,\&\,2: }
        Mean squared errors (MSEs)
        on the Navier-Stokes 2D and Maxwell 3D forecasting tasks (one-step loss) as a function of number of training simulations.
        \emph{Plot 3: } Convergence (test loss) of our model vs. a basic ResNet on the relativistic Maxwell task.
        \emph{Plot 4:}
        Relative $\O(2)$-equivariance errors of different models.
        $G$-FNOs fail as they cannot correctly ingest multivector data.
    }
    \label{fig:results_1}
    \vspace*{-12pt}
\end{figure*}
\end{icml_version}

\begin{arxiv_version}
\begin{figure*}[ht]
    \vspace*{-4pt}
    \centering
    \hspace*{-.02\textwidth}
    \includegraphics[width=1.02\textwidth]{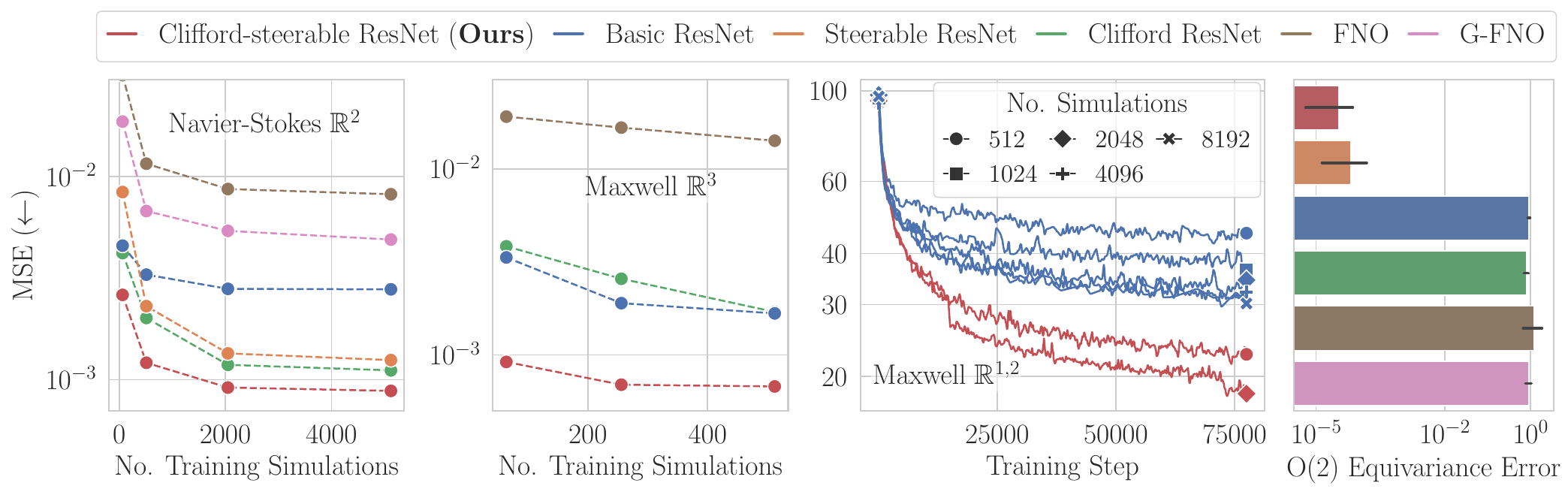}
    \vspace*{-21pt}
    \caption{
        \emph{Plots 1\,\&\,2: }
        Mean squared errors (MSEs)
        on the Navier-Stokes $\Rbb^2$ and Maxwell $\Rbb^3$ forecasting tasks (one-step loss) as a function of number of training simulations.
        \emph{Plot 3: }
        MSE test loss convergence of our model vs. a basic ResNet on the relativistic Maxwell $\Rbb^{1,2}$ task.
        The ResNet does not match the performance of CS-CNNs even for vastly larger training datasets.
        \emph{Plot 4:}
        Relative $\O(2)$-equivariance errors of models trained on Navier-Stokes $\Rbb^2$.
        $G$-FNOs fail as they cannot correctly ingest multivector data.
    }
    \label{fig:results_1}
    \vspace*{-4pt}
\end{figure*}
\end{arxiv_version}

To assess CS-CNNs,
we investigate how well they can learn to simulate dynamical systems by testing their ability to predict future states given a history of recent states \cite{gupta2022pdearena}.
We consider
three tasks:
\begin{icml_version}
\begin{enumerate}[leftmargin=18pt, topsep=-2pt, itemsep=-5pt, label={(\arabic*)}]
    \item Fluid dynamics  on $\Rbb^2\mkern28mu$     ({\small incompressible Navier-Stokes})
    \item Electrodynamics on $\Rbb^3\mkern16mu$     ({\small Maxwell's Eqs.})
    \item Electrodynamics on $\Rbb^{1,2}\mkern5mu$  ({\small Maxwell's Eqs., relativistic})
\end{enumerate}
\end{icml_version}
\begin{arxiv_version}
\begin{enumerate}[leftmargin=18pt, topsep=-2pt, itemsep=-2pt, label={(\arabic*)}]
    \item Fluid dynamics  on $\Rbb^2\mkern28mu$     ({\small incompressible Navier-Stokes})
    \item Electrodynamics on $\Rbb^3\mkern16mu$     ({\small Maxwell's Eqs.})
    \item Electrodynamics on $\Rbb^{1,2}\mkern5mu$  ({\small Maxwell's Eqs., relativistic})
\end{enumerate}
\end{arxiv_version}
Only the last setting is properly incorporating time into
${1\mkern-4mu+\mkern-3mu2}$-dimensional \emph{spacetime},
while the former two are treating time steps improperly as \emph{feature channels}.
The improper setting allows us to compare our method with prior work,
which was not able to incorporate the full spacetime symmetries $\E(1,n)$,
but only the spatial subgroup $\E(n)$ (which is also covered by CS-CNNs).

\begin{icml_version}
\begin{figure}
    \centering
    \includegraphics[width=1.\columnwidth]{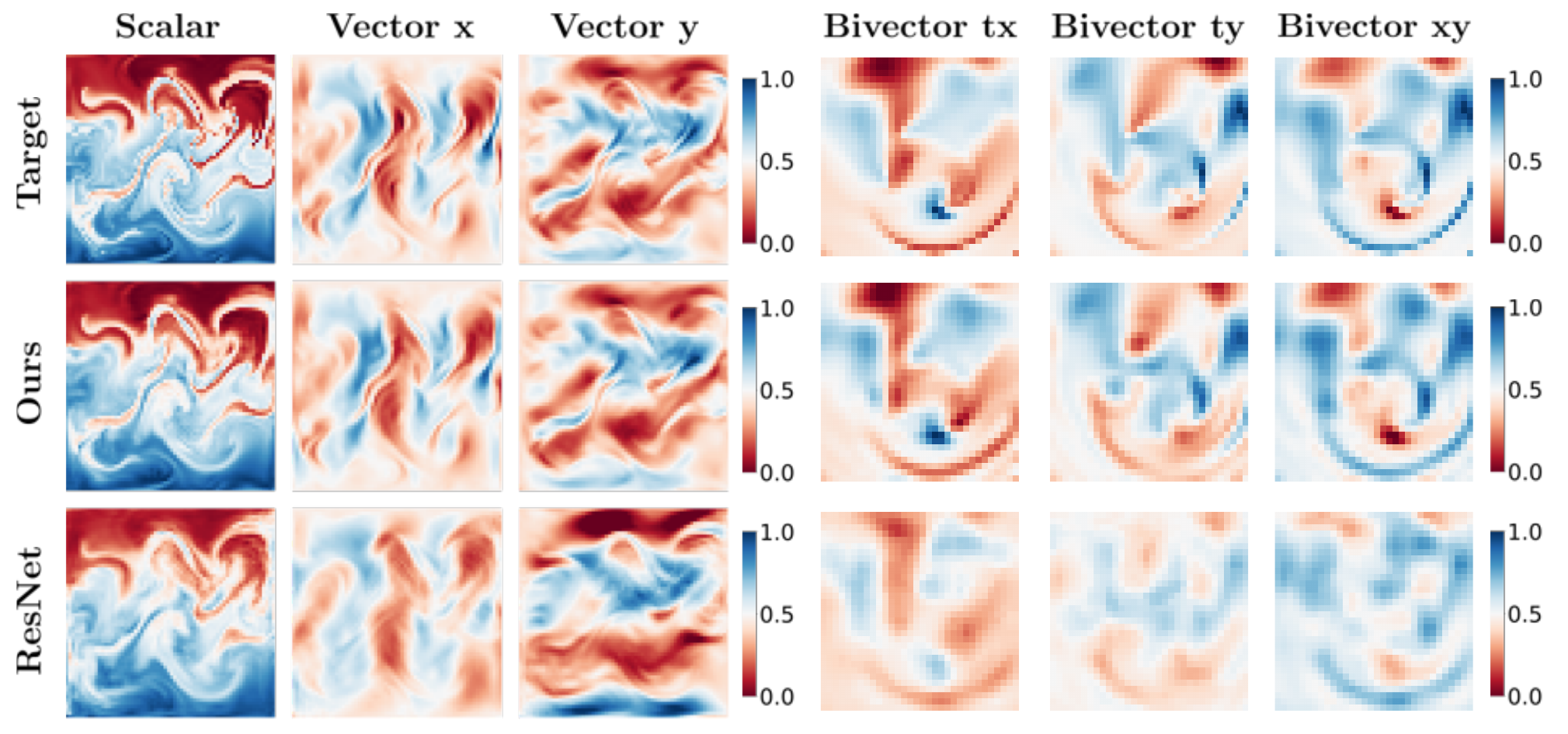}
    \vspace*{-23pt}
    \caption{
        Visual comparison of target and predicted fields.
        \emph{Left}:
        Our CS-ResNet clearly produces better results than the basic ResNet on Navier Stokes,
        despite only being trained on $64$ instead of $5120$ simulations.
        \emph{Right}:
        On Maxwell 2D+1, CS-ResNets capture crisp details like wavefronts more accurately. 
    }
    \label{fig:pdevis}
    \vspace{-14pt}
\end{figure}
\end{icml_version}

\begin{arxiv_version}
\begin{SCfigure*}
    \centering
    \includegraphics[width=1.52\columnwidth]{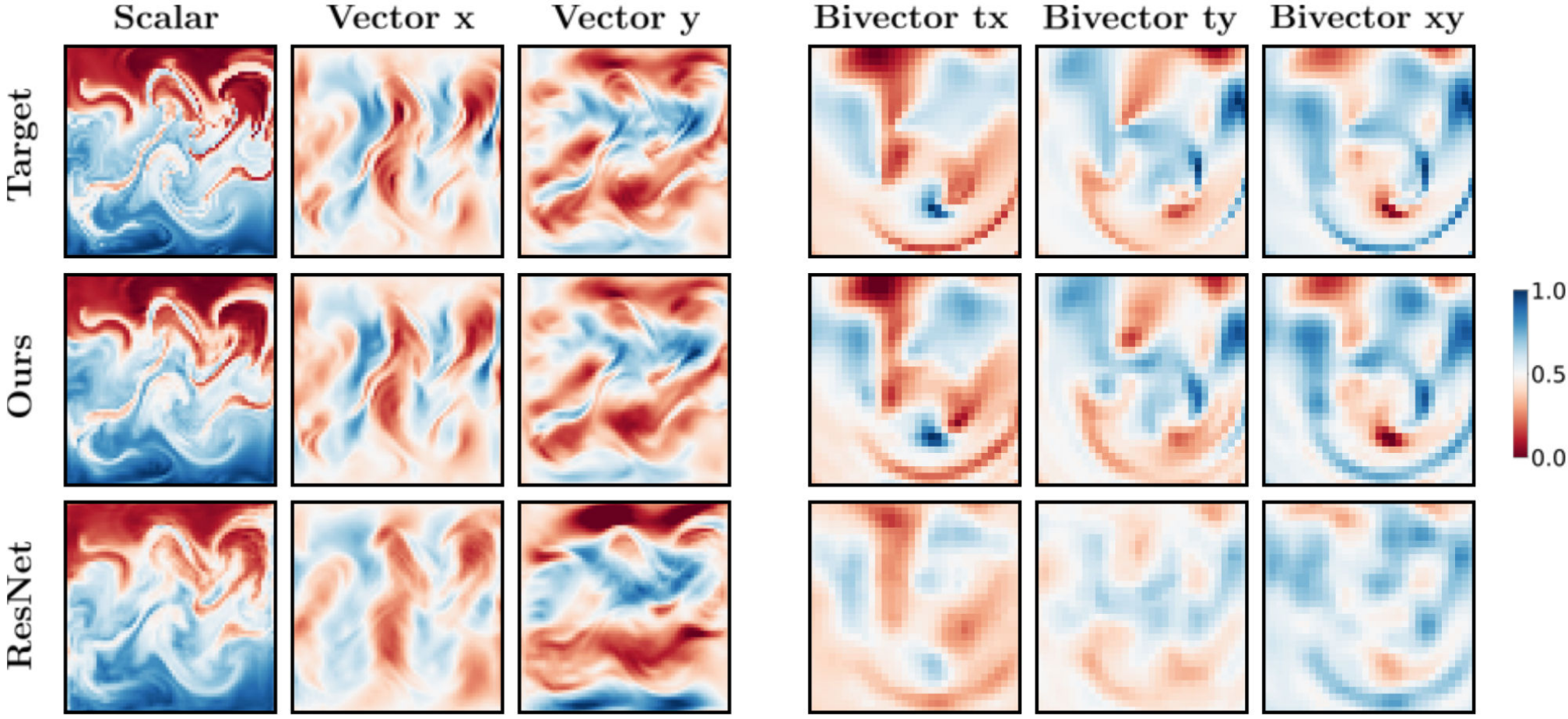}
    \caption{
        Visual comparison of target and predicted fields.
        \emph{Left}:
        Our CS-ResNet clearly produces better results than the basic ResNet on Navier Stokes $\Rbb^2$,
        despite only being trained on $64$ instead of $5120$ simulations.
        \emph{Right}:
        On the relativistic Maxwell simulation task on $\Rbb^{1,2}$,
        CS-ResNets capture crisp details like wavefronts more accurately. 
        This is since they generalize over any isometries of space and any boosted frames of reference.
    }
    \label{fig:pdevis}
\end{SCfigure*}
\end{arxiv_version}

\textbf{Data \& Tasks:}
For both tasks (1) and (2), the goal is to predict the next state given previous 4 time steps. 
In (1), the inputs are scalar pressure and vector velocity fields. 
In (2) the inputs are vector electric and bivector magnetic fields. 
For task (3), the goal is to predict 16 future states given the previous 16 time steps. 
In this case, the \emph{entire electromagnetic field forms a bivector} \cite{Orban2020DimensionalSO}. 
Individual training samples are randomly sliced from long simulations.
More details on the datasets are found in Appendix~\ref{sec:appendix-datasets}. 

\begin{samepage}
\textbf{Architectures:}
We evaluate six network architectures:
\begin{center}
    \vspace*{-4pt}
    \small
    \setlength{\tabcolsep}{3.5pt}
    \begin{tabular}{ r c c }
        \toprule
        architecture                            & matrix group $G$      & isometry group   \\
        \midrule
        Conventional ResNet                     & $\{e\}$               & translations     \\
        Clifford ResNet                         & $\{e\}$               & translations     \\
        Fourier Neural Operators                & $\{e\}$               & translations     \\
        \!$G$-Fourier Neural Operators          & ${\mathrm{D}_4\mkern-2mu<\mkern-1mu\O(2)}$  & $\approx \E(2)$  \\
        Steerable ResNet                        & $\O(n)$               & $\E(n)$          \\
        \underline{Clifford-Steerable ResNet}   & $\O(p,q)$             & $\E(p,q)$        \\
        \bottomrule
    \end{tabular}
\begin{icml_version}
    \vspace*{-10pt}
\end{icml_version}
\begin{arxiv_version}
\end{arxiv_version}
\end{center}
\end{samepage}
    \pagebreak

The basic ResNet model is described in Apx.~\ref{apx:exp_details}.
Clifford,
Steerable,
and our CS-ResNets
are variations of it that substitute vanilla convolutions with their
Clifford \cite{Brandstetter2022CliffordNL},
$\O(n)$-steerable \cite{Weiler2019_E2CNN,cesa2021ENsteerable},
and Clifford-Steerable counterparts, respectively.
We also test Fourier Neural Operators (FNO) \cite{li2020fourier} and $G$-FNO \cite{helwig2023GFNO}.
The latter add equivariance to the Dihedral group ${\mathrm{D}_4\mkern-2mu<\mkern-2mu\O(2)}$.
Assuming scalar or regular representations, they are incapable~of digesting multivector-valued data.
We address this by replacing the initial lifting and final projection with \emph{unconstrained} operations that are able to learn a geometrically correct mapping from/to multivectors.
All models scale their number of channels to match the parameter count of the basic ResNet.

\textbf{Results:}
To evaluate the models, we report mean-squared error losses (MSE) on test sets.
As shown in Fig.~\ref{fig:results_1},
our \emph{CS-ResNets outperform all baselines on all tasks},
especially in higher dimensional space(time)s $\Rbb^3$ and $\Rbb^{1,2}$.
CS-ResNets are extremely sample-efficient:
for the Navier-Stokes experiment, they
require only $64$ training simulations to outperform the basic ResNet and FNOs \emph{trained on 80$\times$ more data}.
On Maxwell $\Rbb^{1,2}$ the basic ResNet does not manage to come close to the CS-ResNet's performance when supplied with 16$\times$ more data.

Plot 1 proves CS-CNNs to be a good alternative to classical $\O(2)$-steerable CNNs in the nonrelativistic case.
We~didn't run $\O(3)$-steerable CNNs on Maxwell
$\Rbb^3$
due to resource constraints
and on
$\Rbb^{1,2}$
as they are not Lorentz-equivariant.
$G$-FNO does not support either of these symmetries.

The Maxwell data on spacetime $\Rbb^{1,2}$ is naturally modeled by \emph{space-time algebra} $\Cl(\Rbb^{1,2})$ \cite{hestenes2015space}.
Contrary to tasks (1) and (2), time appears here as a proper grid dimension, not as a feature channel.
The light cone structure of CS-CNN kernels (Fig. \ref{fig:implicit_steerable_kernel_diagram})
ensures the models' consistency
across different inertial frames of reference.
This is relevant as the simulated electromagnetic fields are induced by particles moving at relativistic velocities.
We see in Plot~3 that CS-CNNs converge significantly faster and are more sample efficient than basic ResNets.

\begin{arxiv_version}
Fig.~\ref{fig:pdevis} visualizes predictions of CS-ResNets and basic ResNets on Navier-Stokes $\Rbb^2$ and Maxwell $\Rbb^{1,2}$.
Our model is much more accurately capturing fine details, despite being trained on less data.
\end{arxiv_version}

\vspace{-1pt}
\textbf{Equivariance error:}
To assess the models' $\E(2)$-equivari- ance, we measure the relative error
$\frac{\left| f(g.x) - g.f(x) \right|}{\left| f(g.x) + g.f(x) \right|}$ between
(1) the output computed from a transformed input; and
(2) the transformed output, given the original input. 
As shown in Fig. \ref{fig:results_1} (right),
both steerable models are equivariant up to numerical artefacts.
Despite training,
the other models did not become equivariant at all.
This holds in particular for $G$-FNO, which covers only a subgroup of discrete rotations.

\section{Conclusions}

We presented Clifford-Steerable CNNs,
a new theoretical framework for ${\E(p,\mkern-1.5muq)}$-equivariant convolutions on pseudo-Euclidean spaces such as Minkowski-spacetime.
CS-CNNs process fields of multivectors -- geometric features which naturally occur in many areas of physics.
The required ${\O(p,\mkern-1.5muq)}$-steerable convolution kernels are implemented implicitly via Clifford group equivariant neural networks.
This makes so far unknown analytic solutions for the steerability constraint unnecessary.
CS-CNNs significantly outperform baselines on a variety of physical dynamics tasks.
\begin{icml_version}
Some limitations of CS-CNNs are discussed in Appendix~\ref{sec:limitations_ICML-version}.
\end{icml_version}

\begin{arxiv_version}
The practically most relevant novel setting unlocked by CS-CNNs are relativistic convolutions on spacetimes $\Rbb^{1,q}$.
Related to this is a branch of research concerned with developing Lorentz group equivariant networks for \emph{jet tagging},
i.e. for the task of binary classifying whether a given jet of elementary particles measured in an accelerator originated from hadronically decaying top quarks \cite{kasieczka2017deep}.
A crucial difference to our work is that jets are numerically represented as \emph{sets} of scalars and momentum 4-vectors that are not associated with specific locations in spacetime, i.e. that are \emph{not fields or point clouds}.%
\footnote{
    These sets are sometimes claimed to be point clouds,
    however, this is misleading as
    they consist of points in \emph{feature space}
    instead of features attached to points in (Minkowski) \emph{base space}.
}
Consequently, such data is not processed by $\E(1,3)$-equivariant CNNs,
but rather by $\O(1,3)$ or $\mathrm{SO}^+(1,3)$-equivariant
MLPs \cite{Bogatskiy2020LorentzGE,finzi2021practical},
GNNs \cite{Gong2022AnEL,ruhe2023CliffordGroupEquivariantNNs}, or
transformers \cite{spinner2024lorentz},
or by models relying on a complete set of pairwise Lorentz-invariants 
\cite{villar2021scalarsAreUniversal,bogatskiy2022PELICAN,li2024doesLorentzSymmetricDesign}.
While these models are not immediately suitable for processing fields (or point clouds) on spacetime,
some of them could be used for parameterizing the implicit kernel networks within our CS-CNNs.
\end{arxiv_version}

\begin{arxiv_version}
From the viewpoint of general steerable CNNs, there are some limitations:
\vspace*{-2pt}
\begin{itemize}[leftmargin=8pt, topsep=-3pt, itemsep=0pt]
    \item
    There exist \emph{more general field types} (${\O(p,\mkern-1.5muq)}$-rep-resentations) than multivectors,
    for which CS-CNNs do not provide steerable kernels.
    For connected Lie groups, e.g. the subgroups ${\SO^+\mkern-2mu(p,\mkern-1.5muq)}$,
    these types can in principle be computed numerically \cite{shutty2020learning}.
    
    \item
    CGENNs and CS-CNNs rely on equivariant operations that treat multivector-grades $\Cl^{(k)}(V,\eta)$ as ``atomic'' features.
    However, it is not clear whether grades are always \emph{irreducible} representations,
    that is, there might be further equivariant degrees of freedom
    which would treat irreducible sub-representations independently.
    
    \item
    We observed that the steerable kernel spaces of CS-CNNs are not necessarily \emph{complete},
    i.e., certain degrees of freedom might be missing.
    However, we show in Apx.~\ref{apx:kernel_space_completeness} how they are recovered by composing multiple convolutions.
    
    \item
    $\O(p,q)$ and their group orbits on $\Rpq$ are for $p,q\neq0$ \emph{non-compact};
    for instance, the hyperboloids in spacetimes $\Rbb^{1,q}$ extend to infinity.
    In practice, we sample convolution kernels on a finite sized grid as shown in
    Fig.~\ref{fig:implicit_steerable_kernel_visual}.
    This introduces a cutoff,
    breaking equivariance for large transformations.
    Note that this is an issue not specific to CS-CNNs, but it applies e.g. to scale-equivariant CNNs as well \cite{bekkers2020bspline,romero2020wavelet}.
\end{itemize}

Despite these limitations, CS-CNNs excel in our experiments.
A major advantage of CGENNs and CS-CNNs is that they allow for a simple, unified implementation for arbitrary signatures ${(p,\mkern-2muq)}$.
This is remarkable, since steerable kernels usually need to be derived for each symmetry group individually.
Furthermore, our implementation applies both to multivector fields sampled on pixel grids and point clouds.

CS-CNNs are, to the best of our knowledge,
the first convolutional networks that respect the full symmetries ${\E(p,\mkern-1.5muq)}$ of fields on Minkowski spacetime or any other pseudo-Euclidean spaces.
Even more generally, CS-CNNs are readily extended to arbitrary curved \emph{pseudo-Riemannian manifolds},
and such convolutions will necessarily rely on ${\O(p,\mkern-1.5muq)}$-steerable kernels.
For more details see Appendix~\ref{sec:cs-cnn-prm} and \cite{weiler2023EquivariantAndCoordinateIndependentCNNs}.
They could furthermore be adapted to
\emph{steerable PDOs} (partial differential operators) \cite{jenner2021steerablePDO},
which would connect them to the \emph{multivector calculus} used in mathematical physics
\cite{hestenes1968multivector,hitzer2002multivector,lasenby1993multivector}.
\end{arxiv_version}

\begin{icml_version}
CS-CNNs are, to the best of our knowledge,
the first convolutional networks respecting the full symmetries ${\E(p,\mkern-1.5muq)}$~of pseudo-Euclidean spaces.
They are readily extended to general \emph{pseudo-Riemannian manifolds}; see Apx.~\ref{sec:cs-cnn-prm}~and~\cite{weiler2023EquivariantAndCoordinateIndependentCNNs}.
They could furthermore be adapted to
\emph{steerable partial differential operators} \cite{jenner2021steerablePDO},
connecting them to \emph{multivector calculus}
\cite{hestenes1968multivector,hitzer2002multivector,lasenby1993multivector}.%
\pagebreak%
\end{icml_version}%

\label{submission}

\begin{icml_version}
\cleardoublepage
\end{icml_version}
\begin{arxiv_version}
\end{arxiv_version}

\section*{Impact Statement}

The broader implications of our work are primarily in the improved modeling of
PDEs, other physical systems, or multi-vector based applications in computational geometry.
Being able to model such systems more \textit{accurately} can lead to better understanding about the physical systems governing our world, while being able to model such systems more \textit{efficiently} could greatly improve the ecological footprint of training ML models for modeling physical systems.

\section*{Acknowledgements}
This research was supported by Microsoft Research AI4Science. 
All content represents the opinion of the authors, which is not necessarily shared or endorsed by their respective employers/sponsors.

\bibliographystyle{icml2024}
\bibliography{bibliography}

\begin{thebibliography}{57}
\providecommand{\natexlab}[1]{#1}
\providecommand{\url}[1]{\texttt{#1}}
\expandafter\ifx\csname urlstyle\endcsname\relax
  \providecommand{\doi}[1]{doi: #1}\else
  \providecommand{\doi}{doi: \begingroup \urlstyle{rm}\Url}\fi

\bibitem[Batatia et~al.(2022)Batatia, Kov{\' a}cs, Simm, Ortner, and Cs{\' a}nyi]{batatia2022MACE}
Batatia, I., Kov{\' a}cs, D.~P., Simm, G. N.~C., Ortner, C., and Cs{\' a}nyi, G.
\newblock Mace: Higher {Order} {Equivariant} {Message} {Passing} {Neural} {Networks} for {Fast} and {Accurate} {Force} {Fields}.
\newblock In \emph{Conference on {Neural} {Information} {Processing} {Systems} ({NeurIPS})}, 2022.

\bibitem[Bekkers(2020)]{bekkers2020bspline}
Bekkers, E.
\newblock B-spline {CNNs} on {Lie} groups.
\newblock \emph{International Conference on Learning Representations (ICLR)}, 2020.

\bibitem[Bekkers et~al.(2018)Bekkers, Lafarge, Veta, Eppenhof, Pluim, and Duits]{bekkers2018roto}
Bekkers, E.~J., Lafarge, M.~W., Veta, M., Eppenhof, K. A.~J., Pluim, J. P.~W., and Duits, R.
\newblock Roto-{Translation} {Covariant} {Convolutional} {Networks} for {Medical} {Image} {Analysis}.
\newblock In \emph{International {Conference} on {Medical} {Image} {Computing} and {Computer}-{Assisted} {Intervention} ({MICCAI})}, 2018.

\bibitem[Bogatskiy et~al.(2020)Bogatskiy, Anderson, Offermann, Roussi, Miller, and Kondor]{Bogatskiy2020LorentzGE}
Bogatskiy, A., Anderson, B.~M., Offermann, J.~T., Roussi, M., Miller, D.~W., and Kondor, R.
\newblock Lorentz group equivariant neural network for particle physics.
\newblock In \emph{International {Conference} on {Machine} {Learning} ({ICML})}, 2020.
\newblock URL \url{https://api.semanticscholar.org/CorpusID:219531086}.

\bibitem[Bogatskiy et~al.(2022)Bogatskiy, Hoffman, Miller, and Offermann]{bogatskiy2022PELICAN}
Bogatskiy, A., Hoffman, T., Miller, D.~W., and Offermann, J.~T.
\newblock Pelican: permutation equivariant and lorentz invariant or covariant aggregator network for particle physics.
\newblock \emph{Advances in Neural Information Processing Systems}, 2022.

\bibitem[Brandstetter et~al.(2023)Brandstetter, Berg, Welling, and Gupta]{Brandstetter2022CliffordNL}
Brandstetter, J., Berg, R. v.~d., Welling, M., and Gupta, J.~K.
\newblock Clifford {Neural} {Layers} for {PDE} {Modeling}.
\newblock In \emph{International {Conference} on {Learning} {Representations} ({ICLR})}, 2023.

\bibitem[Brehmer et~al.(2023)Brehmer, Haan, Behrends, and Cohen]{brehmer2023geometric}
Brehmer, J., Haan, P.~d., Behrends, S., and Cohen, T.~S.
\newblock Geometric {Algebra} {Transformer}.
\newblock In \emph{Conference on {Neural} {Information} {Processing} {Systems} ({NeurIPS})}, 2023.

\bibitem[Cesa et~al.(2022)Cesa, Lang, and Weiler]{cesa2021ENsteerable}
Cesa, G., Lang, L., and Weiler, M.
\newblock A {Program} to {Build} {E}({N})-{Equivariant} {Steerable} {CNNs}.
\newblock In \emph{International {Conference} on {Learning} {Representations} ({ICLR})}, 2022.

\bibitem[Cohen \& Welling(2016)Cohen and Welling]{Cohen2016-GCNN}
Cohen, T. and Welling, M.
\newblock Group {Equivariant} {Convolutional} {Networks}.
\newblock In \emph{International {Conference} on {Machine} {Learning} ({ICML})}, pp.\  2990--2999, 2016.

\bibitem[Cohen et~al.(2019{\natexlab{a}})Cohen, Weiler, Kicanaoglu, and Welling]{cohen2019gaugeIco}
Cohen, T., Weiler, M., Kicanaoglu, B., and Welling, M.
\newblock Gauge {Equivariant} {Convolutional} {Networks} and the {Icosahedral} {CNN}.
\newblock In \emph{International {Conference} on {Machine} {Learning} ({ICML})}, pp.\  1321--1330, 2019{\natexlab{a}}.

\bibitem[Cohen \& Welling(2017)Cohen and Welling]{Cohen2017-STEER}
Cohen, T.~S. and Welling, M.
\newblock Steerable {CNNs}.
\newblock In \emph{International {Conference} on {Learning} {Representations} ({ICLR})}, 2017.

\bibitem[Cohen et~al.(2019{\natexlab{b}})Cohen, Geiger, and Weiler]{Cohen2019-generaltheory}
Cohen, T.~S., Geiger, M., and Weiler, M.
\newblock A {General} {Theory} of {Equivariant} {CNNs} on {Homogeneous} {Spaces}.
\newblock In \emph{Conference on {Neural} {Information} {Processing} {Systems} ({NeurIPS})}, 2019{\natexlab{b}}.

\bibitem[Filipovich \& Hughes(2022)Filipovich and Hughes]{filipovich2022PyCharge}
Filipovich, M.~J. and Hughes, S.
\newblock Pycharge: an open-source python package for self-consistent electrodynamics simulations of lorentz oscillators and moving point charges.
\newblock \emph{Computer Physics Communications}, 274:\penalty0 108291, 2022.

\bibitem[Finzi et~al.(2020)Finzi, Stanton, Izmailov, and Wilson]{finzi2020generalizing}
Finzi, M., Stanton, S., Izmailov, P., and Wilson, A.~G.
\newblock Generalizing {Convolutional} {Neural} {Networks} for {Equivariance} to {Lie} {Groups} on {Arbitrary} {Continuous} {Data}.
\newblock In \emph{International {Conference} on {Machine} {Learning} ({ICML})}, pp.\  3165--3176, 2020.

\bibitem[Finzi et~al.(2021)Finzi, Welling, and Wilson]{finzi2021practical}
Finzi, M., Welling, M., and Wilson, A.~G.
\newblock A {Practical} {Method} for {Constructing} {Equivariant} {Multilayer} {Perceptrons} for {Arbitrary} {Matrix} {Groups}.
\newblock In \emph{International {Conference} on {Machine} {Learning} ({ICML})}, 2021.

\bibitem[Geiger et~al.(2020)Geiger, Smidt, Alby, Miller, Boomsma, Dice, Lapchevskyi, Weiler, Tyszkiewicz, Batzner, et~al.]{geiger2020e3nn}
Geiger, M., Smidt, T., Alby, M., Miller, B.~K., Boomsma, W., Dice, B., Lapchevskyi, K., Weiler, M., Tyszkiewicz, M., Batzner, S., et~al.
\newblock Euclidean neural networks: e3nn.
\newblock \emph{Zenodo. https://doi. org/10.5281/zenodo}, 2020.

\bibitem[Ghosh \& Gupta(2019)Ghosh and Gupta]{ghosh2019scale}
Ghosh, R. and Gupta, A.
\newblock Scale {Steerable} {Filters} for {Locally} {Scale}-{Invariant} {Convolutional} {Neural} {Networks}.
\newblock \emph{ArXiv}, abs/1906.03861, 2019.

\bibitem[Gong et~al.(2022)Gong, Meng, Zhang, Qu, Li, Qian, Du, Ma, and Liu]{Gong2022AnEL}
Gong, S., Meng, Q., Zhang, J., Qu, H., Li, C., Qian, S., Du, W., Ma, Z.-M., and Liu, T.-Y.
\newblock An efficient lorentz equivariant graph neural network for jet tagging.
\newblock \emph{Journal of High Energy Physics}, 2022, 2022.
\newblock URL \url{https://api.semanticscholar.org/CorpusID:246063615}.

\bibitem[Gupta \& Brandstetter(2022)Gupta and Brandstetter]{gupta2022pdearena}
Gupta, J.~K. and Brandstetter, J.
\newblock Towards {Multi}-spatiotemporal-scale {Generalized} {PDE} {Modeling}.
\newblock \emph{ArXiv}, abs/2209.15616, 2022.

\bibitem[Haan et~al.(2021)Haan, Weiler, Cohen, and Welling]{deHaan2020meshCNNs}
Haan, P.~d., Weiler, M., Cohen, T., and Welling, M.
\newblock Gauge {Equivariant} {Mesh} {CNNs}: Anisotropic convolutions on geometric graphs.
\newblock In \emph{International {Conference} on {Learning} {Representations} ({ICLR})}, 2021.

\bibitem[Helwig et~al.(2023)Helwig, Zhang, Fu, Kurtin, Wojtowytsch, and Ji]{helwig2023GFNO}
Helwig, J., Zhang, X., Fu, C., Kurtin, J., Wojtowytsch, S., and Ji, S.
\newblock Group {Equivariant} {Fourier} {Neural} {Operators} for {Partial} {Differential} {Equations}.
\newblock In \emph{International {Conference} on {Machine} {Learning} ({ICML})}, 2023.

\bibitem[Hendrycks \& Gimpel(2016)Hendrycks and Gimpel]{hendrycks2016gaussian}
Hendrycks, D. and Gimpel, K.
\newblock Gaussian {Error} {Linear} {Units} ({GELUs}).
\newblock \emph{arXiv: Learning}, 2016.

\bibitem[Hestenes(1968)]{hestenes1968multivector}
Hestenes, D.
\newblock Multivector calculus.
\newblock \emph{J. Math. Anal. Appl}, 24\penalty0 (2):\penalty0 313--325, 1968.

\bibitem[Hestenes(2015)]{hestenes2015space}
Hestenes, D.
\newblock \emph{Space-time algebra}.
\newblock Springer, 2015.

\bibitem[Hitzer(2002)]{hitzer2002multivector}
Hitzer, E.~M.
\newblock Multivector differential calculus.
\newblock \emph{Advances in Applied Clifford Algebras}, 12:\penalty0 135--182, 2002.

\bibitem[Holl et~al.(2020)Holl, Thuerey, and Koltun]{holl2020phiflow}
Holl, P., Thuerey, N., and Koltun, V.
\newblock Learning to {Control} {PDEs} with {Differentiable} {Physics}.
\newblock In \emph{International {Conference} on {Learning} {Representations} ({ICLR})}, 2020.

\bibitem[Jenner \& Weiler(2022)Jenner and Weiler]{jenner2021steerablePDO}
Jenner, E. and Weiler, M.
\newblock Steerable {Partial} {Differential} {Operators} for {Equivariant} {Neural} {Networks}.
\newblock In \emph{International {Conference} on {Learning} {Representations} ({ICLR})}, 2022.

\bibitem[Kasieczka et~al.(2017)Kasieczka, Plehn, Russell, and Schell]{kasieczka2017deep}
Kasieczka, G., Plehn, T., Russell, M., and Schell, T.
\newblock Deep-learning top taggers or the end of qcd?
\newblock \emph{Journal of High Energy Physics}, 2017\penalty0 (5):\penalty0 1--22, 2017.

\bibitem[Kingma \& Ba(2015)Kingma and Ba]{Kingma2014AdamAM}
Kingma, D.~P. and Ba, J.
\newblock Adam: A {Method} for {Stochastic} {Optimization}.
\newblock In \emph{International {Conference} on {Learning} {Representations} ({ICLR})}, volume abs/1412.6980, 2015.

\bibitem[Lang \& Weiler(2021)Lang and Weiler]{lang2020WignerEckart}
Lang, L. and Weiler, M.
\newblock A {Wigner}-{Eckart} {Theorem} for {Group} {Equivariant} {Convolution} {Kernels}.
\newblock In \emph{International {Conference} on {Learning} {Representations} ({ICLR})}, 2021.

\bibitem[Lasenby et~al.(1993)Lasenby, Doran, and Gull]{lasenby1993multivector}
Lasenby, A., Doran, C., and Gull, S.
\newblock A multivector derivative approach to lagrangian field theory.
\newblock \emph{Foundations of Physics}, 23\penalty0 (10):\penalty0 1295--1327, 1993.

\bibitem[Li et~al.(2024)Li, Qu, Qian, Meng, Gong, Zhang, Liu, and Li]{li2024doesLorentzSymmetricDesign}
Li, C., Qu, H., Qian, S., Meng, Q., Gong, S., Zhang, J., Liu, T.-Y., and Li, Q.
\newblock Does lorentz-symmetric design boost network performance in jet physics?
\newblock \emph{Physical Review D}, 109\penalty0 (5):\penalty0 056003, 2024.

\bibitem[Li et~al.(2021)Li, Kovachki, Azizzadenesheli, Liu, Bhattacharya, Stuart, and Anandkumar]{li2020fourier}
Li, Z., Kovachki, N.~B., Azizzadenesheli, K., Liu, B., Bhattacharya, K., Stuart, A.~M., and Anandkumar, A.
\newblock Fourier {Neural} {Operator} for {Parametric} {Partial} {Differential} {Equations}.
\newblock In \emph{International {Conference} on {Learning} {Representations} ({ICLR})}, 2021.

\bibitem[Lindeberg(2009)]{lindeberg2009scaleSpace}
Lindeberg, T.
\newblock Scale-space.
\newblock 2009.

\bibitem[Loshchilov \& Hutter(2017)Loshchilov and Hutter]{Loshchilov2016SGDRSG}
Loshchilov, I. and Hutter, F.
\newblock Sgdr: Stochastic {Gradient} {Descent} with {Warm} {Restarts}.
\newblock In \emph{International {Conference} on {Learning} {Representations} ({ICLR})}, 2017.

\bibitem[Marcos et~al.(2018)Marcos, Kellenberger, Lobry, and Tuia]{marcos2018scale}
Marcos, D., Kellenberger, B., Lobry, S., and Tuia, D.
\newblock Scale equivariance in {CNNs} with vector fields.
\newblock \emph{arXiv preprint arXiv:1807.11783}, 2018.

\bibitem[Orb{\' a}n \& Mira(2021)Orb{\' a}n and Mira]{Orban2020DimensionalSO}
Orb{\' a}n, X.~P. and Mira, J.
\newblock Dimensional scaffolding of electromagnetism using geometric algebra.
\newblock \emph{European Journal of Physics}, 42\penalty0 (1):\penalty0 015204, 2021.

\bibitem[Romero et~al.(2022)Romero, Kuzinna, Bekkers, Tomczak, and Hoogendoorn]{romero2021ckconv}
Romero, D.~W., Kuzinna, A., Bekkers, E.~J., Tomczak, J.~M., and Hoogendoorn, M.
\newblock {CKConv: Continuous Kernel Convolutions for Sequential Data}.
\newblock In \emph{International {Conference} on {Learning} {Representations} ({ICLR})}, 2022.

\bibitem[Romero et~al.(2023)Romero, Bekkers, Tomczak, and Hoogendoorn]{romero2020wavelet}
Romero, D.~W., Bekkers, E., Tomczak, J.~M., and Hoogendoorn, M.
\newblock Wavelet networks: Scale-translation equivariant learning from raw time-series.
\newblock \emph{Transactions on Machine Learning Research}, 2023.

\bibitem[Ruhe et~al.(2023{\natexlab{a}})Ruhe, Brandstetter, and Forr{\' e}]{ruhe2023CliffordGroupEquivariantNNs}
Ruhe, D., Brandstetter, J., and Forr{\' e}, P.
\newblock Clifford {Group} {Equivariant} {Neural} {Networks}.
\newblock In \emph{Conference on {Neural} {Information} {Processing} {Systems} ({NeurIPS})}, volume abs/2305.11141, 2023{\natexlab{a}}.

\bibitem[Ruhe et~al.(2023{\natexlab{b}})Ruhe, Gupta, Keninck, Welling, and Brandstetter]{ruhe2023geometric}
Ruhe, D., Gupta, J.~K., Keninck, S.~D., Welling, M., and Brandstetter, J.
\newblock Geometric {Clifford} {Algebra} {Networks}.
\newblock In \emph{International {Conference} on {Machine} {Learning} ({ICML})}, pp.\  29306--29337, 2023{\natexlab{b}}.

\bibitem[Shutty \& Wierzynski(2022)Shutty and Wierzynski]{shutty2020learning}
Shutty, N. and Wierzynski, C.
\newblock {Computing Representations for Lie Algebraic Networks}.
\newblock \emph{NeurIPS 2022 Workshop on Symmetry and Geometry in Neural Representations}, 2022.

\bibitem[Sosnovik et~al.(2020)Sosnovik, Szmaja, and Smeulders]{Sosnovik2020scale}
Sosnovik, I., Szmaja, M., and Smeulders, A. W.~M.
\newblock Scale-{Equivariant} {Steerable} {Networks}.
\newblock In \emph{International {Conference} on {Learning} {Representations} ({ICLR})}, 2020.

\bibitem[Spinner et~al.(2024)Spinner, Bres{\'o}, de~Haan, Plehn, Thaler, and Brehmer]{spinner2024lorentz}
Spinner, J., Bres{\'o}, V., de~Haan, P., Plehn, T., Thaler, J., and Brehmer, J.
\newblock Lorentz-equivariant geometric algebra transformers for high-energy physics.
\newblock \emph{arXiv preprint arXiv:2405.14806}, 2024.

\bibitem[Villar et~al.(2021)Villar, Hogg, Storey-Fisher, Yao, and Blum-Smith]{villar2021scalarsAreUniversal}
Villar, S., Hogg, D.~W., Storey-Fisher, K., Yao, W., and Blum-Smith, B.
\newblock Scalars are universal: Equivariant machine learning, structured like classical physics.
\newblock \emph{Advances in Neural Information Processing Systems}, 34:\penalty0 28848--28863, 2021.

\bibitem[Wang et~al.(2021)Wang, Walters, and Yu]{wang2020incorporating}
Wang, R., Walters, R., and Yu, R.
\newblock Incorporating {Symmetry} into {Deep} {Dynamics} {Models} for {Improved} {Generalization}.
\newblock In \emph{International {Conference} on {Learning} {Representations} ({ICLR})}, 2021.

\bibitem[Wang(2022)]{wang2022extensions}
Wang, S.
\newblock Extensions to the navier--stokes equations.
\newblock \emph{Physics of Fluids}, 34\penalty0 (5), 2022.

\bibitem[Weiler \& Cesa(2019)Weiler and Cesa]{Weiler2019_E2CNN}
Weiler, M. and Cesa, G.
\newblock General {E}(2)-{Equivariant} {Steerable} {CNNs}.
\newblock In \emph{Conference on {Neural} {Information} {Processing} {Systems} ({NeurIPS})}, pp.\  14334--14345, 2019.

\bibitem[Weiler et~al.(2018{\natexlab{a}})Weiler, Geiger, Welling, Boomsma, and Cohen]{3d_steerableCNNs}
Weiler, M., Geiger, M., Welling, M., Boomsma, W., and Cohen, T.
\newblock 3d {Steerable} {CNNs}: Learning {Rotationally} {Equivariant} {Features} in {Volumetric} {Data}.
\newblock In \emph{Conference on {Neural} {Information} {Processing} {Systems} ({NeurIPS})}, pp.\  10402--10413, 2018{\natexlab{a}}.

\bibitem[Weiler et~al.(2018{\natexlab{b}})Weiler, Hamprecht, and Storath]{weiler2018SFCNN}
Weiler, M., Hamprecht, F.~A., and Storath, M.
\newblock Learning {Steerable} {Filters} for {Rotation} {Equivariant} {CNNs}.
\newblock In \emph{Computer {Vision} and {Pattern} {Recognition} ({CVPR})}, 2018{\natexlab{b}}.

\bibitem[Weiler et~al.(2021)Weiler, Forr{\'e}, Verlinde, and Welling]{weiler2021coordinateIndependent}
Weiler, M., Forr{\'e}, P., Verlinde, E., and Welling, M.
\newblock {Coordinate Independent Convolutional Networks -- Isometry and Gauge Equivariant Convolutions on Riemannian Manifolds}.
\newblock \emph{arXiv preprint arXiv:2106.06020}, 2021.

\bibitem[Weiler et~al.(2023)Weiler, Forré, Verlinde, and Welling]{weiler2023EquivariantAndCoordinateIndependentCNNs}
Weiler, M., Forré, P., Verlinde, E., and Welling, M.
\newblock \emph{{Equivariant and Coordinate Independent Convolutional Networks}}.
\newblock 2023.
\newblock URL \url{https://maurice-weiler.gitlab.io/cnn_book/EquivariantAndCoordinateIndependentCNNs.pdf}.

\bibitem[Worrall \& Welling(2019)Worrall and Welling]{Worrall2019DeepScaleSpaces}
Worrall, D.~E. and Welling, M.
\newblock Deep {Scale}-spaces: Equivariance {Over} {Scale}.
\newblock In \emph{Conference on {Neural} {Information} {Processing} {Systems} ({NeurIPS})}, pp.\  7364--7376, 2019.

\bibitem[Wu \& He(2018)Wu and He]{wu2018group}
Wu, Y. and He, K.
\newblock Group {Normalization}.
\newblock In \emph{European {Conference} on {Computer} {Vision} ({ECCV})}, pp.\  3--19, 2018.

\bibitem[Zhang \& Williams(2022)Zhang and Williams]{zhang2022similarity}
Zhang, X. and Williams, L.~R.
\newblock Similarity equivariant linear transformation of joint orientation-scale space representations.
\newblock \emph{arXiv preprint arXiv:2203.06786}, 2022.

\bibitem[Zhdanov et~al.(2023)Zhdanov, Hoffmann, and Cesa]{zhdanov2022implicit}
Zhdanov, M., Hoffmann, N., and Cesa, G.
\newblock Implicit {Convolutional} {Kernels} for {Steerable} {CNNs}.
\newblock In \emph{Conference on {Neural} {Information} {Processing} {Systems} ({NeurIPS})}, 2023.

\bibitem[Zhu et~al.(2022)Zhu, Qiu, Calderbank, Sapiro, and Cheng]{zhu2019scale}
Zhu, W., Qiu, Q., Calderbank, A.~R., Sapiro, G., and Cheng, X.
\newblock Scaling-{Translation}-{Equivariant} {Networks} with {Decomposed} {Convolutional} {Filters}.
\newblock \emph{Journal of Machine Learning Research (JMLR)}, 23:\penalty0 68:1--68:45, 2022.

\end{thebibliography}

\clearpage
\appendix

\twocolumn[  
   \begin{@twocolumnfalse}
      \vspace*{3ex}
      {\LARGE\bf Appendix}
      \vspace*{4ex}
   \end{@twocolumnfalse}
]

\section{Implementation details}
\label{apx:implementation_details}

This appendix provides details on the implementation of CS-CNNs.%
\footnote{\url{https://github.com/maxxxzdn/clifford-group-equivariant-cnns}}

Before detailing the Clifford-steerable kernels and convolutions, we first define the following ``kernel shell'' operation, which is used twice in the final kernel computation.
Recall that given the base space $\Rpq$ equipped with the inner product $\ip^{p,q}$, we have a Clifford algebra $\Cl(\Rpq)$.
We want to compute a kernel that maps from $\cin$ multivector input channels to $\cout$ multivector output channels, i.e., 
\begin{align}
   K:\, \Rpq \to \Hom_{\vecrm} \!\big(\! \Cl(\Rpq)^\cin, \Cl(\Rpq)^\cout \big)\,.
\end{align}
$K$ is defined on any $v \in \Rpq$, which allows to model \emph{point clouds}.
In this work, however, we sample it on a grid of shape $X_1, \dots, X_{p+q}$,
analogously to typical CNNs.

\subsection{Clifford Embedding}
We briefly discuss how one is able to embed scalars and vectors into the Clifford algebra.
This extends to other grades such as bivectors.

Let $s \in \Rbb$ and $v \in \Rpq$.
Using the natural isomorphisms
$\mathcal{E}^{(0)}: \Rbb \xrightarrow{\sim} \Cl(\Rpq)^{(0)}$ and
$\mathcal{E}^{(1)}: \Rpq \xrightarrow{\sim} \Cl(\Rpq)^{(1)}$,
we embed the scalar and vector components into a multivector as
{%
\abovedisplayskip=-2pt%
\belowdisplayskip=6pt%
\begin{align}
   m \,:=\, \mathcal{E}^{(0)}(s) + \mathcal{E}^{(1)}(v)\ \ \in\ \Cl(\Rpq)\,.
\end{align}
}%
This is a standard operation in Clifford algebra computations, where we leave the other components of the multivector zero.
We denote such embeddings in the algorithms provided below jointly as ``$\textsc{cl\_embed}([s,v])$''.

\subsection{Scalar Orbital Parameterizations}
Note that the $\O(p,q)$-steerability constraint
\begin{align*}
    K(gv) \,\stackrel{!}{=}\,  \rcl[\cout](g)\, K(v)\, \rcl[\cin](g^{\minus1})
    \,=:\, \rho_{\Hom}(g) (K(v))
    \\[2pt] \forall\ v\in\Rpq,\ g\in\O(p,q)
\end{align*}
couples kernel values \emph{within} but \emph{not across} different $\O(p,q)$-orbits
{%
\abovedisplayskip=-2pt%
\belowdisplayskip=6pt%
\begin{align}
    \O(p,q).v \,:=&\,\ \{ gv \,|\, g\in \O(p,q) \} \\[3pt]
    =&\,\ \{ w \,|\, \eta(w,w)=\eta(v,v) \} \,. \notag
\end{align}
}%
The first line here is the usual definition of group orbits,
while the second line makes use of the Def.~\ref{def:Opq} of pseudo-orthogonal groups as metric-preserving linear maps.

\begin{figure}
\vspace{-10pt}
\begin{algorithm}[H]
   \caption{\textsc{ScalarShell}}
   \label{alg:kernel_shell}
   \begin{algorithmic}
      \STATE {\bfseries input} $\ip^{p,q}$, $v \in \Rpq$, $\sigma$.
      \STATE $s \gets \sgn \left( \ip^{p,q}(v, v) \right) \cdot \exp \left( -\frac{ \left | \ip^{p,q}(v, v) \right | }{2\sigma^2} \right)$
      \STATE {\bfseries return} $s$
   \end{algorithmic}
\end{algorithm}

\vspace{-12pt}
\begin{algorithm}[H]
   \caption{\textsc{CliffordSteerableKernel}}
   \label{alg:cskernel}
   \begin{algorithmic}
      \STATE {\bfseries input}
      $p, q$
      $\Lambda$, $\cin$, $\cout$, $\left( v_n \right)_{n=1}^N \in \Rpq$, $\mathrm{CGENN}$
      \STATE {\bfseries output} $\kb \in \Rbb^{(\cout \cdot 2^{d}) \times (\cin \cdot 2^{d}) \times X_1 \times \dots \times X_{p+q}}$
      \STATE
      \STATE \algcom{Weighted Cayley.}
      \FOR{$i=1\dots\cin$, $o=1\dots\cout$, $a, b, c=1\dots p+q$}
      \STATE $w^{c}_{oiab} \sim \mathcal{N}(0, \frac{1}{\sqrt{\cin \cdot N}})$ \algcom{Weight init.}
      \STATE $W^{c}_{oiab} \gets \Lambda_{ab}^c \cdot w^{c}_{oiab}$
      \ENDFOR
      \STATE 
      \STATE $\sigma \sim \mathcal{U}(0.4, 0.6)$ \algcom{Init if needed.}
      \STATE \algcom{Compute scalars.}
      \STATE $ s_n \leftarrow \textsc{ScalarShell}(\ip^{p,q}, v_n, \sigma)$
      \STATE \algcom{Embed $s$ and $v$ into a multivector.}
      \STATE $x_n \leftarrow \textsc{cl\_embed}\left(\left[ s_n, v_n \right]\right)$ 
      \STATE 
      \STATE \algcom{Evaluate kernel network.}
      \STATE $\kb_n^{io}:= \mathrm{CGENN}\left( x_n \right)$
      \STATE
      \STATE \algcom{Reshape to kernel matrix.}
      \STATE $\kb \gets \textsc{reshape}\left( \kb, (N, \cout, \cin) \right)$
      \STATE
      \STATE \algcom{Compute kernel mask.}
      \FOR{$i=1\dots\cin$, $o=1\dots\cout$, $k=0\dots p+q$}
      \STATE $\sigma_{kio} \sim \mathcal{U}(0.4, 0.6)$ \algcom{Init if needed.}
      \STATE $s_{noi}^{{k}} \leftarrow \textsc{ScalarShell}(\ip^{p,q}, v_n, \sigma_{kio})$ 
      \ENDFOR
      \STATE
      \STATE $\kb_{noi}^{(k)} \gets  \kb_{noi}^{(k)} \cdot s_{noi}^{k}$ \algcom{Mask kernel.}
      \STATE
      \STATE \algcom{Kernel head.}
      \STATE $\kb_{noib}^{c} \gets \sum_{a=1}^{2^{d}} \kb_{noi}^{a} \cdot W^{c}_{oiab}$ \algcom{Partial weighted geometric product.}
      \STATE 
      \STATE \algcom{Reshape to final kernel.}
      \STATE $\kb \gets \textsc{reshape}\left( \kb, \left(\cout \cdot 2^{d}, \cin \cdot 2^{d}, X_1, \dots, X_{p+q} \right)\right) $
      \STATE {\bfseries return} $\kb$
   \end{algorithmic}
\end{algorithm}

\vspace{-12pt}
\begin{algorithm}[H]
   \caption{\textsc{CliffordSteerableConvolution}}
   \label{alg:csconv}
   \begin{algorithmic}
      \STATE {\bfseries input} $F_{\mathrm{in}}$, $\left(v_n\right)_{n=1}^N$, \textsc{Args}
      \STATE {\bfseries output} $F_{\mathrm{out}}$
      \STATE $F_{\mathrm{in}} \leftarrow$ \textsc{reshape}$(F_{\mathrm{in}}, (B, \cin \cdot 2^{d}, Y_1, \dots, Y_{p+q}))$
      \STATE \mbox{$\kb \leftarrow$ \textsc{CliffordSteerableKernel}$(\left(v_n\right)_{n=1}^N, \textsc{Args})$}
      \STATE $F_{\mathrm{out}} \gets \textsc{Conv}(F_{\mathrm{in}}, \kb)$
      \STATE $F_{\mathrm{out}} \leftarrow$ \textsc{reshape}$(F_{\mathrm{out}}, (B, \cout, Y_1, \dots, Y_{p+q}, 2^{d}))$
      \STATE {\bfseries return} $F_{\mathrm{out}}$
   \end{algorithmic}
\end{algorithm}
\end{figure}

In the positive-definite case of $\O(n)$, this means that the only degree of freedom is the radial distance from the origin, resulting in (hyper)spherical orbits.
Examples of such kernels can be seen in Fig.~\ref{fig:implicit_steerable_kernel_visual_O2}.
Other radial kernels are obtained typically through e.g. Gaussian shells, Bessel functions, etc.

In the nondefinite case of $\O(p,q)$, the orbits are hyperboloids, 
resulting in hyperboloid shells for e.g. the Lorentz group $\O(1,3)$
as in
\begin{icml_version}
Fig.~\ref{fig:implicit_steerable_kernel_diagram} (left).
\end{icml_version}
\begin{arxiv_version}
Fig.~\ref{fig:implicit_steerable_kernel_visual}.
\end{arxiv_version}
In this case, we extend the input to the kernel with a scalar component that now relates to the hyperbolic (squared) distance from the origin.

Specifically, we define an exponentially decaying $\ip^{p,q}$-induced (parameterized) scalar \emph{orbital shell} (analogous to the radial shell of typical Steerable CNNs) in the following way.
We parameterize a kernel width $\sigma$ and compute the shell value as
\begin{align}
   s_\sigma(v) = \sgn \left( \ip^{p,q}(v, v) \right) \cdot \exp \left( -\frac{ \left | \ip^{p,q}(v, v) \right | }{2\sigma^2} \right).
\end{align}
The width $\sigma \sim \mathcal{U}(0.4, 0.6)$ is, inspired by \cite{cesa2021ENsteerable}, initialized with a uniform distribution.
Since $\ip^{p,q}(v, v)$ can be negative in the nondefinite case, we take the absolute value and multiply the result by the sign of $\ip^{p,q}(v, v)$.
Computation of the kernel shell (\textsc{ScalarShell}) is outlined in Function \ref{alg:kernel_shell}.
Intuitively, we obtain exponential decay for points far from the origin.
However, the sign of the inner product ensures that we clearly disambiguate between ``light-like'' and ``space-like'' points.
I.e., they are close in Euclidean distance but far in the $\ip^{p,q}$-induced distance.
Note that this choice of parameterizing scalar parts of the kernel is not unique and can be experimented with.

\subsection{Kernel Network}
Recall from Section \ref{sec:clifford_steerable_CNNs_main} that the kernel $K$ is parameterized by a \emph{kernel network}, which is a map
\begin{align}
   \KB:\, \Rpq \to \Cl(\Rpq)^{\cout \times \cin}
\end{align}
implemented as an $\O(p,q)$-equivariant CGENN.
It consists of (linearly weighted) geometric product layers followed by multivector activations.

Let $\{v_n\}_{n=1}^N$ be a set 
of sampling points, where $N := X_1 \cdot \ldots \cdot X_{p+q}$.
In the remainder, we leave iteration over $n$ implicit and assume that the operations are performed for each $n$.
We obtain a sequence of scalars using the kernel shell
{%
\abovedisplayskip=-2pt%
\belowdisplayskip=6pt%
\begin{align}
   s_n &:= s_{\sigma}(v_n) \,.
\end{align}
}%

The input to the kernel network is a batch of multivectors 
\begin{align}
x_n  := \textsc{cl\_embed}([s_n, v_n]) \,.
\end{align}
I.e., taking $s$ and $v$ together, they form the scalar and vector components of the CEGNN's input multivector.
We found including the scalar component crucial for the correct scaling of the kernel to the range of the grid.

Let $i=1,\dots,\cin$ and $o=1,\dots,\cout$ be a sequence of input and output channels.
We then have the kernel network output
\begin{align}
    \kb_{noi}
   \ :=\ \KB(v_{n})_{oi}
   \ :=\ \mathrm{CGENN}(x_n)_{oi} \,,
\end{align}
where $\kb_{noi} \in \Cl(\Rpq)$ is the output of the kernel network for the input multivector $x_n$ (embedded from the scalar $s_n$ and vector $v_n$).
Once the output stack of multivectors is computed, we reshape it
from shape $(N,\cout\cdot\cin)$ to shape $(N,\cout,\cin)$,
resulting in the kernel matrix
\begin{align}
   \kb \gets \textsc{reshape}\left( \kb, (N, \cout, \cin) \right)\,,
\end{align}
where now $\kb \in \Cl(\Rpq)^{N \times \cout \times \cin}$.
Note that $k_n \in \Cl(\Rpq)^{\cout \times \cin}$ is a matrix of multivectors, as desired.

\subsection{Masking}
We compute a second set of scalars which will act as a mask for the kernel.
This is inspired by Steerable CNNs to ensure that the (e.g., radial) orbits of compact groups are fully represented in the kernel, as shown in Figure \ref{fig:implicit_steerable_kernel_visual_O2}.
However, note that for $\O(p,q)$-steerable kernels
with both $p,q\neq0$
this is never fully possible since $\O(p,q)$ is in general not compact,
and all orbits except for the origin extend to infinity.
This can e.g. be seen in the hyperbolic-shaped kernels in Figure \ref{fig:implicit_steerable_kernel_diagram}.

For equivariance to hold in practice, whole orbits would need to be present in the kernel,
which is not possible if the kernel is sampled on a grid with finite support.
This is not specific to our architecture, but is a consequence of the orbits' non-compactness.
The same issue arises e.g. in \emph{scale-equivariant} CNNs
\cite{romero2020wavelet,Worrall2019DeepScaleSpaces,ghosh2019scale,Sosnovik2020scale,bekkers2020bspline,zhu2019scale,marcos2018scale,zhang2022similarity}.
Further experimenting is needed to understand the impact of
truncating the kernel
on the final performance of the model.

We invoke the kernel shell function again to compute a mask for each $k=0,\dots,{p+q}$, $i=1,\dots,\cin$, $o=1,\dots,\cout$.
That is, we have a weight array $\sigma_{kio}$, initialized identically as earlier, which is reused for each position in the grid.
\begin{align}
   s^{k}_{nio} & :=  s_{\sigma_{kio}}(v_n) \,.
\end{align}

We then mask the kernel by scalar multiplication with the shell, i.e., 
\begin{align}
   \kb_{kio}^{(k)} \gets \kb_{nio}^{(k)} \cdot s_{nio}^{k} \,.
\end{align}

\subsection{Kernel Head}
\label{sec:kernel_head}
Finally, the \emph{kernel head} turns the ``multivector matrices'' into a kernel that can be used by, for example, \texttt{torch.nn.ConvNd} or \texttt{jax.lax.conv}.
This is done by a partial evaluation of a (weighted) geometric product.
Let $\mu, \nu \in \Cl(\Rpq)$ be two multivectors.
Recall that $\dim \Cl(\Rpq) = 2^{p+q} = 2^d$.
\begin{align}
   (\mu \sbull \nu)^C = \sum\nolimits_{A} \sum\nolimits_{B} \mu^{A} \cdot \nu^{B} \cdot \Lambda^C_{AB}\,,
\end{align}
where $A, B, C \subseteq [d]$ are multi-indices running over the $2^d$ basis elements of $\Cl(\Rpq)$.
Here, $\Lambda \in \Rbb^{2^{d} \times 2^{d} \times 2^{d}}$ is the \emph{Clifford multiplication table} of $\Cl(\Rpq)$, also sometimes called a \emph{Cayley table}.
It is defined as 
\begin{align}
   \Lambda_{A, B}^C = \begin{cases} 0 & \text{if } A \triangle B \neq C \\ \sgn^{A, B} \cdot  \bar{\ip}(e_{A \cap B}, e_{A \cap B}) & \text{if } A \triangle B = C \end{cases}\,.
\end{align}
Here, $\triangle$ denotes the symmetric difference of sets, i.e., $A \triangle B = (A \setminus B) \cup (B \setminus A)$.
Further, 
\begin{align}
   \sgn^{A, B} := (-1)^{n_{A, B}},
\end{align}
where $n_{A, B}$ is the number of adjacent ``swaps'' one needs to fully sort the tuple $(i_1, \dots, i_s, j_1, \dots, j_t)$, where $A = \{i_1, \dots, i_s\}$ and $B = \{j_1, \dots, j_t\}$.
In the following, we identify the multi-indices $A$, $B$, and $C$ with a relabeling $a$, $b$, and $c$ that run from $1$ to $2^d$.

Altogether, $\Lambda$ defines a multivector-valued bilinear form which represents the geometric product relative to the chosen multivector basis.
We can weight its entries with parameters $w_{oiab}^{c} \in \Rbb$,
initialized as $w_{oiab}^{c} \sim \mathcal{N}(0, \frac{1}{\sqrt{\cin \cdot N}})$.
These weightings can be redone for each input channel and output channel, as such we have a weighted Cayley table 
$W \in \Rbb^{2^{d} \times 2^{d} \times 2^{d} \times \cin \times \cout}$ 
with entries
\begin{align}
\label{eq:kernel_head_weights}
   W^{c}_{oiab}\ :=\ \Lambda^c_{ab} w^{c}_{oiab} \,.
\end{align}
An ablation study in appendix~\ref{apx:kernel_head_weight_ablation}
demonstrates the great relevance of the weighting parameters empirically.

Given the kernel matrix $\kb$, we compute the kernel by partial (weighted) geometric product evaluation, i.e.,
\begin{align}
   \kb_{noib}^{c} \gets \sum\nolimits_{a=1}^{2^{d}} \kb_{noi}^{a} \cdot W^{c}_{oiab}\,.
\end{align}

Finally, we reshape and permute $\kb_{noib}^{c}$ from shape $(N, \cout, \cin, 2^{d}, 2^d)$ to its final shape, i.e.,
\begin{align*}
   \kb \gets \textsc{reshape}\left( \kb, \left(\cout \cdot 2^{d}, \cin \cdot 2^{d}, X_1, \dots, X_{p+q} \right) \right)\,.
\end{align*}
This is the final kernel that can be used in a convolutional layer, and can be interpreted (at each sample coordinate) as an element of $\Hom_{\vecrm} \!\big(\! \Cl(\Rpq)^\cin, \Cl(\Rpq)^\cout \big)$.
The pseudocode for the Clifford-steerable kernel (\textsc{CliffordSteerableKernel}) is given in Function \ref{alg:cskernel}.

\subsection{Clifford-steerable convolution:}
As defined in \Cref{sec:clifford_steerable_CNNs_main}, Clifford-steerable convolutions can be efficiently implemented with conventional convolutional machinery such as \texttt{torch.nn.ConvNd} or \texttt{jax.lax.conv} (see Function \ref{alg:csconv} (\textsc{CliffordSteerableConvolution}) for pseudocode). 
We now have a kernel $\kb \in \mathbb{R}^{(\cout \cdot 2^{d}) \times (\cin \cdot 2^{d}) \times X_1 \times \dots \times X_{p+q}}$ that can be used in a convolutional layer.
Given batch size $B$, we now reshape the input stack of multivector fields
$(B, \cin, Y_1, \dots, Y_{p+q},  2^d)$ into $(B, \cin \cdot 2^d, Y_1, \dots, Y_{p+q})$.
The output array of shape $(B, \cout \cdot 2^d, Y_1, \dots, Y_{p+q})$ is obtained by convolving the input with the kernel, which is then reshaped to $(B, \cout, Y_1, \dots, Y_{p+q}, 2^d)$, which can then be interpreted as a stack of multivector fields again.

\begin{icml_version}
    
\section{Limitations}
\label{sec:limitations_ICML-version}

From the viewpoint of general steerable CNNs, there are some limitations:
\vspace*{-2pt}
\begin{itemize}[leftmargin=8pt, topsep=-3pt, itemsep=-2pt]
    \item
    There exist \emph{more general field types} (${\O(p,\mkern-1.5muq)}$-rep-resentations) than multivectors,
    for which CS-CNNs do not provide steerable kernels.
    For connected Lie groups, such as the subgroups ${\SO^+(p,\mkern-1.5muq)}$,
    these types can in principle be computed numerically \cite{shutty2020learning}.
    
    \item
    CGENNs and CS-CNNs rely on equivariant operations that treat multivector-grades $\Cl^{(k)}(V,\eta)$ as ``atomic'' features.
    However, it is not clear whether grades are always \emph{irreducible} representations,
    that is, there might be further equivariant degrees of freedom
    which would treat irreducible sub-representations independently.
    
    \item
    We observed that the steerable kernel spaces of CS-CNNs are not necessarily \emph{complete},
    that is, certain degrees of freedom might be missing.
    However, we show in Appendix~\ref{apx:kernel_space_completeness} how they are recovered by composing multiple convolutions.
    
    \item
    $\O(p,q)$ and their group orbits on $\Rpq$ are for $p,q\neq0$ \emph{non-compact};
    for instance, the hyperbolas in spacetimes $\Rbb^{1,q}$ extend to infinity.
    In practice, we sample convolution kernels on a finite sized grid as shown in
    Fig.~\ref{fig:implicit_steerable_kernel_diagram} (left).
    This introduces a cutoff,
    breaking equivariance for large transformations.
    Note that this is an issue not specific to CS-CNNs, but it applies e.g. to scale-equivariant CNNs as well \cite{bekkers2020bspline,romero2020wavelet}.
\end{itemize}

Despite these limitations, CS-CNNs excel in our experiments.
A major advantage of CGENNs and CS-CNNs is that they allow for a simple, unified implementation for arbitrary signatures ${(p,\mkern-2muq)}$.
This is remarkable, since steerable kernels usually need to be derived for each symmetry group individually.
Furthermore, our implementation applies both to multivector fields sampled on pixel grids and point clouds.

\end{icml_version}

\section{Completeness of kernel spaces}
\label{apx:kernel_space_completeness}

In order to not over-constrain the model,
it is essential to parameterize a \emph{complete basis} of ${\O(p,\mkern-2muq)}$-steerable kernels.
Comparing our implicit ${\O(2,\mkern-2mu0)\mkern-1mu=\mkern-1mu\O(2)}$-steerable kernels
with the analytical solution by \cite{Weiler2019_E2CNN},
we find that certain \emph{degrees of freedom are missing};
see Fig.~\ref{fig:implicit_steerable_kernel_visual_O2}.

However, while these degrees of freedom are \emph{missing in a single convolution operation},
they can be \emph{fully recovered by applying two consecutively convolutions}.
This suggests that the overall expressiveness of CS-CNNs is (at least for $\O(2)$) not diminished.
Moreover, two convolutions with kernels $\widehat{K}$ and $K$ can always be expressed as a single convolution with a composed kernel ${\widehat{K} \mkern-3mu*\mkern-3mu K}$.
As visualized below, this composed kernel recovers the full degrees of freedom reported in \cite{Weiler2019_E2CNN}:
\begin{figure}[h!]
\vspace*{-4pt}
    \hfill
    \begin{minipage}[c]{0.12\columnwidth}
        \caption{}
        \label{fig:kernels_composed_convs}
    \end{minipage}
    \hspace{5pt}
    \begin{minipage}[c]{0.75\columnwidth}
        \includegraphics[width=.9\columnwidth]{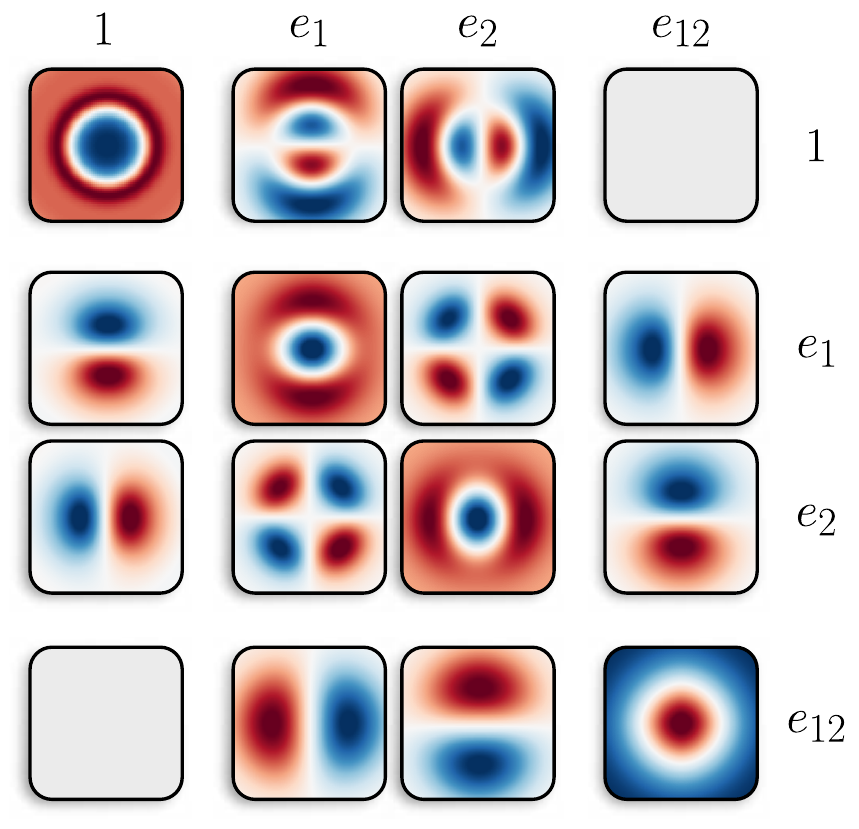}
    \end{minipage}
    \hfill
\vspace*{-10pt}
\end{figure}

The following two sections discuss the initial differences in kernel parametrizations and how they are resolved by adding a second linear or convolution operation.
Unless stated otherwise, we focus here on $\cin=\cout=1$ channels to reduce clutter.

\begin{figure*}
   \centering
   \begin{minipage}{\textwidth}
      \centering
      \begin{minipage}{\textwidth}
         \centering
         \includegraphics[width=.8\linewidth]{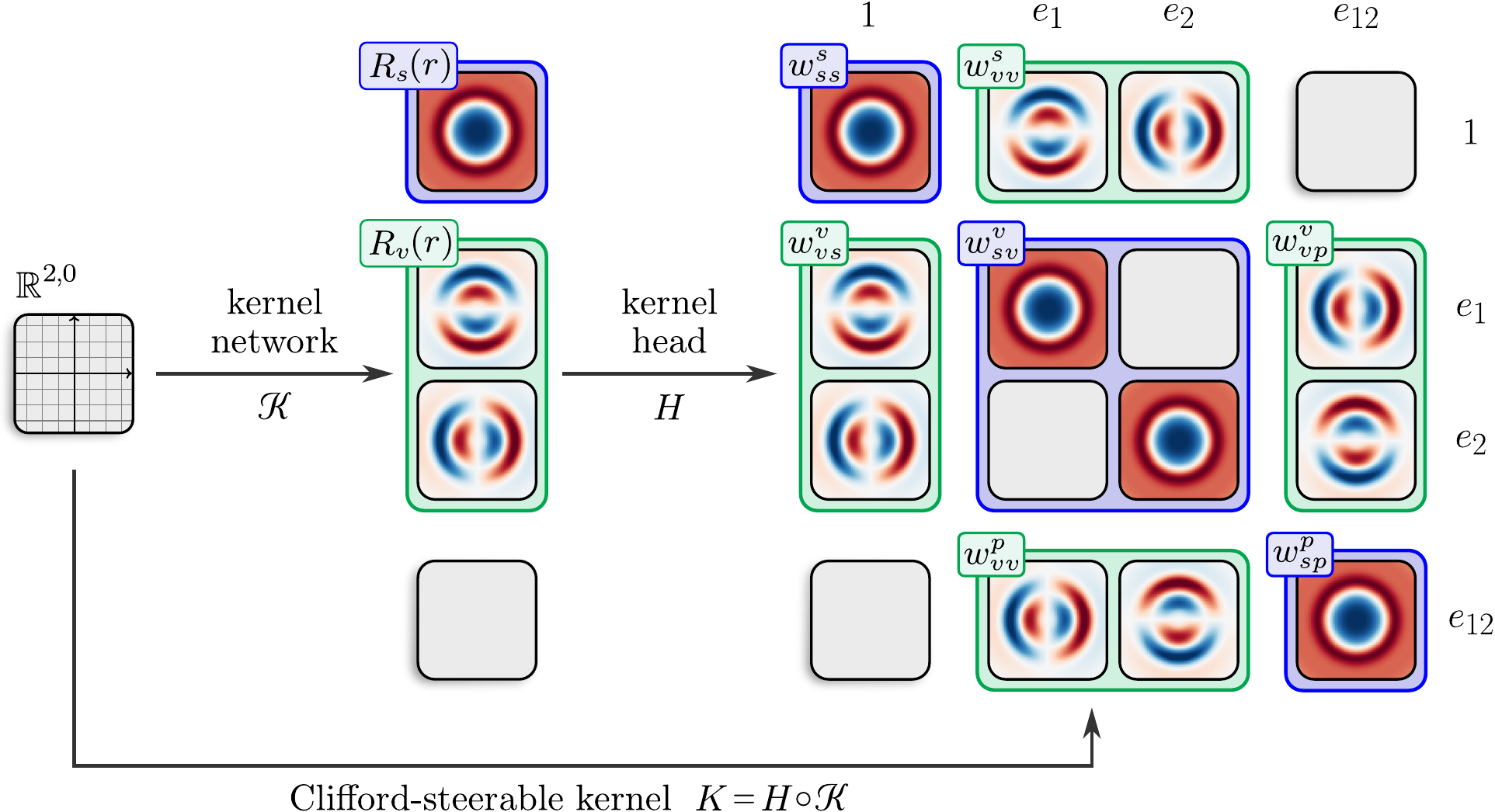}
      \end{minipage}
      \vphantom{foo}\vspace*{20pt}
      \begin{minipage}{\textwidth}
         \centering

{\newcommand\onestrut{\rule[-10pt]{0pt}{25pt}}
\centering%
CS-CNN parametrization
\\\vspace*{4pt}
\small%
\scalebox{1.}{%
        \setlength{\tabcolsep}{5pt}%
        \renewcommand\arraystretch{1.7}%
        \begin{tabu}{r|[1pt]c|c|c}%
               \diagbox[height=14pt]{\raisebox{-3pt}{out}}{\ \ \raisebox{7pt}{in}}
            &  scalar
            &  vector
            &  pseudoscalar
            \\[-4pt]
            &  $1$
            &  $\phantom{\!\bigg|^\top}\big[e_1,e_2\big]^\top$
            &  $e_{12}$
            \\ \tabucline[1pt]{-}
               $1$
            &  $w^s_{ss}{\color{blue}R_s(r)}\, \big[\mkern2mu 1 \mkern2mu\big]$
            &  $w^s_{vv}{\color{Green}R_v(r)}\, \big[\minus\sin(\phi)\, \cos(\phi) \big]$
            &  $\varnothing$
            \onestrut \\[20pt] \hline
                \renewcommand\arraystretch{1.2}%
                $\begin{bmatrix}
                  e_1 \\
                  e_2
                \end{bmatrix}$
            \vphantom{\raisebox{-15pt}{\rule{0pt}{35pt}}}
            &  \renewcommand\arraystretch{1.4}%
               $w^v_{vs}{\color{Green}R_v(r)}\, 
               \begin{bmatrix}
                           \minus \!\sin(\phi) \\
                  \phantom{\minus}\!\cos(\phi)
               \end{bmatrix}$
            &  \renewcommand\arraystretch{1.4}%
               \hspace*{53pt}
               $w^v_{sv}{\color{blue}R_s(r)}\, 
               \begin{bmatrix}
                  \;1 & 0~\\
                  \;0 & 1~
               \end{bmatrix}$
               \hspace*{53pt}
            &  \renewcommand\arraystretch{1.4}%
               $w^v_{vp}{\color{Green}R_v(r)}\, 
               \begin{bmatrix}
                  \cos(\phi) \\
                  \sin(\phi) \\
               \end{bmatrix}$
            \onestrut \\ \hline
               $e_{12}$
            &  $\varnothing$
            &  $w^p_{vv}{\color{Green}R_v(r)}\, \big[\cos(\phi)\,\phantom{-}\sin(\phi) \big]$
            &  $w^p_{sp}{\color{blue}R_s(r)}\, \big[\mkern2mu 1 \mkern2mu\big]$  
        \end{tabu}%
    }%
    }%

      \end{minipage}
      \vphantom{foo}\vspace*{20pt}
      \begin{minipage}{\textwidth}
         \centering
         
{\newcommand\onestrut{\rule[-10pt]{0pt}{25pt}}
\centering%
complete \texttt{e2cnn} parametrization \cite{Weiler2019_E2CNN}
\\\vspace*{10pt}
\small%
\scalebox{1.}{%
        \setlength{\tabcolsep}{5pt}%
        \renewcommand\arraystretch{1.7}%
        \begin{tabu}{c|[1pt]c|c|c}%
               \diagbox[height=14pt]{\raisebox{-3pt}{out}}{\ \ \raisebox{7pt}{in}}
            &  $1$
            &  $\big[e_1,e_2\big]^\top$
            &  $e_{12}$
            \\ \tabucline[1pt]{-}
               $1$
            &  $R^s_{\ s}(r)\, \big[\mkern2mu 1 \mkern2mu\big]$
            &  $R^s_{\ v}(r)\, \big[\minus\sin(\phi)\, \cos(\phi) \big]$
            &  $\varnothing$
            \onestrut \\[20pt] \hline
                \renewcommand\arraystretch{1.2}%
                $\begin{bmatrix}
                  e_1 \\
                  e_2
                \end{bmatrix}$
            \vphantom{\raisebox{-15pt}{\rule{0pt}{35pt}}}
            &  \renewcommand\arraystretch{1.4}%
                \hspace*{1.5pt}
                $R^v_{\ s}(r)\,
                \begin{bmatrix}
                           \minus \!\sin(\phi) \\
                  \phantom{\minus}\!\cos(\phi)
               \end{bmatrix}$
                \hspace*{1.5pt}
            &  \renewcommand\arraystretch{1.4}%
               $R^v_{\ v}(r)\mkern-3mu
               \begin{bmatrix}
                  \;1 & 0~\\
                  \;0 & 1~
               \end{bmatrix}$
               ,$\mkern8mu$
              \renewcommand\arraystretch{1.4}%
              {\color{OrangeRed}
               $\widehat{R}^v_{\ v}(r)\mkern-3mu
               \begin{bmatrix}
                  \cos(2\phi) \mkern-8mu& \phantom{\minus}\sin(2\phi) \\
                  \sin(2\phi) \mkern-8mu&          \minus \cos(2\phi) \\
               \end{bmatrix}$
               }
            &  \renewcommand\arraystretch{1.4}%
                \hspace*{1.5pt}
               $R^v_{\ p}(r)\,
               \begin{bmatrix}
                  \cos(\phi) \\
                  \sin(\phi) \\
               \end{bmatrix}$
                \hspace*{1.5pt}
            \onestrut \\ \hline
               $e_{12}$
            &  $\varnothing$
            &  $R^p_{\ v}(r)\, \big[\cos(\phi)\,\phantom{-}\sin(\phi) \big]$
            &  $R^p_{\ p}(r)\, \big[\mkern2mu 1 \mkern2mu\big]$  
        \end{tabu}%
    }%
    }%

      \end{minipage}
    \end{minipage}
    \vspace*{8pt}
    \caption{%
        Comparison of the parametrization of $\O(2)$-steerable kernels in
        CS-CNNs (top and middle) and \texttt{e2cnn} (bottom).
        While the \texttt{e2cnn} solutions are proven to be \emph{complete},
        CS-CNN seems to miss certain degrees of freedom:
        \\[3pt]
        (1) Their \emph{radial parts are coupled} in the components highlighted in {\color{blue}blue} and {\color{Green}green}, while \texttt{escnn} allows for independent radial parts.
        By ``coupled'' we mean that they are merely scaled relative to each other with weights $w^k_{mn}$
        from the weighted geometric product operation in the kernel head $H$,
        where $m$ labels grade $\KB^{(m)}$ of the kernel network output
        while $n,k$ label input and output grades of the expanded kernel in
        $\Hom_{\vecrm} \!\big(\mkern-2mu \Cl(\Rpq),\, \Cl(\Rpq) \big)$; 
        \\[3pt]
        (2) CS-CNN is missing kernels of angular frequency $2$ that are admissible for mapping between vector fields;
        highlighted in {\color{OrangeRed}red}.
        \\[3pt]
        As explained in Appendix~\ref{apx:kernel_space_completeness},
        these \emph{missing degrees of freedom are recovered} when \emph{composing two convolution layers}.
        A kernel corresponding to the composition of two convolutions in a single one is visualized in Fig.~\ref{fig:kernels_composed_convs}.
   }
   \label{fig:implicit_steerable_kernel_visual_O2}
\end{figure*}

\subsection{Coupled radial dependencies in CS-CNN kernels}

The first issue is that the CS-CNN parametrization implies a \emph{coupling of radial degrees of freedom}.
To make this precise, note that the $\O(2)$-steerability constraint
\begin{align*}
    K(gv) \,\stackrel{!}{=}\,  \rcl[\cout](g)\, K(v)\, \rcl[\cin](g^{\minus1})
    \quad\forall\ v\in\Rbb^2,\, g\in\O(2)
\end{align*}
decouples into independent constraints on individual $\O(2)$-orbits on $\Rbb^2$,
which are rings at different radii (and the origin);
visualized in Fig.~\ref{fig:pe_spaces} (left).
\cite{3d_steerableCNNs,Weiler2019_E2CNN} parameterize the kernel therefore in (hyper)spherical coordinates.
In our case these are \emph{polar coordinates} of $\Rbb^2$, i.e. a radius $r\in\Rbb_{\geq0}$ and angle $\phi\in S^1$:
\begin{align}
   K(r,\phi) \,:=\, R(r) \kappa(\phi)
\end{align}
The $\O(2)$-steerability constraint affects only the angular part
and leaves the radial part entirely free,
such that it can be parameterized in an arbitrary basis or via an MLP.

\paragraph{e2cnn:}
\citet{Weiler2019_E2CNN} solved analytically for complete bases of the angular parts.
Specifically, they derive solutions
\begin{align}
   K^k_{\ n}(r,\phi) \,=\, R^k_{\ n}(r) \kappa^k_{\ n}(\phi)
\end{align}
for any pair of input and output field types (irreps of grades) $n$ and $k$, respectively.
This complete basis of $\O(2)$-steerable kernels is shown in the bottom table of Fig.~\ref{fig:implicit_steerable_kernel_visual_O2}.

\paragraph{CS-CNNs:}
CS-CNNs parameterize the kernel in terms of a kernel network $\KB: \Rpq\to \Cl(\Rpq)^{\cout\times\cin}$,
visualized in Fig.~\ref{fig:implicit_steerable_kernel_visual_O2} (top).
Expressed in polar coordinates, assuming $\cin=\cout=1$,
and considering the independence of $\KB$ on different orbits due to its $\O(2)$-equivariance,
we get the factorization
\begin{align}
   \KB(r,\phi)^{(m)}\ =\ R_m(r) \kappa_m(\phi) \,,
\end{align}
where $m$ is the grade of the multivector-valued output.
As described in Appendix~\ref{sec:kernel_head} (Eq.~\eqref{eq:kernel_head_weights}),
the kernel head operation $H$ expands this output by multiplying it with weights $W^k_{mn} = \Lambda^k_{mn} w^k_{mn}$,
where $w^k_{mn}\in\Rbb$ are parameters and $\Lambda^k_{mn} \in\{\minus1,0,1\}$ represents the geometric product relative to the standard basis of $\Rpq$.
Note that we do not consider multiple in or output channels here.
The final expanded kernel for CS-CNNs is hence given by
\begin{align}
   K^k_{\ n}(r,\phi)\ 
   &=\ \sum_m W^k_{mn} \KB(r,\phi)^{(m)} \\
   &=\ \sum_m \Lambda^k_{mn} w^k_{mn} R_m(r) \kappa_m(\phi) \,.\notag
\end{align}
These solutions are listed in the top table in Fig.~\ref{fig:implicit_steerable_kernel_visual_O2},
and visualized in the graphics above.%
\footnote{
   The parameter $\Lambda^k_{mn}$ appears in the table as selecting to which entry $k,n$ of the table grade $\KB(r,\phi)^{(m)}$ is added (optionally with minus signs).
}

\paragraph{Comparison:}
Note that the complete solutions by \cite{Weiler2019_E2CNN} allow for a
\emph{different radial part $R^k_{\ n}$} for each pair of input and output type (grade/irrep).
In contrast, the CS-CNN parametrization expands \emph{coupled radial parts $R_m$},
additionally multiplying them with weights $w^k_{mn}$
(highlighted in the table in {\color{blue}blue} and {\color{Green}green}).
The CS-CNN parametrization is therefore clearly less general (incomplete).

\paragraph{Solutions:}
One idea to resolve this shortcoming is to make the 
weighted geometric product parameters themselves radially dependent,
\begin{align}
   w_{mn}^k:\ \Rbb_{\geq0}\to\Rbb, \quad r \mapsto w_{mn}^k(r) \,,
\end{align}
for instance by parameterizing the weights with a neural network.
This would fully resolve the under-parametrization,
and would preserve equivariance, since $\O(2)$-steerability depends only on the angular variable.

However, doing this is actually not necessary,
since the missing flexibility of radial parts can always be resolved
by running a convolution followed by a linear layer (or a second convolution)
when $\cout>1$.
The reason for this is that different channels $i=1,\dots,\cout$ of a kernel network $\KB: \Rbb\to \Cl(\Rbb)^{\cout\times\cin}$ do have independent radial parts.
Their convolution responses in different channels can by a subsequent linear layer be mixed with grade-dependent weights.
By linearity, this is equivalent to immediately mixing the channels' radial parts with grade-dependent weights,
resulting in effectively \emph{decoupled radial parts}.

\subsection{Circular harmonics order 2 kernels}

A second issue is that the CS-CNN parametrization is \emph{missing a basis kernel of angular frequency $2$} that maps between vector fields;
highlighted in red in the bottom table of Fig.~\ref{fig:implicit_steerable_kernel_visual_O2}.
However, it turns out that this degree of freedom is reproduced as the difference of two consecutive convolutions ($*$),
one mapping vectors to pseudoscalars and back to vectors,
the other one mapping vectors to scalars and back to vectors,
as suggested in the (non-commutative!) computation flow diagram below:
\begin{equation*}
  \begin{tikzcd}[row sep=.15em, column sep=1.8em]
      & \textup{pseudo}
         \arrow[r, maps to, "*"]
      & \textup{vector}
         \arrow[rd, maps to, ""]
      \\
        \textup{vector}
         \arrow[ru, maps to, "*"]
         \arrow[rd, maps to, "*"']
      & &
      & \mkern-4mu\raisebox{-.6pt}{\scalebox{1.3}{$\ominus$}}\mkern-4mu
         \arrow[r, maps to, ""]
      &[-6pt] \textup{vector}
      \\
      & \textup{scalar}
         \arrow[r, maps to, "*"']
      & \textup{vector}
         \arrow[ru, maps to, ""]
  \end{tikzcd}
\end{equation*}

As background on the angular frequency $2$ kernel, note that
$\O(2)$-steerable kernels between irreducible field types of angular frequencies $j$ and $l$ contain angular frequencies $|j-l|$ and $j+l$ --
this is a consequence of the \emph{Clebsch-Gordan decomposition} of $\O(2)$-irrep tensor products \cite{lang2020WignerEckart}.
We identify
multivector grades $\Cl(\Rbb^{2,0})^{(k)}$ with the following $\O(2)$-irreps:%
\footnote{
    As mentioned earlier, multivector grades may in general not be \emph{irreducible},
    however, for $(p,q)=(2,0)$ they are.
}\footnote{
    There are two different $\O(2)$-irreps corresponding to ${j=0}$ (trivial and sign-flip);
    see \cite{weiler2023EquivariantAndCoordinateIndependentCNNs}[Section 5.3.4].
}
\begin{align*}
    \textup{scalars} \in \Cl(\Rbb^{2,0})^{(0)}
    \ &\leftrightarrow\,\ 
    \textup{trivial irrep (${j\!=\!0}$)}
    \\
    \textup{vectors} \in \Cl(\Rbb^{2,0})^{(1)}
    \ &\leftrightarrow\,\ 
    \textup{defining irrep (${j\!=\!1}$)}
    \\
    \textup{pseudo-scalars} \in \Cl(\Rbb^{2,0})^{(2)}
    \ &\leftrightarrow\,\ 
    \textup{sign-flip irrep (${j\!=\!0}$)}
\end{align*}
Kernels that map vector fields (${j\!=\!1}$) to vector fields (${l\!=\!1}$) should hence contain angular frequencies ${|j\mkern-2mu-\mkern-2mul|=0}$ and ${j\mkern-2mu+\mkern-2mul=2}$.
The latter is missing since $\O(2)$-irreps of order $2$ are not represented by any grade of $\Cl(\Rbb^{2,0})$.

To solve this issue, it seems like one would have to replace the CEGNNs underlying the kernel network $\KB$ with a more general $\O(2)$-equivariant MLP, e.g. \cite{finzi2021practical}.
However, it can as well be implemented as a succession of two convolution operations.
To make this claim plausible, observe first that convolutions are associative,
that is, two consecutive convolutions with kernels $K$ and $\widehat{K}$ are equivalent to a single convolution with kernel $\widehat{K} \ast K$:
\begin{align}
    \widehat{K}\ast \big( K \ast f \big)
    \ =\ \big(\widehat{K} \ast K \big) \ast f
\end{align}
Secondly, convolutions are linear, such that
\begin{align}
    \alpha(\widehat{K}*f) + \beta (K*f)
    \ =\ \big(\alpha\widehat{K} + \beta K \big) \ast f
\end{align}
for any $\alpha,\beta\in\Rbb$.

Using associativity, we can express two consecutive convolutions,
first going from vector to scalar fields via
\begin{align}
    K^s_{\ v}(r,\phi)
    \ =\ R^s_{\ v}(r)
    \begin{pmatrix}
        \minus\sin(\phi) &\mkern-3mu \cos(\phi)
    \end{pmatrix}
\end{align}
then going back from scalars to vectors via
\begin{align}
    K^v_{\ s}(r,\phi)
    \ =\ R^v_{\ s}(r)
    \begin{pmatrix}
        \minus\sin(\phi) \\[1pt] \phantom{\minus}\cos(\phi)
    \end{pmatrix}
\end{align}
as a single convolution between vector fields, where the combined kernel is given by:
\begin{align}
    &\Sigma^v_{\ v} \,:=\, K^v_{\ s} \ast K^s_{\ v}
    \\[6pt]
    &\ =\  
    \begin{pmatrix}
        \inputgraphicskernel{10pt}{30pt}{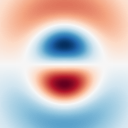} \\
        \inputgraphicskernel{10pt}{30pt}{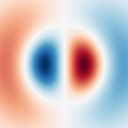}
    \end{pmatrix}
    \mkern-3mu\ast\mkern-3mu
    \begin{pmatrix}
        \inputgraphicskernel{10pt}{30pt}{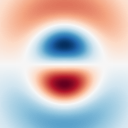} &\mkern-8mu
        \inputgraphicskernel{10pt}{30pt}{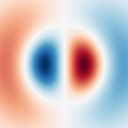}
    \end{pmatrix}
    =
    \begin{pmatrix}
        \inputgraphicskernel{10pt}{22pt}{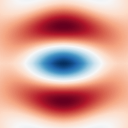} &\mkern-8mu
        \inputgraphicskernel{10pt}{22pt}{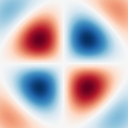} \\
        \inputgraphicskernel{10pt}{22pt}{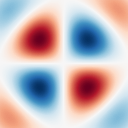} &\mkern-8mu
        \inputgraphicskernel{10pt}{22pt}{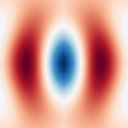}
    \end{pmatrix} \notag
\end{align}

We can similar define a convolution going
from vector to pseudoscalar fields via
\begin{align}
    K^p_{\ v}(r,\phi)
    \ =\ R^p_{\ v}(r)
    \begin{pmatrix}
        \cos(\phi) &\mkern-3mu \sin(\phi)
    \end{pmatrix}
\end{align}
and back to vector fields via
\begin{align}
    K^v_{\ p}(r,\phi)
    \ =\ R^v_{\ p}(r)
    \begin{pmatrix}
        \cos(\phi) \\[1pt] \sin(\phi)
    \end{pmatrix}
\end{align}
as a single convolution with combined kernel:
\begin{align}
    &\Pi^v_{\ v} \,:=\, K^p_{\ v} \ast K^v_{\ p}
    \\
    &\ =\  
    \begin{pmatrix}
        \inputgraphicskernel{10pt}{30pt}{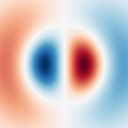} \\
        \inputgraphicskernel{10pt}{30pt}{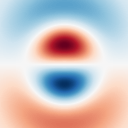}
    \end{pmatrix}
    \mkern-3mu\ast\mkern-3mu
    \begin{pmatrix}
        \inputgraphicskernel{10pt}{30pt}{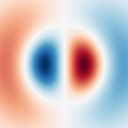} &\mkern-8mu
        \inputgraphicskernel{10pt}{30pt}{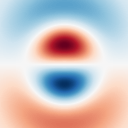}
    \end{pmatrix}
    =
    \begin{pmatrix}
        \inputgraphicskernel{10pt}{22pt}{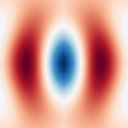} &\mkern-8mu
        \inputgraphicskernel{10pt}{22pt}{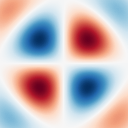} \\
        \inputgraphicskernel{10pt}{22pt}{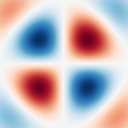} &\mkern-8mu
        \inputgraphicskernel{10pt}{22pt}{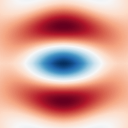}
    \end{pmatrix} \notag
\end{align}

By linearity, we can define yet another convolution between vector fields
by taking the difference of these kernels,
which results in:
\begin{align}
    \Pi^v_{\ v} - \Sigma^v_{\ v}
    \ =\ 
    \begin{pmatrix}
        \inputgraphicskernel{12pt}{26pt}{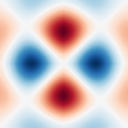} &
        \inputgraphicskernel{12pt}{26pt}{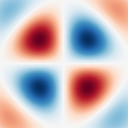} \\
        \inputgraphicskernel{12pt}{26pt}{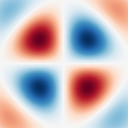} &
        \inputgraphicskernel{12pt}{26pt}{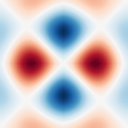}
    \end{pmatrix}
\end{align}

Such kernels parameterize exactly the missing $\O(2)$-steerable kernels of angular frequency $2$;
highlighted in
red
in the bottom table in Fig.~\ref{fig:implicit_steerable_kernel_visual_O2}.
This shows that the missing kernels can be recovered by two convolutions, if required.

The ``visual proof'' by convolving kernels is clearly only suggestive.
To make it precise, it would be required to compute the convolutions of two kernels analytically.
This is easily done by identifying circular harmonics with derivatives of Gaussian kernels;
a relation that is well known in classical computer vision \cite{lindeberg2009scaleSpace}.

\section{Experimental details}
\label{apx:exp_details}

\subsection{Model details:}
For ResNets, we follow the setup of~\citet{wang2020incorporating,Brandstetter2022CliffordNL,gupta2022pdearena}:
the ResNet baselines consist of 8 residual blocks, each comprising two convolution layers with $7 \times 7$ (or $7 \times 7 \times 7$ for 3D) kernels, shortcut connections, group normalization~\citep{wu2018group}, and GeLU activation functions~\citep{hendrycks2016gaussian}.
We use two embedding and two output layers, i.e., the overall architectures could be classified as Res-20 networks.
Following \citep{gupta2022pdearena,Brandstetter2022CliffordNL}, we abstain from employing down-projection techniques and instead maintain a consistent spatial resolution throughout the networks.
The best models have approx.
7M parameters for Navier-Stokes and
1.5M parameters for Maxwell's equations, in both 2D and 3D.

\subsection{Optimization:}
For each experiment and each model, we tuned the learning rate to find the optimal value.
Each model was trained until convergence.
For optimization, we used Adam optimizer \cite{Kingma2014AdamAM} with no learning decay and cosine learning rate scheduler \cite{Loshchilov2016SGDRSG} to reduce the initial value by the factor of 0.01. Training was done on a single node with 4 NVIDIA GeForce RTX 2080 Ti GPUs.

\subsection{Datasets}
\label{sec:appendix-datasets}

\paragraph{Navier Stokes:}
We use the Navier-Stokes data from~\citet{gupta2022pdearena},
which is based on $\Phi$Flow~\citep{holl2020phiflow}.
It is simulated on a grid with spatial resolution of $128 \times 128$ pixels of size $\Delta x = \Delta y = 0.25$m and temporal resolution of $\Delta t = 1.5$s.
For validation and testing, we randomly selected $1024$ trajectories from corresponding partitions.

\paragraph{Maxwell 3D:}
Simulations of the 3D Maxwell equations are taken from~\citet{Brandstetter2022CliffordNL}.
This data is discretized on a grid with a spatial resolution of $32 \times 32 \times 32$ voxels with $\Delta x = \Delta y = \Delta z = 5 \cdot 10^{-7}$m and was reported to have a temporal resolution of $\Delta t = 50$s.
In the \emph{non-relativistically modeled} setting $\Cl(\Rbb^{3, 0})$, $\mathbf{E}$ is treated as a vector field, and $\mathbf{B}$ as a bivector field.
Validation and test sets comprise $128$ simulations.

\paragraph{Maxwell 2D:}
We simulate data for Maxwell's equations on spacetime $\Rbb^{2,1}$ using \texttt{PyCharge} \cite{filipovich2022PyCharge}.
Electromagnetic fields are emitted by point sources that move, orbit and oscillate at relativistic speeds.
The spacetime grid has a resolution of $128$ points in both spatial and the temporal dimension.
Its spatial extent are $50$nm and the temporal extent are $3.77 \cdot 10^{-14}$s.

Sampled simulations contain between $2$ to $4$ oscillating charges and $1$ to $2$ orbiting charges.
The sources have charges sampled uniformly as integer values between $-3$e and $3$e.
Their positions are sampled uniformly on the grid, with a predefined minimum initial distance between them.
Each charge has a random linear velocity and either oscillates in a random direction or orbits with a random radius.
Oscillation and rotation frequencies, as well as velocities are sampled such that the overall particle velocity does not exceed $0.85$c, which is necessary since the PyCharge simulation becomes unstable beyond this limit.

As the field strengths span many orders of magnitude,
we normalize the generated fields by dividing bivectors by their Minkowski norm and multiplying them by the logarithm of this norm.
This step is non-trivial sincewMinkowski-norms can be zero or negative, however, we found that they are always positive in the generated data.
We filter out numerical artifacts by removing outliers with a standard deviation greater than $20$.
The final dataset comprises $2048$ training, $256$ validation and $256$ test simulations.

\paragraph{Dataset symmetries:}
The classical Navier Stokes equations are \emph{Galilean invariant} \cite{wang2022extensions}.
Our CS-CNN for $\Cl(\Rbb^2)$ is $\E(2)$-equivariant, capturing the subgroup of isometries without boosts.

Maxwell's equations are \emph{Poincaré invariant}.
Similar to the case of Navier Stokes, our model for $\Cl(\Rbb^3)$ is $\E(3)$-equivariant.
The relativistic spacetime model for $\Cl(\Rbb^{1,2})$ is fully equivariant w.r.t. the Poincaré group $\E(1,2)$.

The invariance of a system's equations of motion imply an equivariant system dynamics.
This statement assumes that the system is transformed \emph{as a whole},
i.e. together with boundary conditions or background fields.
It does obviously not hold when fixed symmetry-breaking boundary conditions or background fields are given.
However, implicit kernels may in this case be informed about the symmetry breaking geometric structure by providing it in form of additional inputs to the kernel network as described in \cite{zhdanov2022implicit}.

\begin{figure}
    \centering
    \includegraphics[width=0.8\columnwidth]{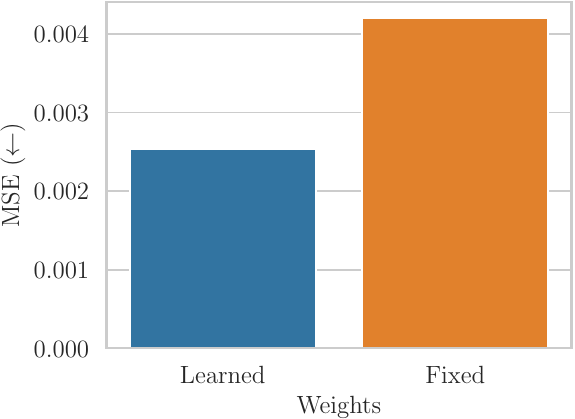}%
    \caption{
        Performance of CS-CNNs with freely \emph{learned} weights in the kernel head
        and such that ablate to \emph{fixed} weights $w^k_{mn,ij}=1$.
        \\
    }
    \label{fig:kernel_head_ablation}%
\end{figure}

\subsection{Kernel head weight ablation}
\label{apx:kernel_head_weight_ablation}

As discussed in Def.~\ref{def:kernel_head} and Appendix~\ref{sec:kernel_head},
the kernel head is essentially a partially evaluated geometric product operation
with additional weighting parameters that are learned during training. 
To check how relevant this weighting is in practice,
we ran an ablation study that fixed all kernel head weights to $w^k_{mn,ij}=1$.
It turns out that the weighting is quite relevant:
Our fully weighted CS-CNN achieved a test MSE of $2.53\cdot10^{-3}$ on the Navier Stokes forecasting task,
while the MSE for the fixed weight CS-CNN increased to $4.30\cdot10^{-3}$; see Fig.~\ref{fig:kernel_head_ablation}.
This drastic loss in performance is explained by the fact that these weights allow to scale different kernel channels relative to each other as visualized in Fig.~\ref{fig:implicit_steerable_kernel_visual_O2},
which is essential to parameterize the complete space of steerable kernels.

\section{The Clifford Algebra}
\label{sec:cliff-alg}

For completeness purposes and to complement \Cref{sec:clifford_algebra_group}, in this sections, we give a short and formal definition of the Clifford algebra. For this, we first need to introduce the tensor algebra of a vector space.

\begin{Def}[The tensor algebra]
Let $V$ be finite dimensional $\Rbb$-vector space of dimension $d$.
Then the \emph{tensor algebra} of $V$ is defined as follows:
\begin{align}
    \Tens(V) &:= \bigoplus_{m =0}^\infty V^{\otimes m} \\
    &=\Span\lC v_1 \otimes \cdots \otimes v_m \st m \ge0, v_i \in V \rC, \notag 
\end{align}
where we used the following abbreviations for the $m$-times tensor product of $V$ for $m \ge 0$:
\begin{align}
    V^{\otimes m} &:= \underbrace{V \otimes \cdots \otimes V}_{m\text{-times}}, & V^{\otimes 0}&:=\Rbb.
\end{align}
\end{Def}

Note that the above definition turns $(\Tens(V),\otimes)$ into a (non-commutative, infinite dimensional, unital, associative) algebra over $\Rbb$. In fact, the tensor algebra $(\Tens(V),\otimes)$ is, in some sense, the biggest algebra generated by $V$.

We now have the tools to give a proper definition of the Clifford algebra:

\begin{Def}[The Clifford algebra]
\label{def:CliffordAlgebraAppendix}
    Let $(V,\ip)$ be a finite dimensional innner product space over $\Rbb$ of dimension $d$.
    The \emph{Clifford algebra} of $(V,\ip)$ is then defined as the following quotient algebra:
    \begin{align}
        \Cl(V,\ip) &:= \Tens(V)/I(\ip), \\[8pt]
        I(\ip) &:= \big\langle v \otimes v - \ip(v,v) \cdot 1_{\Tens(V)} \big| v \in V\big\rangle \\[2pt]
        &:= \Span\!\Big\{ x \otimes \big( v \otimes v - \ip(v,v) \cdot 1_{\Tens(V)} \big) \otimes y \notag \\
        & \mkern170mu \Big|\; v\in V, x,y \in \Tens(V)   \Big\}
        \mkern-1mu,\mkern-10mu \notag
    \end{align} 
    where $I(\ip)$ denotes the \emph{two-sided ideal} of $\Tens(V)$ generated by the relations $v \otimes v \sim  \ip(v,v) \cdot 1_{\Tens(V)}$ for all $v \in V$.

    The product on $\Cl(V,\ip)$ that is induced by the tensor product $\otimes$ is called the \emph{geometric product} $\gp$ and will be denoted as follows:
    \begin{align}
    \label{eq:GeometricProductDefAppendix}
        x_1 \gp x_2 &:= [z_1 \otimes z_2],
    \end{align}
    with the equivalence classes $x_i=[z_i]\in \Cl(V,\ip)$, $i=1,2$.
\end{Def}

Note that, since $I(\ip)$ is a two-sided ideal, the geometric product is well-defined. The above construction turns $(\Cl(V,\ip),\gp)$ into a (non-commutative, unital, associative) algebra over $\Rbb$. 

In some sense, $(\Cl(V,\ip),\gp)$ is the biggest (non-commutative, unital, associative) algebra $(\Acal,\gp)$ over $\Rbb$ that is generated by $V$ and satisfies the relations $v \gp v = \ip(v,v) \cdot 1_{\Acal}$ for all $v \in V$.

It turns out that $(\Cl(V,\ip),\gp)$ is of the finite dimension $2^d$ and carries a \emph{parity grading} of algebras and a \emph{multivector grading} of vector spaces, see \cite{ruhe2023geometric} Appendix D. More properties are also  explained in \Cref{sec:clifford_algebra_group}.

From an abstract, theoretical point of view, the most important property of the Clifford algebra is its \emph{universal property}, which fully characterizes it:

\begin{Thm}[The universal property of the Clifford algebra]
    \label{rem:cliff-univ-prop}
    Let $(V,\ip)$ be a finite dimensional innner product space over $\Rbb$ of dimension $d$.  
    For every (non-commutative, unital, associative) algebra  $(\Acal,\agp)$ over $\Rbb$ and every $\Rbb$-linear map $f:\, V \to \Acal$ such that 
    for all $v \in V$ we have:
    \begin{align}
        f(v) \agp f(v) &= \ip(v,v) \cdot 1_\Acal, \label{eq:univ-prop}
    \end{align}
    there exists a unique algebra homomorphism (over $\Rbb$):
    \begin{align}
            \bar f:\, (\Cl(V,\ip),\gp) \to (\Acal,\agp),
    \end{align}
    such that $\bar f(v) = f(v)$ for all $v \in V$.
\end{Thm}
\begin{proof}
    The map $f:\, V \to \Acal$ uniquely extends to an algebra homomorphism on the tensor algebra:
    \begin{align}
       f^\otimes:\, \Tens(V) \to \Acal, 
    \end{align}
    given by:
    \begin{align}
      &f^\otimes\lp \sum_{i \in I} c_i \cdot v_{i,1} \otimes \cdots \otimes v_{i,l_i} \rp \notag\\
      &:= \sum_{i \in I} c_i \cdot f(v_{i,1}) \agp \cdots \agp f(v_{i,l_i}). 
    \end{align}
    Because of \Cref{eq:univ-prop} we have for every $v \in V$:
    \begin{align}
        &f^\otimes\lp v\otimes v - \ip(v,v) \cdot 1_{\Tens(V)} \rp \notag \\
        &= f(v) \agp f(v) - \ip(v,v) \cdot 1_\Acal \\
        &=0,
    \end{align}
    and thus:
    \begin{align}
        f^\otimes(I(\ip)) &= 0.
    \end{align}
    This shows that $f^\otimes$ then factors through the thus well-defined induced quotient map of algebras:
    \begin{align}
      \bar f:\,  \Cl(V,\ip) = \Tens(V)/I(\ip) &\to \Acal \\
      \bar f([z]) &:= f^\otimes(z).
    \end{align}
    This shows the claim.
\end{proof}

\begin{Rem}[The universal property of the Clifford algebra]
    The universal property of the Clifford algebra can more explicitely be stated as follows:
    
    If $f$ satisfies \Cref{eq:univ-prop} and $x \in \Cl(V,\ip)$, then we can take \emph{any} representation of $x$ of the following form: 
    \begin{align}
        x &= \sum_{i \in I} c_i \cdot v_{i,1} \gp \cdots \gp v_{i,l_i},
    \end{align}
    with any finite index sets $I$, any $l_i\in \Nbb$ and any coefficients $c_0,c_i \in \Rbb$ and any vectors $v_{i,j} \in V$, $j =1,\dots,l_i$, $i \in I$, and, then
    we can compute $\bar f(x)$ by the following formula:
    \begin{align}
        \bar f(x) &= \sum_{i \in I} c_i \cdot f(v_{i,1})\agp\cdots \agp f(v_{i,l_i}),
    \end{align}
    and \emph{no ambiguity} can occur for $\bar f(x)$ if one uses a different such representation for $x$.
\end{Rem}

\begin{Eg}
The universal property of the Clifford algebra can, for instance, be used to show that the action of the (pseudo-)orthogonal group:
\begin{align}
   \O(V,\ip) \times \Cl(V,\ip) &\to \Cl(V,\ip), \\
   (g,x) & \mapsto \rcl(g)(x),
\end{align}
given by:
\begin{align}
    &\rcl(g)\lp \sum_{i \in I} c_i \cdot v_{i,1} \gp \cdots \gp v_{i,l_i} \rp \notag \\ 
    &:= \sum_{i \in I} c_i \cdot (g v_{i,1}) \gp \cdots \gp (g v_{i,l_i}), \label{eq:Opq-action-clf}
\end{align}
is well-defined. For this one only would need to check \Cref{eq:univ-prop} for $v \in V$:
\begin{align}
    (gv) \gp (gv) &= \ip(gv,gv) \cdot 1_{\Cl(V,\ip)} \\ & = \ip(v,v) \cdot 1_{\Cl(V,\ip)},
\end{align}
where the first equality holds by the fundamental relation of the Clifford algebra and 
where the last equality holds by definition of $\O(V,\ip) \ni g$. 
So the linear map $g:\, V \to \Cl(V,\ip)$, by the universal property of the Clifford algebra, thus uniquely extends to the algebra homomorphism:
\begin{align}
    \rcl(g):\, \Cl(V,\ip) \to \Cl(V,\ip),
\end{align}
as defined in \Cref{eq:Opq-action-clf}. One can then check the remaining rules for a group action in a straightforward way.

More details can be found in \cite{ruhe2023geometric} Appendix D and E.
\end{Eg}

\section{Proofs}
\label{sec:proofs}

\begin{Prf}[Equivariance of the kernel head]{prp:kernel-head-equiv}
Recall the definition of the kernel head:
\begin{align}
    H\mkern-4mu:\mkern1mu  \Cl(\Rpq)^{\cout\mkern-1mu\times \cin}  \!&\to
    \Hom_{\vecrm} \mkern-4mu\big(\mkern-4mu \Cl(\Rpq)^\cin\!, \Cl(\Rpq)^\cout\mkern-1mu \big) \notag \\ 
    \kb  &\mapsto H(\kb) =\big[ \fd \mapsto H(\kb)[\mkern2mu\fd\mkern2mu]\big], 
\end{align}
which on each output channel $i \in [\cout]$ and grade component $k=0,\dots,d$, was given by:
\begin{align*}
    H(\kb)[\mkern2mu\fd\mkern2mu]_i^{(k)}
    &:= \sum\nolimits_{\substack{j \in [\cin]\\m,n=0,\dots,d}} w_{mn,ij}^k \cdot \lp \kb_{ij}^{(m)} \gp \fd_j^{(n)}\rp^{(k)}, 
    \label{eq:kernel-head}
\end{align*}
with:
\vspace*{-2ex}
\begin{alignat*}{3}
    && w_{mn,ij}^k\ &\in\ \Rbb \,, \\
    \kb\ &=&\ [\kb_{i,j}]_{\substack{i \in [\cout]\\j \in\ [\cin]}}\ &\in \Cl(\Rpq)^{\cout \times \cin} \,, \\
    \fd\ &=&\ [\fd_{1},\dots,\fd_{\cin}]\ &\in\ \Cl(\Rpq)^\cin \,.
\end{alignat*}

Clearly, $H(\kb)$ is a $\Rbb$-linear map (in $\fd$).
Now let $g \in \O(p,q)$.
We are left to check the following equivariance formula:
\begin{align}
    H\big(\rcl[\cout\times\cin](g)(\kb)\big)
    \stackrel{?}{=}\ & \rho_{\Hom}(g)\big(H(\kb)\big) \\
    :=\ & \rcl[\cout](g)\, H(\kb) \, \rcl[\cin](g^{-1}). \notag
\end{align}
We abbreviate
\begin{alignat*}{3}
    s\ :=&\ \rcl[\cin](g^{-1}) (\mkern2mu\fd\mkern2mu)\ &\in&\ \Cl(\Rpq)^\cin \,, \\
    Q\ :=&\ \rcl[\cout\times\cin](g)(\kb)\ &\in&\ \Cl(\Rpq)^{\cout \times \cin} \,.
\end{alignat*}

First note that we have for $j \in [\cin]$:
\begin{align}
    \rcl(g)(s_j) &= \fd_j.
\end{align}
We then get:
\begin{align*}
&   \Big[ \rho_{\Hom}(g)\big(H(\kb)\big)[\mkern2mu\fd\mkern2mu] \Big]_i^{(k)} \\
&=  \Big[ \rcl[\cout](g) \Big( H(\kb) \big[ \rcl[\cin](g^{-1})(\mkern2mu\fd\mkern2mu) \big] \Big) \!\Big]_i^{(k)} \\
&=  \Big[ \rcl[\cout](g) \big( H(\kb) \lB s \rB \big) \Big]_i^{(k)} \\
&= \rcl(g) \lp \big[  H(\kb) \lB s \rB \big]_i^{(k)} \rp \\
&= \rcl(g) \lp   \sum\nolimits_{\substack{j \in [\cin]\\m,n=0,\dots,d}} w_{mn,ij}^k \cdot \lp \kb_{ij}^{(m)} \gp s_j^{(n)}\rp^{(k)}  \rp \\
&=  \sum_{\substack{j \in [\cin]\\m,n=0,\dots,d}} \mkern-14mu w_{mn,ij}^k \mkern-4mu\cdot\mkern-2mu \lp \big[\rcl(g)(\kb_{ij})\big]^{(m)} \gp \big[ \rcl(g)(s_j)\big]^{(n)}\rp^{(k)}  \\
&=  \sum\nolimits_{\substack{j \in [\cin]\\m,n=0,\dots,d}} w_{mn,ij}^k \cdot \lp Q_{ij}^{(m)} \gp \fd_j^{(n)}\rp^{(k)}  \\
&= \Big[ H(Q)[\mkern2mu\fd\mkern2mu] \Big]_i^{(k)} \\
&= \Big[ H\big(\rcl[\cout\times\cin](g)(\kb)\big)[\mkern2mu\fd\mkern2mu] \Big]_i^{(k)}.
\end{align*}
Note that we repeatedly made use of the rules in \Cref{thm:group_action_on_Clifford_full} and \Cref{thm:group_action_on_Clifford_grade}, i.e.\ the linearity, composition, multiplicativity and grade preservation of $\rcl(g)$.
As this holds for all $m$, $k$ and $\fd$ we get the desired equation,
    \begin{align}
        \rho_{\Hom}(g)(H(\kb)) &=  H(\rcl[\cout\times\cin](g)(\kb)),
    \end{align}
which shows the claim. 
\end{Prf}

\section{Clifford-steerable CNNs on pseudo-Riemannian manifolds}
\label{sec:cs-cnn-prm}

In this section we will assume that the reader is already familiar with the general definitions of differential geometry, which can also be found in \citet{weiler2021coordinateIndependent,weiler2023EquivariantAndCoordinateIndependentCNNs}.
We will in this section state the most important results for deep neural networks that process feature fields on \emph{$G$-structured pseudo-Riemannian manifolds}.
These results are direct generalizations from those in \citet{weiler2023EquivariantAndCoordinateIndependentCNNs}, where they were stated for ($G$-structured) Riemannian manifolds, but which verbatim generalize to ($G$-structured) \emph{pseudo}-Riemannian manifolds if one replaces $\O(d)$ with $\O(p,q)$ everywhere.

Recall, that in this geometric setting a signal $f$ on the manifold $M$ is typically represented by a feature field $f:\, M \to \Acal$ of a certain ``type'', like a scalar field, vector field, tensor field, multi-vector field, etc. 
Here $f$ assigns to each point $z$ an $n$-dimensional feature $f(z) \in \Acal_z \cong \Rbb^n$.
Formally, $f$ is a global section of a \emph{$G$-associated vector bundle} $\Acal$ with typical fibre $\Rbb^n$, i.e.\ $f \in \Gs(\Acal)$, see \citet{weiler2023EquivariantAndCoordinateIndependentCNNs} for details. We can consider $\Gs(\Acal)$ as the vector space of all vector fields of type $\Acal$.
A deep neural network $F$ on $M$ with $N$ layers can then, as before, be considered as a composition:
\begin{align}
  F:\,  \Gs(\Acal_0) \stackrel{L_1}{\to} \Gs(\Acal_1) \stackrel{L_2}{\to} \Gs(\Acal_2) \stackrel{L_3}{\to} \cdots  \stackrel{L_N}{\to}   \Gs(\Acal_N),
\end{align}
where $L_1, \dots, L_N$ are maps between the vector spaces of vector fields $\Gs(\Acal_\ell)$, which are typically linear maps or simple fixed non-linear maps. 

For the sake of analysis we can focus on one such linear layer: $L:\, \Gs(\Ain) \to \Gs(\Aout)$.

Our goal is to describe the case, where $L$ is an integral operator with an convolution kernel\footnote{Note that a convolution operator $L(f)(u)=\int K(u,v) f(v) \,dv$ can be seen as a continuous analogon to a matrix multiplication. In our theory $K$ will need to depend on only one argument, corresponding to a circulant matrix.} such that: 
\begin{enumerate*}[label=\roman*.)]
\item it is well-defined, i.e.\ independent of the choice of (allowed) local coordinate systems (\emph{covariance}),
\item we can use the \emph{same} kernel $K$ (not just corresponding ones) in \emph{any} (allowed) local coordinate system (\emph{gauge equivariance}),
\item it can do \emph{weight sharing} between different locations, meaning that the \emph{same} kernel $K$ will be applied at every location,
\item input and output transform correspondingly under global transformations (\emph{isometry equivariance}).
\end{enumerate*} 

The \emph{isometry equivariance} here is the most important property. Our main results in this Appendix will be that isometry equivariance will in fact follow from the first points, see \Cref{thm:G-cnns-isom-equiv} and \Cref{thm:cs-cnn-prm-fin}.

Before we introduce our \emph{Clifford-steerable CNNs} on general \emph{pseudo-Riemannian manifolds} with multi-vector feature fields in \Cref{sec:cs-cnns-prm}, we first recall the general theory of $G$-steerable CNNs on $G$-structured pseudo-Riemannian manifolds in total analogy to \citet{weiler2023EquivariantAndCoordinateIndependentCNNs} in the next section,  \Cref{sec:G-cnns-prm}.

\subsection{General $G$-steerable CNNs on $G$-structured pseudo-Riemannian manifolds}
\label{sec:G-cnns-prm}

For the convenience of the reader, we will now recall the most important needed concepts from pseudo-Riemannian geometry in some more generality, but refer to \citet{weiler2023EquivariantAndCoordinateIndependentCNNs} for further details and proofs.

We will assume that the curved space $M$ will carry a (non-degenerate, possibly indefinite) metric tensor $\eta$ of signature $(p,q)$, $d=p+q$, and will also come with ``internal symmetries'' encoded by a closed subgroup $G \ins \GL(d)$.

\begin{Def}[$G$-structure]
    Let $(M,\ip)$ be pseudo-Riemannian manifold of signature $(p,q)$, $d=p+q$, and $G \le \GL(d)$ a closed subgroup.
    A \emph{$G$-structure} on $(M,\ip)$ is a principle $G$-subbundle $\iota:\, GM \hookrightarrow \Frm M$ of the frame bundle $\Frm M$ over $M$. Note that $GM$ is supposed to carry the right $G$-action induced from $\Frm M$:
    \begin{align}
        \ract:\,  GM \times G & \to GM, & [e_i]_{i \in [d]} \ract g & :=  \lB \sum_{j \in [d]} e_j \, g_{j,i} \rB_{i \in [d]},
    \end{align}
    which thus makes the embedding $\iota$ a $G$-equivariant embedding.
\end{Def}

\begin{Def}[$G$-structured pseudo-Riemannian manifold]
    Let $G \le \GL(d)$ be closed subgroup. A \emph{$G$-structured pseudo-Riemannian manifold} $(M,G,\ip)$ of signature $(p,q)$ - per definition - consists of a pseudo-Riemannian manifold $(M,\ip)$ of dimension $d=p+q$ with a metric tensor $\eta$ of signature $(p,q)$, and, a \emph{fixed choice} of a \emph{$G$-structure} $\iota:\, GM \hookrightarrow \Frm M$ on $M$.

    We will denote the $G$-structured pseudo-Riemannian manifold with the triple $(M,G,\ip)$ and keep the fixed $G$-structure $\iota:\,GM \hookrightarrow \Frm M$ implicit in the notation, as well as the corresponding \emph{$G$-atlas} of local tangent bundle trivializations: 
    \begin{align}
        \Abb_G=\lC (\Psi^A,U^A) \st \pi_{\Tan M}^{-1}(U_A) \bij[\Psi^A] U_A \times \Rbb^d \rC_{A \in \Ical}
    \end{align}
    where $\Ical$ is an index set and $U^A \ins M$ are certain open subsets of $M$.
\end{Def}

\begin{Rem}
Note that for any given $G \le \GL(d)$ there might not exists a corresponding $G$-structure $GM$ on $(M,\ip)$ in general.
Furthermore, even if it existed it might not be unique. So, when we talk about such a $G$-structure in the following we always make the implicit assumption of its existence and we also fix a specific choice.
\end{Rem}

\begin{Def}[Isometry group of a $G$-structured pseudo-Riemannian manifold]
    Let $(M,G,\ip)$ be a $G$-structured pseudo-Riemannian manifold. Its ($G$-structure preserving) \emph{isometry group} is defined to be:
    \begin{align}
        & \Isom(M,G,\ip) \notag\\
        &:= \big\{ \phi:\, M \bij M \text{ diffeo} \,|\, \forall z \in M,\, v \in \Tan_zM. \notag\\
        &\qquad \eta_{\phi(z)}(\phi_{*,\Tan M}(v),\phi_{*,\Tan M}(v)) = \eta_z(v,v), \notag \\
        &\qquad \phi_{*,\Frm M}(G_zM) = G_{\phi(z)} M \big\}.
    \end{align}
\end{Def}

The intuition here is that the first condition constrains $\phi$ to be an isometry w.r.t.\ the metric $\ip$. The second condition constrains $\phi$ to be a symmetry of the $G$-structure, i.e.\ it maps $G$-frames to $G$-frames.

\begin{Rem}[Isometry group]
  Recall that the (usual/full) \emph{isometry group} of a pseudo-Riemannian manifold $(M,\ip)$ is defined as:
    \begin{align}
        & \Isom(M,\ip) \notag\\
        &:= \big\{ \phi:\, M \bij M \text{ diffeo} \,|\, \forall z \in M,\, v \in \Tan_zM. \notag\\
        &\qquad \eta_{\phi(z)}(\phi_{*,\Tan M}(v),\phi_{*,\Tan M}(v)) = \eta_z(v,v) \big\}.
    \end{align}
    Also note that for a  $G$-structured pseudo-Riemannian manifold $(M,G,\ip)$ of signature $(p,q)$ such that $\O(p,q) \le G$ we have:
 \begin{align}
     \Isom(M,G,\ip) &= \Isom(M,\ip).
 \end{align}
\end{Rem}

\begin{Def}[$G$-associated vector bundle]
\label{def:G_associated_vector_bundle}
    Let $(M,G,\ip)$ be a $G$-structured pseudo-Riemannian manifold and let $\rho:\, G \to \GL(n)$ be a left linear representation of $G$.
    A vector bundle $\Acal$ over $M$ is called a \emph{$G$-associated vector bundle (with typical fibre $(\Rbb^n,\rho)$)} if there exists a vector bundle isomorphism over $M$ of the form:
    \begin{align}
        \Acal & \bij \lp GM \times \Rbb^n \rp/\eqr_\rho =: GM \times_\rho \Rbb^n ,
    \end{align}
    where the equivalence relation is given as follows:
    \begin{align}
       & (e',v') \sim_\rho (e,v) \notag \\&:\iff \exists g \in G. \quad (e',v') = (e \ract g, \rho(g^{-1}) v).
    \end{align}
\end{Def}

\begin{Def}[Global sections of a fibre bundle]
\label{def:global_bundle_section}
    Let $\pi_\Acal: \Acal \to M$ be a fibre bundle over $M$.
    We denote the set of \emph{global sections} of $\Acal$ as:
\begin{align}
    \Gs(\Acal) &:= \lC f: M \to \Acal \st \forall z \in M.\, f(z) \in \Acal_z \rC,
\end{align}
where $\Acal_z := \pi_\Acal^{-1}(z)$ denotes the \emph{fibre} of $\Acal$ over $z \in M$.
\end{Def}

\begin{Rem}[Isometry action]
    \label{rem:isom-action}
    For a $G$-associated vector bundle $\Acal = GM \times_\rho \Rbb^n$ and $\phi \in \Isom(M,G,\ip)$ we can define the induced $G$-associated vector bundle automorphism $\phi_{*,\Acal}$ on $\Acal$ as follows:
\begin{align}
    \phi_{*,\Acal}:\, \Acal &\to \Acal, \\ \phi_{*,\Acal}\lp e,v \rp 
    &:= \lp \phi_{*,GM}(e),v \rp.
\end{align}
With this we can define a left action of the group $\Isom(M,G,\ip)$ on the corresponding space of feature fields $\Gs(\Acal)$ as follows:
\begin{align}
    \lact:\, \Isom(M,G,\ip) \times \Gs(\Acal) \to \Gs(\Acal), \\
\phi \lact f := \phi_{*,\Acal} \circ f \circ \phi^{-1}:\, M \to \Acal.
\end{align}
\end{Rem}

To construct a well-behaved convolution operator on $M$ we first need to introduce the idea of a transporter of feature fields along a curve $\gamma:\, I \to M$.

\begin{Rem}[Transporter]
    A \emph{transporter} $\Trpt_\Acal$ on the vector bundle $\Acal$ over $M$ takes any (sufficiently smooth) curve $\gamma:\, I \to M$ with $I \ins \Rbb$ some interval and two points $s,t \in I$, $s \le t$, and provides an invertible linear map:
    \begin{align}
        \Trpt_{\Acal,\gamma}^{s,t}:\, \Acal_{\gamma(s)} &\bij \Acal_{\gamma(t)}, & v & \mapsto \Trpt_{\Acal,\gamma}^{s,t}(v).
    \end{align}
    $\Trpt_\Acal$ is thought to transport the vector $v \in \Acal_{\gamma(s)}$ at location $\gamma(s) \in M$ along the curve $\gamma$ to the location $\gamma(t) \in M$ and outputs a vector $\tilde v=\Trpt_{\Acal,\gamma}^{s,t}(v)$ in $\Acal_{\gamma(t)}$.
   
    For consistency we require that $\Trpt_\Acal$ satisfies the following points for such $\gamma$:
    \begin{enumerate}
        \item For $s \in I$ we get: $\Trpt_{\Acal,\gamma}^{s,s} \stackrel{!}{=} \id_{\Acal_{\gamma(s)}} :\, \Acal_{\gamma(s)} \bij \Acal_{\gamma(s)}$, 
        \item For $s \le t \le u$ we have:
    \begin{align}
        \Trpt_{\Acal,\gamma}^{t,u} \circ \Trpt_{\Acal,\gamma}^{s,t} \stackrel{!}{=} \Trpt_{\Acal,\gamma}^{s,u}:\, \Acal_{\gamma(s)} &\bij \Acal_{\gamma(u)}.
    \end{align}
    \end{enumerate}
      Furthermore, the dependence on  $s$, $t$ and $\gamma$ shall be ``sufficiently smooth'' in a certain sense.
 
      We call a transporter $\Trpt_{\Tan M}$ on the tangent bundle $\Tan M$ a \emph{metric transporter} 
      if the map: 
    \begin{align}
      \Trpt_{\Tan M,\gamma}^{s,t}:\, (\Tan_{\gamma(s)} M, \ip_{\gamma(s)} ) \bij (\Tan_{\gamma(t)} M, \ip_{\gamma(t)})
  \end{align}
   is always an isometry.
\end{Rem}

To construct \emph{transporters} we need to introduce the notion of a \emph{connection} on a vector bundle, which formalized how vector fields change when moving from one point to the next.

\begin{Def}[Connection]
    A \emph{connection} on a vector bundle $\Acal$ over $M$ is an $\Rbb$-linear map:
    \begin{align}
        \nabla:\, \Gs(\Acal) &\to \Gs(\Tan^* M \otimes \Acal),
    \end{align}
    such that for all $c: M \to \Rbb$ and $f \in \Gs(\Acal)$ we have:
    \begin{align}
        \nabla(c \cdot f) &= dc \otimes f + c \cdot \nabla(f),
    \end{align}
    where $dc \in \Gs(\Tan^* M)$ is the differential of $c$.
\end{Def}

A special form of a connection are affine connections, which live on the tangent space.

\begin{Def}[Affine connection]
    An \emph{affine connection} on $M$  (or more precisely, on $\Tan M$) is an $\Rbb$-bilinear map:
    \begin{align}
        \nabla:\, \Gs(\Tan M) \times \Gs(\Tan M)  & \to \Gs(\Tan M), \\ (X,Y) & \mapsto \nabla_X Y,
    \end{align}
    such that for all $c: M \to \Rbb$ and $X,Y \in \Gs(\Tan M)$ we have:
    \begin{enumerate}
        \item $\nabla_{c \cdot X} Y = c \cdot \nabla_X Y$,
        \item $\nabla_X(c \cdot Y) = (\partial_X c) \cdot Y + c \cdot \nabla_X Y$,
    \end{enumerate}
    where $\partial_Xc$ denotes the directional derivative of $c$ along $X$.
\end{Def}

\begin{Rem}
Certainly, an affine connection can also be re-written in the usual connection form:
    \begin{align}
        \nabla:\, \Gs(\Tan M) &\to \Gs(\Tan^* M \otimes \Tan M).
    \end{align}
\end{Rem}

Every connection defines a (parallel) transporter $\Trpt_{\Acal}$.

\begin{DefLem}[Parallel transporter of a connection]
    Let $\nabla$ be a connection on the vector bundle $\Acal$ over $M$. Then $\nabla$ defines a (parallel) transporter $\Trpt_{\Acal}$ for $\gamma: I=[s,t] \to M$ as follows:
   \begin{align}
       \Trpt_{\Acal,\gamma}^{s,t}:\, \Acal_{\gamma(s)}  &\bij \Acal_{\gamma(t)}, & v & \mapsto f(t),
  \end{align}
  where $f$ is the unique vector field $f \in \Gs(\gamma^*\Acal)$  with:
  \begin{enumerate}
      \item $(\gamma^*\nabla)(f) = 0$,
      \item $f(s) = v$,
  \end{enumerate}
  which always exists. Here $\gamma^*$ denotes the corresponding pullback from $M$ to $I$.
\end{DefLem}

For pseudo-Riemannian manifolds there is a ``canonical'' choice of a metric connection, the Levi-Cevita connection, which always exists and is uniquely characterized by its two main properties.

\begin{DefThm}[Fundamental theorem of pseudo-Riemannian geometry: the Levi-Civita connection]
    Let $(M,\ip)$ be a pseudo-Riemannian manifold. Then there exists a unique affine connection $\nabla$ on $(M,\ip)$ such that the following two conditions hold for all $X,Y,Z \in \Gs(\Tan M)$;
    \begin{enumerate}
        \item metric preservation: %
            \begin{align}
                \partial_Z \lp \ip(X,Y) \rp &= \ip(\nabla_Z X,Y) + \ip(X,\nabla_Z Y).
            \end{align}
        \item torsion-free: 
            \begin{align}
                 \nabla_XY-\nabla_YX =  [X,Y],
            \end{align}
            where $[X,Y]$ is the Lie bracket of vector fields. 
    \end{enumerate}
    This affine connection is called the \emph{Levi-Cevita connection} of $(M,\ip)$ and is denoted as $\nabla^\LC$.
\end{DefThm}

\begin{Rem}[Levi-Civita transporter] %
    Let $(M,G,\ip)$ be a pseudo-Riemannian manifold with Levi-Cevita connection $\nabla^\LC$.
\begin{enumerate}
    \item The corresponding Levi-Cevita transporter $\Trpt_{\Tan M}$ on $\Tan M$ is always a metric transporter, i.e.\ it always induces (linear) isometries of vector spaces:
     \begin{align}
         \Trpt_{\Tan M,\gamma}^{s,t}:\, (\Tan_{\gamma(s)} M, \ip_{\gamma(s)}) \bij (\Tan_{\gamma(t)} M, \ip_{\gamma(t)}).
    \end{align}
\item Furthermore, the Levi-Cevita transporter extends to every $G$-associated vector bundle $\Acal$ as $\Trpt_{\Acal}$.
\item     For every $G$-associated vector bundle $\Acal$, every curve $\gamma: I \to M$ and $\phi \in \Isom(M,G,\ip)$, 
    the Levi-Cevita transporter $\Trpt_{\Acal,\gamma}$ always satisfies:
    \begin{align}
        \phi_{*,\Acal} \circ \Trpt_{\Acal,\gamma} &= \Trpt_{\Acal,\phi \circ \gamma} \circ\, \phi_{*,\Acal}. \label{eq:isom-trpt}
    \end{align}
\end{enumerate}
\end{Rem}

\begin{Def}[Geodesics]
    Let $M$ be a manifold with affine connection $\nabla$ and $\gamma: I \to M$ a curve.
    We call $\gamma$ a \emph{geodesic} of $(M,\nabla)$ if for all $t \in I$ we have:
    \begin{align}
        \nabla_{\dot\gamma(t)}\dot \gamma(t) = 0,
    \end{align}
    i.e.\ if $\gamma$ runs parallel to itself.

    For pseudo-Riemannian manifolds $(M,\ip)$ we will typically use the Levi-Cevita connection $\nabla^\LC$ to define geodesics. %
\end{Def}

\begin{DefLem}[Pseudo-Riemannian exponential map]
    For a manifold $M$ with affine connection $\nabla$, $z \in M$ and $v \in \Tan_zM$ there exists 
    a unique geodesic $\gamma_{z,v}:I=(-s,s) \to M$ of $(M,\nabla)$ with maximal domain $I$ such that:
    \begin{align}
        \gamma_{z,v}(0) & = z, & \dot \gamma_{z,v} (0) & = v.
    \end{align}
    The \emph{$\nabla$-exponential map} at $z \in M$ is then the map:
    \begin{align}
        \exp_z: \, \Tan_z^\circ M &\to M, & \exp_z(v) &:= \gamma_{z,v}(1),
    \end{align}
    with domain:
    \begin{align}
        \Tan_z^\circ M &:= \lC v \in \Tan_z M \st \gamma_{z,v}(1) \text{ is defined} \rC.
    \end{align}
    For pseudo-Riemannian manifolds $(M,\ip)$ we will call the exponential map $\exp_z$ defined via the Levi-Cevita connection $\nabla^\LC$ the \emph{pseudo-Riemannian exponential map} of $(M,\ip)$ at $z \in M$.
\end{DefLem}

\begin{Rem}
    For a pseudo-Riemannian manifold $(M,\ip)$ the differential $d\exp_z|_v:  \Tan_v \Tan_z M \to \Tan_{\exp_z(v)}M$ 
    is the identity map on $\Tan_z M$ at $v=0 \in \Tan_z M$: 
$d\exp_z|_{v=0}\stackrel{!}{=}\id_{\Tan_z M}:$  $\Tan_z M = \Tan_0 \Tan_z M \to  \Tan_{\exp_z(0)}M =  \Tan_zM$.   
    
    Furthermore, there exist an open subset $U_z \ins \Tan_z M$ such that $0 \in U_z$ and $\exp_z: U_z \to \exp_z(U_z) \ins M$ is a diffeomorphism and $\exp_z(U_z) \ins M$ is an open subset.
\end{Rem}

\begin{Not}
    For a transporter $\Trpt_\Acal$ for a vector bundle on $(M,\nabla)$  we abbreviate for $z \in M$ and $v \in \Tan_z^\circ M$:
    \begin{align}
        \Trpt_{z,v} := \Trpt_{\Acal,\gamma_{z,v}^-}: \, \Acal_{\exp_z(v)} & \bij \Acal_z,
    \end{align}
    where 
    $\gamma_{z,v}^- : [0,1] \to M$ is given by $\gamma_{z,v}^-(t) :=  \exp_z((1-t)\cdot v)$.
\end{Not}

\begin{Def}[Transporter pullback, see \citet{weiler2023EquivariantAndCoordinateIndependentCNNs} Def.\ 12.2.4] %
    \label{def:trpt-pullback}
    Let $(M,\ip)$ be a pseudo-Riemannian manifold and $\Acal$ a vector bundle over $M$.
    Furthermore, let $\exp_z$ denote the pseudo-Riemannian exponential map (based on the Levi-Civita connection)
    and $\Trpt_\Acal$ any transporter on $\Acal$.
    We then define the \emph{transporter pullback}:
    \begin{align}
        \Exp_z^*:\, \Gs(\Acal) &\to C(\Tan_z^\circ M, \Acal_z), \\
        \Exp_z^*(f)(v) & := \Trpt_{z,v}\big( \underbrace{f(\exp_z(v))}_{\in \Acal_{\exp_z(v)}} \big) \in \Acal_z.
    \end{align}
\end{Def}

\begin{Lem}[See \citet{weiler2023EquivariantAndCoordinateIndependentCNNs} Thm.\ 13.1.4]
    For $G$-structured pseudo-Riemannian manifold $(M,G,\ip)$ and $G$-associated vector bundle $\Acal$,
    $z \in M$, $\phi \in \Isom(M,G,\ip)$ and $f \in \Gs(\Acal)$ we have:
    \begin{align}
        \Exp_z^*(\phi \lact f) &= \phi_{*,\Acal} \circ [\Exp_{\phi^{-1}(z)}^*(f)] \circ \phi_{*,\Tan M}^{-1},
    \end{align}
    provided the transporter map $\Trpt_\Acal$ satisfies \Cref{eq:isom-trpt}.
\end{Lem}

\emph{Weight sharing} for the convolution operator $I$ boils down to the use of a \emph{template convolution kernel} $K$, which is then applied/re-used at every location $z \in M$.

\begin{Def}[Template convolution kernel]
    Let $M$ be a manifold of dimension $d$ and $\Ain$ and $\Aout$ two vector bundles over $M$ with typical fibres $\Win$ and $\Wout$, resp.
 A \emph{template convolution kernel} for $(M,\Ain,\Aout)$ is then a (sufficiently smooth, non-linear) map:
\begin{align}
    K:\, \Rbb^d \to \Hom_\vecrm(\Win,\Wout),
\end{align}
that is sufficiently decaying when moving away from the origin $0 \in \Rbb^d$ (to make all later constructions, like convolution operations, etc., well-defined).
\end{Def}

The \emph{$G$-gauge equivariance} of a convolution operator $I$ is encoded by the following \emph{$G$-steerability} of the template convolution kernel.

\begin{Def}[$G$-steerability convolution kernel constraints]
    Let $G \le \GL(d)$ be a closed subgroup and $(M,G,\ip)$ be a $G$-structured pseudo-Riemannian manifold of signature $(p,q)$, $d=p+q$, and $\Ain$ and $\Aout$ two $G$-associated vector bundles with typical fibre $(\Win,\rin)$ and $(\Wout,\rout)$, resp.
    A template convolution kernel $K$ for $(M,\Ain,\Aout)$:
\begin{align}
    K:\, \Rbb^d \to \Hom_\vecrm(\Win,\Wout),
\end{align}
    will be called \emph{$G$-steerable} if 
    for all $g \in G$ and $v \in \Rbb^d$ we have:
 \begin{align}
    K(gv) &= \frac{1}{|\det g|}\, \rout(g) \, K(v) \, \rin(g)^{-1} \label{eq:kernel-constraints}
    \\&=: \rhom(g)(K(v)).  \label{eq:kernel-constraints-as-equiv}
\end{align}
\end{Def}

\begin{Rem}
   Note that the \emph{$G$-steerability} of $K$ is expressed through \Cref{eq:kernel-constraints}, while the \emph{$G$-gauge equivariance} of $K$ will, more closely, be expressed through the re-interpretation in \Cref{eq:kernel-constraints-as-equiv}. 
\end{Rem}

\begin{Def}[Convolution operator, see \citet{weiler2023EquivariantAndCoordinateIndependentCNNs} Thm.\ 12.2.9]
    Let $(M,G,\ip)$ be a $G$-structured pseudo-Riemannian manifold and $\Ain$ and $\Aout$ two $G$-associated vector bundles over $M$ with typical fibres $(\Win,\rin)$ and $(\Wout, \rout)$ and $K$ a \emph{$G$-steerable} template convolution kernel, see \Cref{eq:kernel-constraints}.
    Let $f_\jin \in \Gs(\Ain)$ and consider a  local trivialization $(\Psi^C,U^C) \in \Abb_G$ around $z \in  U^C \ins M$ (which locally trivializes $\Ain$ and $\Aout$).
 Then we have a well-defined \emph{convolution operator}:
    \begin{align}
        L:\, \Gs(\Ain) &\to \Gs(\Aout), & f_\jin &\mapsto L(f_\jin):=f_\jout,
    \end{align}
 given by the local formula:
\begin{align}
    f_\jout^C(z) &:= \int_{\Rbb^d} K(v^C) \lB [\Exp^*_z f_\jin]^C(v^C) \rB \, dv^C,   \label{eq:conv-def}
\end{align}
where $\Exp_z^*$ is the transporter pullback from \Cref{def:trpt-pullback}, where $\exp_z$ denotes the pseudo-Riemannian exponential map (based on the Levi-Cevita connection $\nabla^\LC$) and $\Trpt_{\Ain}$ any transporter satisfying \Cref{eq:isom-trpt} (e.g.\ parallel transport based on $\nabla^\LC$).
\end{Def}

\begin{Rem}[Coordinate independence of the convolution operator]
The coordinate independence of the convolution operator $L: \Gs(\Ain) \to \Gs(\Aout)$ comes from the following \emph{covariance relations} and \Cref{eq:kernel-constraints}.

If we use a different local trivialization $(\Psi^B,U^B) \in \Abb_G$ in \Cref{eq:conv-def} with $z \in U^B \cap U^C$ then there exists a $g \in G$ such that:
    \begin{align}
        v^C &= g \, v^B \in \Rbb^d,\\
        dv^C &= |\det g| \cdot dv^B, \\
        [\Exp_z^*f_\jin]^C(v^C) &= \rin(g)\, [\Exp_z^* f_\jin]^B(v^B)  \in \Win, \\
        f_\jout^C(z) &= \rout(g) f_\jout^B(z) \in \Wout.
    \end{align}
So, $f_\jout: M \to \Aout$ is a well-defined global section in $\Gs(\Aout)$.
\end{Rem}

We are finally in the place to state the main theorem of this section, stating that every 
\emph{$G$-steerable} template convolution kernel leads to an \emph{isometry equivariant} convolution operator.

\begin{Thm}[Isometry equivariance of convolution operator, see \citet{weiler2023EquivariantAndCoordinateIndependentCNNs} Thm.\ 13.2.6]
    \label{thm:G-cnns-isom-equiv}
    Let $G \le \GL(d)$ be closed subgroup and $(M,G,\ip)$ be a $G$-structured pseudo-Riemannian manifold of signature $(p,q)$ with $d=p+q$.
    Let $\Ain$ and $\Aout$ be two $G$-associated vector bundles with typical fibres $(\Win,\rin)$ and $(\Wout,\rout)$.
    Let $K$ be a \emph{$G$-steerable} template convolution kernel, see \Cref{eq:kernel-constraints}.
    Consider the corresponding convolution operator $L:\,\Gs(\Ain) \to \Gs(\Aout)$ given by \Cref{eq:conv-def}, 
    where $\exp_z$ denotes the pseudo-Riemannian exponential map (based on the Levi-Cevita connection $\nabla^\LC$) and $\Trpt_{\Ain}$ any transporter satisfying \Cref{eq:isom-trpt} (e.g.\ parallel transport based on $\nabla^\LC$).
    
    Then the convolution operator $L:\,\Gs(\Ain) \to \Gs(\Aout)$  is \emph{equivariant} w.r.t.\ the $G$-structure preserving \emph{isometry group} $\Isom(M,G,\ip)$: for every $\phi \in \Isom(M,G,\ip)$ and $f_\jin \in \Gs(\Ain)$ we have:
    \begin{align}
          L(\phi \lact f_\jin) &=\phi \lact L(f_\jin).
    \end{align}
\end{Thm}

So the main obstruction for constructing a well-behaved convolution operator $L$ are thus the kernel constraints \Cref{eq:kernel-constraints}, which are generally notoriously difficult to solve, especially for continuous \emph{non-compact} groups $G$ like $\O(p,q)$.

\subsection{Clifford-steerable CNNs on pseudo-Riemannian manifolds}
\label{sec:cs-cnns-prm}

Let $(M,\ip)$ be a pseudo-Riemannian manifold of signature $(p,q)$ and dimension $d=p+q$. 

Then $(M,\ip)$ carries a unique $\O(p,q)$-structure $\O\! M$ induced by $\ip$. The intuition is that $\O\! M$ consists of all orthonormal frames w.r.t.\ $\ip$. 
In fact, the choice of an $\O(p,q)$-structure on $M$ is equivalent to the choice of a metric $\ip$ of signature $(p,q)$ on $M$.
That said, we will now restrict to the structure group $G=\O(p,q)$ everywhere in the following.

We will further restrict to \emph{multi-vector feature fields} $\Ain := \Cl(\Tan M, \ip)^\cin$ and $\Aout := \Cl(\Tan M, \ip)^\cout$, which we first need to formalize properly.

\begin{Def}[Clifford algebra bundle]
    Let $(M,\ip)$ be a pseudo-Riemannian manifold. 
Then the \emph{Clifford algebra bundle} over $M$ is defined (as a set) as the disjoint union of the Clifford algebras of the corresponding tangent spaces:
\begin{align}
    \Cl(\Tan M,\ip) &:= \bigsqcup_{z\in M} \Cl(\Tan_zM,\ip_z).
\end{align}
$\Cl(TM,\ip)$ becomes an algebra bundle over $M$ with the standard constructions of local trivialization and bundle projections.
\end{Def}

\begin{Def}[Othonormal frame bundle of signature $(p,q)$.]
Let $(M,\ip)$ be a pseudo-Riemannian manifold of signature $(p,q)$ and dimension $d=p+q$.
Abbreviate for indices $i,j \in [d]$:
\begin{align}
    \delta_{i,j}^{p,q} &:= 
    \begin{cases}
        0 & \text{ if } i \neq j,\\
        +1& \text{ if } i=j \in [1,p],\\
        -1& \text{ if } i=j \in [p+1,d].
    \end{cases}
\end{align}
Then the \emph{orthonormal frame bundle of signature $(p,q)$} is defined as:
    \begin{align}
        \O\! M & := \bigsqcup_{z\in M} \O_z\! M, \label{eq:orth-norm-frame-bdl}
    \end{align}
where we put:
    \begin{align}
        \O_z\! M &:= \Big\{ [e_1,\dots,e_d] \;\Big|\; \forall j \in [d].\; e_j \in \Tan_zM, \\
        &\qquad \forall i,j \in [d]. \quad \ip_z(e_i,e_j) = \delta^{p,q}_{i,j} \Big\}.
    \end{align}
 Then $\O\! M$ becomes an $\O(p,q)$-structure for $(M,\ip)$ together with the standard constructions of local trivialization, bundle projection and right group action:
     \begin{align}
        \ract:\,  \O\! M \times \O(p,q) & \to \O\! M, \\ [e_i]_{i \in [d]} \ract g & :=  \lB \sum_{j \in [d]} e_j \, g_{j,i} \rB_{i \in [d]}.
    \end{align}
\end{Def}

\begin{Lem}
    Let $(M,\ip)$ be a pseudo-Riemannian manifold of signature $(p,q)$ and dimension $d=p+q$. 
    We have an algebra bundle isomorphism over $M$:
    \begin{align}
        \Cl(\Tan M,\ip) & \cong \O\! M \times_{\rcl} \Cl(\Rpq), \label{eq:clf-bdl-assc}
    \end{align}
    where $\rcl: \O(p,q) \to \O_\algrm (\Cl(\Rpq),\bar \ip^{p,q})$ is the usual action of the orthogonal group $\O(p,q)$ on  $\Cl(\Rpq)$ by rotating all vector components individually.
    In particular, the Clifford algebra bundle $\Cl(\Tan M,\ip)$ is an $\O(p,q)$-associated algebra bundle over $M$ with typical fibre $\Cl(\Rpq)$.
\end{Lem}

\begin{Def}[Multivector fields]
    A \emph{multivector field} on $M$ is a global section $f \in \Gs(\Cl(\Tan M,\ip)^c)$ for some $c \in \Nbb$, i.e.\
    a map $f: M \to \Cl(\Tan M,\ip)^c$ that assigns to every point $z \in M$ a tuple of multivectors:
    $f(z) = [f_1(z),\dots,f_c(z)] \in \Cl(\Tan_z M,\ip_z)^c$.
\end{Def}

\begin{Rem}[The action of the isometry group on multivector fields]
    Let $\phi \in \Isom(M,\ip)$ then $\phi$ is a diffeomorphic map $\phi: M \bij M$ such that for every $z \in M$
    the differential map is an isometry:
    \begin{align}
        \phi_{*,\Tan M,z}:\, (\Tan_z M, \eta_z) \bij (\Tan_{\phi(z)}, \eta_{\phi(z)}).
    \end{align}
    We can now describe the induced map $\phi_{*,\Cl(\Tan M,\ip)}$ via the general construction on associated vector fields, see \Cref{rem:isom-action}, with help of the identification \Cref{eq:clf-bdl-assc}:
    \begin{align}
       \phi_{*,\Cl(\Tan M,\ip)}: \O\! M \times_{\rcl} \Cl(\Rpq) &\to \O\! M \times_{\rcl} \Cl(\Rpq), \notag\\
        \phi_{*,\Cl(\Tan M,\ip)}(e,x) &= (\phi_{*,\Frm M}(e),x),
    \end{align}
    or we can look at the fibres directly, $z \in M$:
    \begin{align}
       & \phi_{*,\Cl(\Tan M,\ip),z}:\, \Cl(\Tan_z M,\ip_z) \to\Cl(\Tan_{\phi(z)} M,\ip_{\phi(z)}),\notag\\
       & \phi_{*,\Cl(\Tan M,\ip),z}\lp \sum_{i \in I} c_i \cdot v_{i,1} \gp \cdots \gp v_{i,k_i} \rp  \notag\\
       & = \sum_{i \in I} c_i \cdot \phi_{*,\Tan M,z}(v_{i,1}) \gp \cdots \gp \phi_{*,\Tan M,z}(v_{i,k_i}).
    \end{align}
With this we can define a left action of the isometry group $\Isom(M,\ip)$ on the corresponding space of multivector fields $\Gs(\Cl(\Tan M,\ip)^c)$ as follows:
\begin{align}
    \lact:\, \Isom(M,\ip) \times \Gs(\Cl(\Tan M,\ip)^c) \to \Gs(\Cl(\Tan M,\ip)^c), \\
\phi \lact f := \phi_{*,\Cl(\Tan M,\ip)^c} \circ f \circ \phi^{-1}:\, M \to \Cl(\Tan M,\ip)^c.
\end{align}
\end{Rem}

We are now in the position to state the main theorem of this section.

\begin{Thm}[Clifford-steerable CNNs on pseudo-Riemannian manifolds are gauge and isometry equivariant]
    \label{thm:cs-cnn-prm-fin}
    Let $(M,\ip)$ be a pseudo-Riemannian manifold of signature $(p,q)$ and dimension $d=p+q$. 
    We consider $(M,\ip)$ to be endowed with the structure group $G=\O(p,q)$.
  Consider multi-vector feature fields $\Ain = \Cl(\Tan M, \ip)^\cin$ and $\Aout = \Cl(\Tan M, \ip)^\cout$ over $M$.
 
  Let $K=H \circ \KB$ be a \emph{Clifford-steerable kernel}, the same template convolution kernel as presented in the  main paper in \Cref{sec:clifford_steerable_CNNs_main}:
    \begin{align}
        K:\, \Rpq &\to \Hom_\vecrm\lp \Cl(\Rpq)^\cin, \Cl(\Rpq)^\cout \rp,
    \end{align}
    where $\KB: \Rpq \to \Cl(\Rpq)^{\cout \times \cin}$ is the \emph{kernel network},
    a Clifford group equivariant neural network with $(\cin \cdot \cout)$ number of Clifford algebra outputs, and, where $H$ is the $\O(p,q)$-equivariant \emph{kernel head}:
   \begin{align}
        H:\, \Cl(\Rpq)^{\cout \times \cin} & \\
         \to \Hom_\vecrm \!\big(\! & \Cl(\Rpq)^\cin, \Cl(\Rpq)^\cout \big). \notag
    \end{align}
    Then $K$ is automatically $\O(p,q)$-steerable, i.e.\ for $g \in \O(p,q)$, $v \in \Rpq$ we have\footnote{Note that the factor $|\det g|^{-1}$ does not appear here, in contrast to the general formula in \Cref{eq:kernel-constraints}, because $|\det g|=1$ anyways for all $g \in \O(p,q)$.}:
 \begin{align}
    K(gv) &= \rcl[\cout](g) \, K(v) \, \rcl[\cin](g)^{-1}. \label{eq:kernel-constraints-clf}
\end{align}
Furthermore, the corresponding convolution operator $L:\,\Gs(\Ain) \to \Gs(\Aout)$, given by \Cref{eq:conv-def}, 
is equivariant w.r.t.\ the full isometry group $\Isom(M,\eta)$: for every $\phi \in \Isom(M,\ip)$ and $f_\jin \in \Gs(\Ain)$ we have:
    \begin{align}
          L(\phi \lact f_\jin) &=\phi \lact L(f_\jin).
    \end{align}
\end{Thm}

\begin{Rem}
    A similar theorem to \Cref{thm:cs-cnn-prm-fin} can be stated for \emph{orientable} pseudo-Riemannian manifolds $(M,\ip)$ and structure group $G=\SO(p,q)$,
    if one reduces the Clifford group equivariant neural network parameterizing the kernel network $\KB$ to be (only) $\SO(p,q)$-equivariant.
\end{Rem}

\end{document}